%% file: main.tex
\documentclass{article}

\usepackage[final]{neurips_2022}
\input{packages}
\input{commands}

\hypersetup{
  colorlinks=true,
  linkcolor=RawSienna,
  citecolor=gray,
  urlcolor=blue
}

\title{Evaluating Robustness to Dataset Shift \\via Parametric Robustness Sets}

\author{%
  Nikolaj Thams\thanks{Equal Contribution, order determined by coin flip.  Code is available at \href{https://github.com/clinicalml/parametric-robustness-evaluation}{this link}.} \\
  Dept.\ of Mathematical Sciences\\
  University of Copenhagen\\
  Copenhagen, Denmark\\
  \texttt{thams@math.ku.dk} \\
  \And%
  Michael Oberst$^{*}$\\
  CSAIL \& IMES\\
  MIT\\
  Cambridge, MA\\
  \texttt{moberst@mit.edu} \\
  \And%
  David Sontag\\
  CSAIL \& IMES\\
  MIT\\
  Cambridge, MA\\
  \texttt{dsontag@csail.mit.edu} \\
}

\begin{document}
\maketitle

\begin{abstract}
  We give a method for proactively identifying small, plausible shifts in distribution which lead to large differences in model performance.  These shifts are defined via parametric changes in the causal mechanisms of observed variables, where constraints on parameters yield a \enquote{robustness set} of plausible distributions and a corresponding worst-case loss over the set. 
  While the loss under an individual parametric shift can be estimated via reweighting techniques such as importance sampling, the resulting worst-case optimization problem is non-convex, and the estimate may suffer from large variance. 
  For small shifts, however, we can construct a local second-order approximation to the loss under shift and cast the problem of finding a worst-case shift as a particular non-convex quadratic optimization problem, for which efficient algorithms are available.  We demonstrate that this second-order approximation can be estimated directly for shifts in conditional exponential family models, and we bound the approximation error. We apply our approach to a computer vision task (classifying gender from images), revealing sensitivity to shifts in non-causal attributes.
\end{abstract}

\section{Introduction}%
\label{sec:introduction}

Predictive models may perform poorly outside of the training distribution, a problem broadly known as dataset shift \citep{Quionero-Candela2009-hg}.  In high-stakes applications, such as healthcare, it is important to understand the limitations of a model in advance \citep{Finlayson2021-tx}:  
given a model trained on data from one hospital, how will it perform under changes in the population of patients, in the incidence of disease, or in the treatment policy?

In this paper, our goal is to \textbf{proactively} understand the sensitivity of a predictive model to dataset shift, using only data from the training distribution. 
This requires domain knowledge, to specify what type of distributional changes are plausible.
Formally, for a model $f(X)$ trained on data from $\P(X, Y)$, with loss function $\ell(f(X), Y)$, we seek to understand the loss of the model under a set of \textit{plausible} future distributions $\cP$.  We seek to evaluate the worst-case loss over $\cP$,
\begin{equation}\label{eq:maximum_loss}
    \sup_{P \in \cP} \E_P[\ell(f(X), Y)],
\end{equation}
and provide an interpretable description of a distribution $P$ which maximizes this objective. If the value of the worst-case loss is low, this can build confidence prior to deployment, and otherwise, examining the worst-case distribution $P$ can help identify weaknesses of the model.
To illustrate, we use the following running example, inspired by~\citet{Subbaswamy2020-vr}.
\begin{example}[Changes in laboratory testing]\label{ex:lab_testing_rates}
  We seek to classify disease $(Y)$ based on the age ($A$) of a patient, whether a laboratory test has been ordered $(O)$, and test results $(L)$ if a test was ordered. The performance of a predictive model may be sensitive to changes in testing policies, as the \textit{fact that a test has been ordered} itself is predictive of disease. \Cref{fig:lab_testing_shift} (left) gives a plausible causal relationship between variables.  Let $\P(O = 1 |  A, Y) = \sigma(\eta(A, Y))$, where $\sigma$ is the sigmoid function and $\eta(A, Y)$ is the log-odds.  In \cref{fig:lab_testing_shift} (right), we show the loss under a set of new distributions parameterized by $\delta=(\delta_0, \delta_1)$, where we modify $\P_{\delta}(O = 1 |  A, Y) = \sigma(\eta(A, Y) + s(Y; \delta))$ for a \textit{shift function} $s(Y; \delta) = \delta_1 \cdot Y + \delta_0\cdot (1 - Y)$, which modifies the log-odds of testing for both sick and healthy patients. If $\delta_0, \delta_1$ are unconstrained, the worst-case occurs when all healthy patients are tested, and no sick patients are tested.
\end{example}

\begin{figure}[t]
\vspace{-1cm}
\centering
  \begin{subfigure}[t]{0.45\textwidth}
  \centering
    \includegraphics[width=\linewidth]{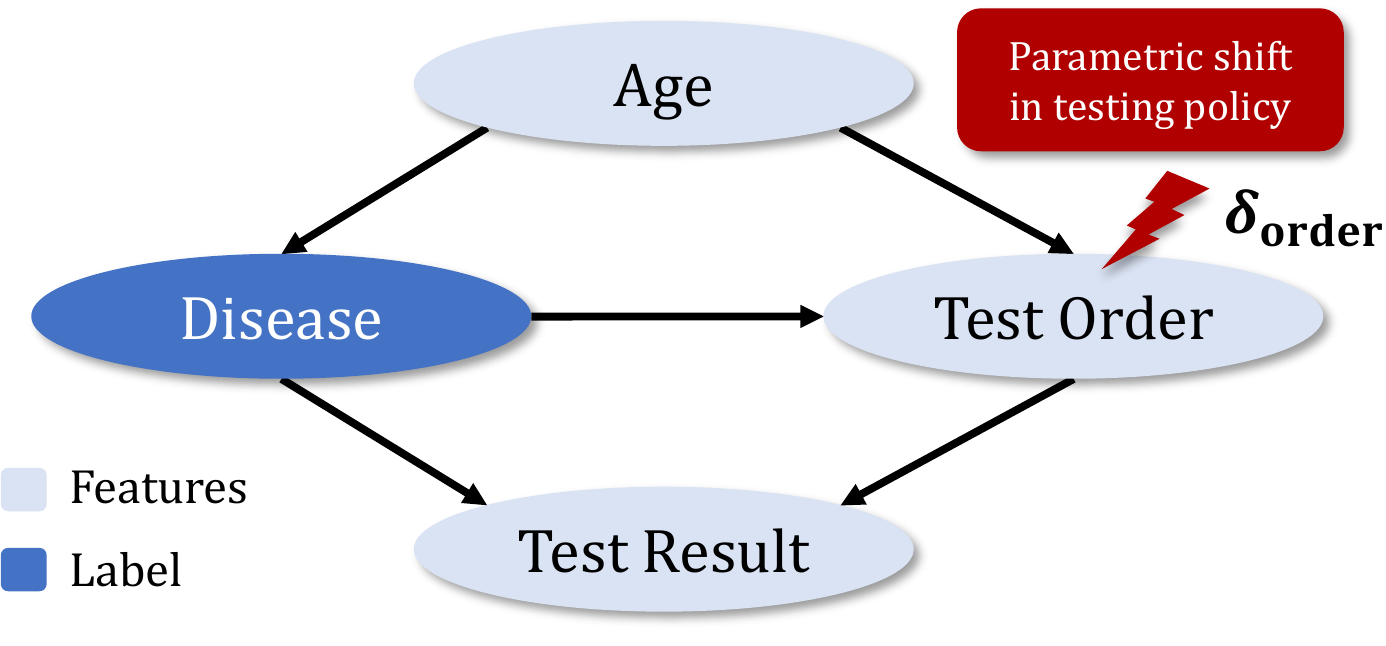}
  \end{subfigure}%
  \begin{subfigure}[t]{0.40\textwidth}
  \centering
    \includegraphics[width=0.7\linewidth]{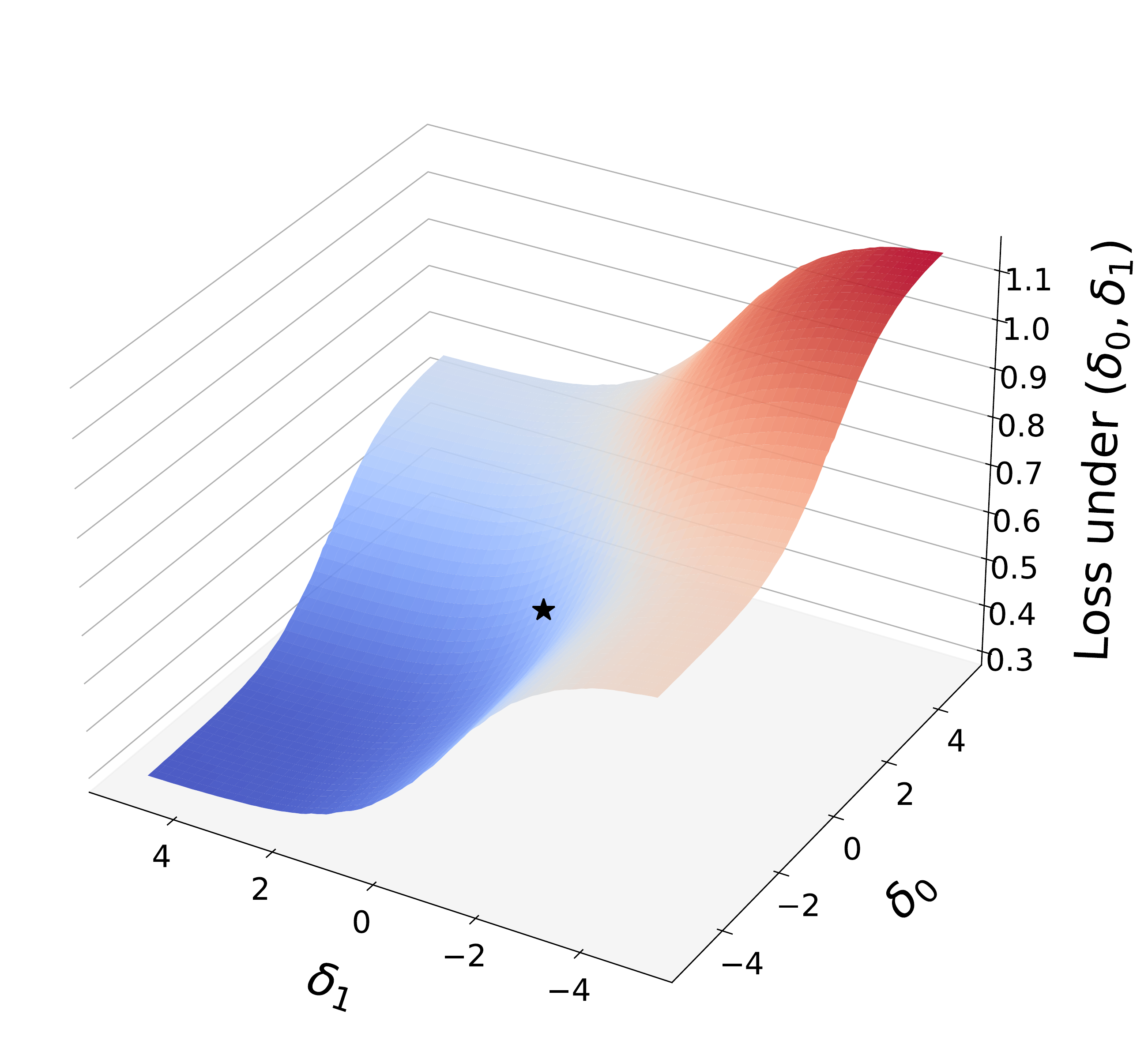}
  \end{subfigure}
  \caption{
    (Left) Causal graph for~\cref{ex:lab_testing_rates}, with a shift in conditional testing rates, parameterized by $\delta_{\text{order}}$. 
  (Right) We illustrate a shift using $s(Y; \delta_{\text{order}}) = \delta_1 \cdot Y + \delta_0 (1 - Y)$, where $\delta_{\text{order}} = (\delta_0, \delta_1)$.  Here we plot the (non-concave) landscape of the expected cross-entropy loss of a fixed model over distributions parameterized by $(\delta_0, \delta_1)$, with the training distribution given as the black star.  Simulation details are given in~\Cref{app:details_of_labtest_example_fig1}.
  }%
\label{fig:lab_testing_shift}
\end{figure}

The first challenge is to define a set of possible distributions $\cP$ such that each distribution $P\in\cP$ satisfies two desiderata: First, they should be \textit{causally interpretable and simple to specify}, without placing unnecessary restrictions on the data-generating process. Second, they should be \textit{realistic}, which often entails bounding the magnitude of the shift. We construct causally interpretable shifts by defining perturbed distributions $\P_{\delta}$ using changes in causal mechanisms, parameterized by a finite-dimensional parameter $\delta$. Our main requirement is that the shifting mechanisms follow a conditional exponential family distribution. For discrete variables, this places no restriction on $\P$: In \cref{ex:lab_testing_rates}, $O$ is binary and the log-odds $\eta(A, Y)$ can be any function of $A, Y$.
We also demonstrate that constraining $\delta$ can ensure that shifts are realistic: The unconstrained worst-case shift in~\cref{ex:lab_testing_rates} is implausible, where all healthy patients (and no sick patients) are tested.
\Cref{eq:maximum_loss} becomes
\begin{equation}\label{eq:maximum_loss_param}
    \sup_{\delta \in \Delta} \E_\delta[\ell(f(X), Y)],
\end{equation}
where $\E_{\delta}$ is the expectation in the shifted distribution $\P_\delta$ and $\Delta$ is a bounded set of shifts.

The second challenge is evaluation of the expected loss under shift, as well as finding the worst-case shift. 
Under our definition of shifts, we show that the test distribution can always be seen as a reweighting of the training distribution, allowing for reweighting approaches, such as importance sampling, to estimate the expected loss under shifts. 
While this is practical for some distribution shifts, for others, importance sampling can lead to extreme variance in estimation. 
Further, finding the worst-case shift using a reweighted objective involves maximization over a non-concave objective (see~\cref{fig:lab_testing_shift}), a problem that is generally NP-hard. 
We derive a second-order approximation to the expected loss under shift, and show how it can be estimated without the use of reweighting. When $\Delta$ is a single quadratic constraint, we can approximate the general non-convex optimization problem in~\cref{eq:maximum_loss_param} with a particular non-convex, quadratically constrained quadratic program (QCQP) for which efficient solvers exist \citep[Section~7]{conn2000trust}. We bound the approximation error of this surrogate objective, and show in experiments that it tends to find impactful adversarial shifts.

Our contributions are as follows: 
\begin{enumerate}[leftmargin=*,align=left]
    \item We provide a novel formulation of robustness sets which are defined using parametric shifts.
    This formulation 
    only require that the shifting mechanisms (i.e., conditional distributions) can be modelled as a conditional exponential family (see~\cref{sec:defining_parametric_robustness_sets}).
    \item We derive a second-order approximation to the expected loss 
    and provide a bound on the approximation error. We show that this translates the general non-convex problem into a particular
    non-convex quadratic program, for which efficient solvers exist (see~\cref{sec:evaluation_of_the_worst_case_loss}).
    \item In a computer vision task, we find that this approach finds more impactful shifts than a reweighting approach,
    while taking far less time to compute, and that the resulting estimates of accuracy are substantially more reliable (see~\cref{sec:experiments}).
\end{enumerate}

\subsection{Related Work}%
\label{sec:related_work}

\textbf{Distributionally robust optimization/evaluation}: Distributionally robust optimization (DRO) seeks to learn models that minimize objectives like~\cref{eq:maximum_loss} with respect to the model \citep{Duchi2018-ju, Duchi2020-at, sagawa2019distributionally}.  We focus on proactive worst-case evaluation of a fixed model, not optimization, similar to \citet{Subbaswamy2020-vr, Li2021-qe}, but we also differ in our \textbf{definition of the set of plausible future distributions $\cP$}, often called an \enquote{uncertainty set} in the optimization literature. Prior work often defines these sets using distributional distances (such as $f$-divergences): For instance, Joint DRO \citep{Duchi2018-ju} allows for shifts in the entire joint distribution (i.e., all distributions in an $f$-divergence ball around $\P(X, Y)$), which may be overly conservative. 
Marginal DRO \citep{Duchi2020-at} considers shifts in a marginal distribution (e.g., $\P(X)$), while assuming that the remaining conditionals (e.g., $\P(Y \mid X)$) are fixed. However, this assumption is not applicable in all scenarios: In~\cref{ex:lab_testing_rates}, for instance, this assumption does not hold for a shift in testing policy.
Conditional shifts are considered in recent work that focuses on evaluation \citep{Subbaswamy2020-vr}, using worst-case conditional subpopulations. However, choosing a plausible size of conditional subpopulation is often non-obvious. In~\cref{sec:comparison_to_subpopulation_shift_app} we give a simple lab-testing example where taking worst-case 20\% conditional subpopulations corresponds to an implausible shift: Healthy patients are always tested, and sick patients never tested.

In contrast, our approach uses explicit parametric perturbations to define shifts, as opposed to distributional distances or subpopulations. In addition, our approach allows for shifts in multiple marginal or conditional distributions simultaneously: In~\cref{ex:lab_testing_rates}, for instance, we could model a simultaneous change in both the marginal distribution of age $\P(A)$, as well as the conditional distribution of lab testing $\P(O \mid A, Y)$, leaving other factors unchanged.

\textbf{Causality-motivated methods for learning robust models:} Several approaches proactively specify shifting causal mechanisms/conditional distributions, and then seek to learn predictors that have good performance under arbitrarily large changes in these mechanisms \citep{subbaswamy2019preventing,Veitch2021-cu,Makar2022-eq,Puli2022-dn}. Other approaches use environments \citep{magliacane2018domain,Rojas-Carulla2018-qe,arjovsky2019invariant} or identity indicators \citep{Heinze-Deml2021-mz} to learn models that rely on invariant conditional distributions.

However, when shifts are not arbitrarily strong, causality-motivated predictors can be overly conservative. In~\cref{ex:lab_testing_rates}, a model that ignores all test-related features (and only uses age as a predictor) is a particularly simple example of a causality-motivated predictor, with invariant risk over changes in testing policy. Closer to our setting is a line of work that considers bounded mechanism changes in linear causal models \citep{rothenhausler2021anchor, oberst2021regularizing}, where estimation of the worst-case loss enables learning of worst-case optimal models.  Our work can be seen as extending this idea to more general non-linear causal models, where we focus on evaluation rather than optimization.

\textbf{Evaluating out-of-distribution performance with unlabelled samples}: A recent line of work has focused on predicting model performance in out-of-distribution settings, where unlabelled data is available from the target distribution \citep{Garg2022-xp,Jiang2022-pn,Chen2021-pd}.  In contrast, our method operates using only samples from the original source distribution, and seeks to estimate the worst-case loss over a set of possible target distributions. 

In~\cref{sec:related_work_app} we give a more detailed discussion of these approaches and others.

\section{Defining parametric robustness sets}%
\label{sec:defining_parametric_robustness_sets}

\textbf{Notation}: Let $\bV$ denote all observed variables, where $(X, Y) \subseteq \bV$ for features $X$ and labels $Y$, and use $\P(\bV)$ to denote the probability density/mass function in the training distribution. We also refer to $\P$ as simply \enquote{the training distribution}. $\E[\cdot]$ and $\cov(\cdot,\cdot)$ refer to the mean and covariance in $\P$, and for a shifted distribution $\P_\delta$ (\cref{def:parametric-shift}) we use $\E_\delta[\cdot]$, $\cov_\delta(\cdot,\cdot)$.  For a random variable $Z$, we use $\cZ$ to denote the space of realizations, and $d_{Z}$ for dimension e.g., $Z \in \cZ \subseteq \R^{d_Z}$.  For a set of random variables $\bV = \{V_1, \ldots, V_d\}$, we use $V_i$ to denote an individual element, and use $\PA_{\cG}(V_i)$ to denote the set of parents in a directed acyclic graph (DAG) $\cG$, omitting the subscript when otherwise clear.

We begin with a general definition of a parameterized robustness set of distributions $\cP$.
\begin{definition}\label{def:parametric-shift}
  A \emph{parameterized robustness set around $\P(\bV)$} is a family of distributions $\cP$ with elements $\P_{\delta}(\bV)$ indexed by $\delta \in \Delta \subseteq \R^{d_{\delta}}$, with $0 \in \Delta$, where $\P_0(\bV) = \P(\bV)$. 
\end{definition}
We give examples shortly that satisfy this general definition. 
To construct such a robustness set, we consider distributions $\P_{\delta}$ that differ from $\P$ in one or more conditional distributions (\cref{assmp:cef_factorization}). 
We require that the relevant conditional distributions can be described by an exponential family.
\begin{definition}[Conditional exponential family (CEF) distribution]\label{def:conditional_exponential_family}
    $\P(W|Z)$ is a conditional exponential family distribution if there exists a function $\eta(Z): \R^{d_Z} \to \R^{d_T}$ such that the conditional probability density (for continuous $W$) or probability mass function (for discrete $W$) is given by
    \begin{equation}\label{eq:def-exponential-family}
        \P(W | Z) = g(W)\exp\left( \eta{(Z)}^\top T(W) - h(\eta(Z)) \right),
    \end{equation}
    where $T(W)$ is a vector of sufficient statistics, $T(W) \in \R^{d_T}$, $g(\cdot)$ specifies the density of a base measure and $h(\eta(Z))$ is the log-partition function.
\end{definition}
\Cref{def:conditional_exponential_family} does not restrict $\P(W |  Z)$ for binary/categorical $W$, and captures a wide range of distributions, including the conditional Gaussian (see~\cref{appsub:examples_of_conditional_exponential_family_models} for other examples). 
\Cref{def:conditional_exponential_family} extends to marginal distributions where $Z = \varnothing$ and $\eta(Z)$ is a constant function.
\begin{example}[continues=ex:lab_testing_rates]
  Suppose the probability of ordering a test $(O)$ depends on age $(A)$ and disease $(Y)$, such that $\P(O = 1|A, Y) = \sigma(\eta(A, Y))$, where $\sigma$ is the sigmoid, and $\eta$ is an arbitrary function.  Here,~\cref{def:conditional_exponential_family} is satisfied with $W = O$, $Z = (A, Y)$, and sufficient statistic $T(O) = O$.  
\end{example}
We now state our main assumption, 
where we distinguish between the terms in the joint distribution of $\P$ that shift, which we will need to model, and those that remain fixed, which we do not. 
\begin{assumption}[Factorization into CEF distributions]\label{assmp:cef_factorization}
  Let $\bW = \{W_1, \ldots, W_m\} \subseteq \bV$ be a \enquote{intervention set} of variables and let
  \begin{equation}\label{eq:factorization-assumption-1}
    \P(\bV) = \underbrace{\prod_{W_i \in \bW} \P(W_i|Z_i)}_{\text{Conditionals that shift}} \underbrace{\prod_{V_j \in \bV\setminus \bW} \P(V_j|U_j)}_{\text{Conditionals we do not model}}
  \end{equation}
  be a factorization, where $Z_i, U_j, V_j \subseteq \bV$ are possibly overlapping (or empty) sets of variables, where $\P(V_j \mid \varnothing) \coloneqq \P(V_j)$.
  For each $W_i$ we assume $Z_i$ is known and $\P(W_i |  Z_i)$ satisfies~\cref{def:conditional_exponential_family}.
\end{assumption}
If $\P(\bV)$ factorizes according to a DAG $\cG$, the factorization in~\cref{assmp:cef_factorization} is always satisfied by $Z_i = \PA_{\cG}(W_i)$. 
While we assume data is generated according to \cref{eq:factorization-assumption-1}, we do not require knowledge of the full distribution, but only the conditionals that shift.
In~\cref{appsub:adding_causal_edges_to_the_graph} we show that we can also consider shifts that extend $Z_i$
to include additional variables, subject to an acyclicity constraint. 
We now define parametric perturbations and give the general form of the robustness sets that we consider in this work, involving simultaneous perturbations to multiple $W_i$.
\begin{definition}[Parameterized shift functions and $\delta$-perturbations]\label{def:exponential_family_shift}\label{def:delta-perturbation}
    Let $s(Z; \delta): \R^{d_Z} \to \R^{d_T}$ be a \textit{parameterized shift function} with parameters $\delta \in \Delta \subseteq \R^{d_{\delta}}$ which is twice-differentiable with respect to $\delta$ and which satisfies $s(Z; 0) = 0$ for all $Z$.  For $\P(W|Z)$ satisfying~\cref{eq:def-exponential-family}, we refer to
    \begin{equation*}
      \P_{\delta}(W | Z) = g(W)\exp\left( \eta_\delta(Z)^\top T(W) - h(\eta_\delta(Z)) \right)
    \end{equation*}
    as a $\delta$-perturbation of $\P(W |  Z)$ with shift function $s(Z;\delta)$, 
    where $\eta_\delta(Z) := \eta(Z) + s(Z; \delta)$.  Note that this differs from~\cref{eq:def-exponential-family} in that $\eta(Z)$ is replaced by $\eta_{\delta}(Z)$.
\end{definition}
\begin{example}[continues=ex:lab_testing_rates]
    A model developer may be concerned about a uniform change in testing rates across all types of patients. This can be modelled by choosing $s(Z; \delta) = \delta$, for $\delta \in \R$, an additive intervention on the log-odds scale. A separate change in testing rates for sick and healthy patients could instead be modeled using $s(Z; \delta) = \delta_0 (1 - Y) + \delta_1 Y$, using $\delta \in \R^2$. This reasoning extends readily to more complex shifts (e.g., allowing for age-specific changes in testing rates, with a non-linear dependence on age), as long as $s(Z;\delta)$ remains a parametric function.
\end{example}
While the shift function $s(Z; \delta)$ is parametric, $\eta(Z)$ is unconstrained in~\cref{def:conditional_exponential_family,def:exponential_family_shift}. 
Note that this formulation 
includes multiplicative shifts $\eta_\delta(Z)=(1+\delta)\eta(Z)$ by letting $s(Z; \delta) = \delta \cdot \eta(Z)$.

\begin{definition}[CEF parameterized robustness set]\label{def:multiple_shift}
  For a distribution $\P$ and intervention set $\bW = \{W_1, \ldots, W_m\} \subseteq \bV$ satisfying~\cref{assmp:cef_factorization}, let each $\P_{\delta_i}(W_i |  Z_i)$ be a $\delta_i$-perturbation (Definition~\ref{def:delta-perturbation}) of $\P(W_i|Z_i)$. Then 
  \[
    \P_\delta(\bV) = \left(\prod_{W_i \in \bW} \P_{\delta_i}(W_i|Z_i)\right) \left( \prod_{V_j \in \bV\setminus \bW} \P(V_j|U_j) \right)
  \]
  is called a $\delta$-perturbation of $\P(\bV)$, and the robustness set $\cP$ consists of all $\P_{\delta}$ for $\delta \in \Delta_1 \times \cdots \Delta_m$.
\end{definition}

To estimate the expected loss under $\P_{\delta}$, we will typically\footnote{As a special case, in~\cref{sec:example-estimate-eta-not-needed}, we show the second-order approximation (\cref{thm:sg-cond-cov-several-shift}) can be estimated in the case of variance-scaled mean-shifts in a conditional Gaussian without estimation of all of $\eta(Z)$.}
need to estimate $\eta(Z_i)$ for each $W_i\in\bW$.  However, we make no distributional assumptions on the remaining variables $\bV\setminus \bW$. This is useful in applications such as computer vision, where we do not need to restrict the generative model of images given attributes (e.g., background, camera type, etc), but can still model the expected loss under changes in the joint distribution of those attributes.

\begin{remark}[Causal Interpretation of Shifts]
  If available, causal knowledge helps identify which factors in the joint distribution are subject to shifts (e.g., $\P(O \mid Y, A)$ in~\cref{ex:lab_testing_rates}), and which remain stable. 
  It is worth noting, however, that our methodology can be used to model any change in distribution that satisfies~\cref{assmp:cef_factorization}, including choices of \enquote{non-causal} factorizations and shifting factors. 
  For example, in the context of~\cref{ex:lab_testing_rates}, we could choose the factorization $\P(Y) \P(O \mid Y) \P(L, A \mid O, Y)$, and model a change only in the conditional $\P(O \mid Y)$ while keeping other factors unchanged. This shift is not interpretable as a change in causal mechanisms: The shifted distribution would imply a change in the marginal distribution of age, which should be unaffected by a real-world change in laboratory testing.  Nonetheless, we can still estimate a worst-case loss over such non-causal shifts in distribution.
  In short, our machinery can model shifts in non-causal conditionals (for example because the causal structure is unknown), though the resulting shifted distribution is not interpretable as a plausible shift in the ground-truth data generating mechanism.
\end{remark}

\section{Evaluation of the worst-case loss}%
\label{sec:evaluation_of_the_worst_case_loss}
For a fixed predictor and loss function, we can use data from $\P(\bV)$ to estimate the expected loss $\E_{\delta}[\ell] \coloneqq \E_{\delta}[\ell(f(X), Y)]$ for a fixed $\delta$, and estimate the worst-case loss over all $\delta$ of bounded magnitude. In~\cref{sec:modelling_shifted_loss_using_ipw}, we show that $\P_{\delta}$ shares support with $\P$, suggesting the use of reweighting estimators. However, these estimators can exhibit high variance for shifts that produce large density ratios (see \cref{ex:lab_test_variance_plots} for an example), and 
maximizing a reweighted objective over $\delta$ is generally a non-convex problem.  In~\cref{sec:approximating_the_shift_for_exp_fam} we derive an approximation to the expected loss under $\P_{\delta}$, yielding a tractable surrogate optimization problem under quadratic constraints such as $\norm{\delta}_2 \leq \lambda$.

\begin{remark}
The methods here can be used with an arbitrary predictor $f$ and loss function $\ell:=\ell(f(X), Y)$.  We do not even require access to the original predictor $f$.  Both methods here simply treat $\ell$ as a random variable in $\P$, for which we have samples from the training distribution.
\end{remark}

\subsection{Modelling shifted losses using reweighting}%
\label{sec:modelling_shifted_loss_using_ipw}
The shifts defined in \cref{sec:defining_parametric_robustness_sets} share common support, with the following density ratio.
\begin{restatable}{proposition}{CommonSupport}\label{prop:ipw-weights-exp-family}
  For any $\P_\delta(\bV), \P(\bV)$ that satisfy~\cref{def:multiple_shift}, $\supp(\P) = \supp(\P_\delta)$ and the density ratio $w_{\delta} \coloneqq \P_{\delta} / \P$ is given by 
    \begin{equation*}
        w_\delta(\bV) = \exp\bigg(\sum_{i=1}^m s_i(Z_i;\delta_i)^\top T_i(W_i)\bigg) \exp\left(\sum_{i=1}^m h(\eta_i(Z_i))-h(\eta(Z_i) + s_i(Z_i; \delta_i))\right).
    \end{equation*}
\end{restatable}
The proof can be found in~\cref{app:proofs}, along with all proofs for all other claims.
\begin{example}[continues=ex:lab_testing_rates]\label{ex:lab_test_rates_ipw_weights}
  Suppose we perturb the probability of ordering a test $O$ given age $A$ and disease $Y$ with shift function $s(Y; \delta) = \delta_0 (1 - Y) + \delta_1 Y$, independently changing the conditional probability of testing for healthy and sick patients. Here, the density ratio is given by
    \begin{equation}\label{eq:ipw-weights-exponential-family}
        w_\delta(O, A, Y) = \exp(s(Y; \delta) \cdot O) \frac{1 + \exp(\eta(A, Y))}{1 + \exp(\eta(A, Y) + s(Y; \delta))}.
    \end{equation}
\end{example}

To model the loss $\E_\delta[\ell]$ using data from $\P$, we can consider an importance sampling (IS) estimator \citep{horvitz1952generalization,SHIMODAIRA2000227}, observing that $\E_\delta[\ell] = \E[w_\delta(\bV) \cdot \ell]$.  This requires estimation of the density ratio $w_{\delta}(\bV)$, and (given a sample $\{\bV^j\}_{j=1}^n$ from $\P$) yields the estimator
\begin{equation}\label{eq:ipw-estimate-of-mean}
  \E_\delta[\ell] \approx \ipw := \frac{1}{n} \sum_{j=1}^n \hat{w}_\delta(\bV^j) \ell(\bV^j).
\end{equation}
\Cref{eq:ipw-estimate-of-mean} can have high variance when density ratios are large, and maximizing this equation with respect to $\delta$ is a general non-convex optimization problem, which is generally NP-hard to solve.

\subsection{Approximating the shifted loss for exponential family models}%
\label{sec:approximating_the_shift_for_exp_fam}
We now propose an alternative approach for approximating the loss $\E_{\delta}[\ell]$. Recalling that $\P_{\delta = 0} = \P$, we use a second-order Taylor expansion around the training distribution
\begin{equation}\label{eq:taylor-approx}
    \E_{\delta}[\ell] \approx \E[\ell] + \delta^\top \sg^1 + \tfrac{1}{2}\delta^\top \sg^2\delta,
\end{equation}
where $\E[\ell]$ denotes the loss in the training distribution and $\sg^1, \sg^2$ are defined as follows.
\begin{definition}[Shift gradient and Hessian]\label{def:shift-gradient}
  For a parametric shift satisfying \Cref{def:parametric-shift} where $\delta\mapsto\E_{\delta}[\ell]$ is twice-differentiable, we denote the \emph{shift gradient} $\sg^{1}$ and \emph{shift Hessian} $\sg^2$ as 
  \begin{align*}
    \sg^1 \coloneqq \nabla_\delta \E_\delta[\ell]\big\vert_{\delta = 0} & &\text{and} & & \sg^2 \coloneqq \nabla_\delta^2 \E_\delta[\ell]\big\vert_{\delta=0}.
  \end{align*}
\end{definition}
\Cref{eq:taylor-approx} is a local approximation of the loss, whose approximation error we bound in~\cref{thm:taylor-approximation-bound}, with smaller approximation error for smaller shifts.\footnote{In~\Cref{sec:linear-model-exact}, we give an example of a linear-Gaussian generative model where this second-order expansion is exact, corresponding to the setting of Anchor Regression~\citep{rothenhausler2021anchor}.} For $\P_{\delta}$ satisfying \Cref{def:multiple_shift}, $\sg^1$ and $\sg^2$ can be computed as expectations in the training distribution, without estimation of density ratios. Recall that the conditional covariance is given by $\cov(A, B|C) := \E[(A - \E[A|C])(B - \E[B|C])|C]$.
\begin{restatable}[Shift gradients and Hessians as covariances]{theorem}{SgConvCovSeveralShift}\label{thm:sg-cond-cov-several-shift}
  Assume that $\P_{\delta}, \P$ satisfy~\cref{def:multiple_shift}, with intervened variables $\bW = \{W_1, \ldots, W_m\}$ and shift functions $s_i(Z_i; \delta_i)$, where $\delta = (\delta_1, \ldots, \delta_m)$.  Then the shift gradient is given by $\sg^1 = (\sg_{1}^1, \ldots, \sg_{m}^1) \in \R^{d_{\delta}}$ where
    \begin{equation*}
        \sg_{i}^1 = \E\left[D_{i,1}^\top\cov\bigg(\ell,\, T_i(W_i)\bigg| Z_i\bigg)\right],
    \end{equation*}
  and the shift Hessian is a matrix of size $(d_{\delta} \times d_{\delta})$, where the $(i,j)$th block of size $d_{\delta_i} \times d_{\delta_j}$ equals
    \begin{equation*}
        \{\sg^2\}_{i,j} = 
        \begin{cases}
            \E\left[D_{i,1}^\top\cov\left(\ell,\, \epsilon_{T_i|Z_i}\epsilon_{T_i|Z_i}^\top|Z_i\right)D_{i,1}\right]- \E\left[\ell \cdot D_{i,2}^\top \epsilon_{T |  Z} \right] & i = j \\
            \cov(\ell,\,\, D_{i, 1}^\top\epsilon_{T_i|Z_i}\epsilon_{T_j|Z_j}^\top D_{j, 1}) & i \neq j,
        \end{cases}
    \end{equation*}
    where $D_{i,k} := \nabla^k_{\delta_i} s_i(Z_i; \delta_i)|_{\delta=0}$, is the gradient of the shift function for $k = 1$, and the Hessian for $k = 2$. Here, $T_i(W_i)$ is the sufficient statistic of $\P(W_i|Z_i)$ and $\epsilon_{T_i|Z_i} := T_i(W_i) - \E[T(W_i)|Z_i]$.
\end{restatable}
\Cref{thm:sg-cond-cov-several-shift} handles arbitrary parametric shift functions in multiple variables, but for simple shift functions in a single variable, the notation simplifies substantially, as we show in~\Cref{thm:sg-cond-cov}.
\begin{restatable}[Simple shift in a single variable]{corollary}{SgCondCov}\label{thm:sg-cond-cov}
    Assume the setup of~\Cref{thm:sg-cond-cov-several-shift}, restricted to a shift in a single variable $W$, and that $s(Z;\delta)=\delta$. Then $D_1 = 1$, $D_2 = 0$, and
    \begin{equation*}
        \sg^1 = \E\left[\cov\bigg(\ell, T(W) \bigg| Z\bigg)\right] \qquad\text{and}\qquad
        \sg^2 = \E\left[\cov \bigg(\ell, \epsilon_{T|Z}\epsilon_{T|Z}^\top \bigg| Z\bigg)\right],
    \end{equation*}
    where $T(W)$ is the sufficient statistic of $W$ and $\epsilon_{T|Z} := T(W) - \E[T(W)|Z]$.
\end{restatable}
\begin{example}[continues=ex:lab_testing_rates] 
  Suppose that age ($A$) follows a normal distribution with mean $\mu$ and variance $\sigma^2$, and consider a shift in the mean (without changing lab testing). We can parameterize $\P(A)$ as an exponential family with parameter $\eta = \mu / \sigma$ and sufficient statistic $T(A) = A / \sigma$.  Here, $s(\delta) = \delta$ implies a shift in the mean of $\delta$ standard deviations $\eta_{\delta} = \eta + s(\delta) = (\mu + \sigma \delta) / \sigma$, and we can write that 
        $\sg^1 = \cov\left(\ell, A\right) / \sigma$ and $\sg^2 = \cov\left(\ell, (A- \E[A])^2\right) / \sigma^2$.
\end{example}
To estimate the shift gradient and Hessian from a sample from $\P$, for each $i=1,\dots, m$ we fit models $\hat{\mu}_{\ell}(Z_i) \approx \E[\ell |  Z_i]$ and $\hat{\mu}_{W_i}(Z_i) \approx \E[T_i(W_i) |  Z_i]$ and compute residuals on these predictions, which permits estimation of the gradient/Hessian as a sample average of residuals. A detailed treatment is given in~\cref{appsec:algorithm_calculation_theorem1}.
Using estimates of the gradient and Hessian, we estimate the expected loss as
\begin{equation}\label{eq:taylor-approx-finie-sample}
  \E_\delta[\ell] \approx \taylor := \hat{\E}[\ell] + \delta^\top \hat{\sg}^1 + \frac{1}{2}\delta^\top \hat{\sg}^2 \delta.
\end{equation}
Here, there are two sources of error: Finite-sample error, due to the estimates of $\sg^1,\sg^2$, as well as approximation error. The latter is bounded by the norm of $\delta$ and a term that depends on the covariance between the loss and the deviations of the sufficient statistic from its shifted mean.
\begin{restatable}{theorem}{TaylorApproxBound}\label{thm:taylor-approximation-bound}
    Assume that $\P_\delta, \P$ satisfy the conditions of~\Cref{thm:sg-cond-cov-several-shift}, with a shift in a single variable $W$, where $s(Z; \delta) = \delta$. Let $\taylorPop$ be the population Taylor estimate (\cref{eq:taylor-approx}) and let $\sigma(M)$ denote the largest absolute value of the eigenvalues of a matrix $M$. Then 
    \begin{align*}
        \bigg|\E_\delta[\ell] - \taylorPop\bigg| \leq \tfrac{1}{2} \sup_{t\in[0,1]} \sigma\bigg(\cov_{t\cdot\delta}(\ell, \epsilon_{t\cdot\delta, T|Z}\epsilon_{t\cdot\delta, T|Z}^\top) - \cov(\ell, \epsilon_{0, T|Z}\epsilon_{0, T|Z}^\top)\bigg) \cdot \|\delta\|^2,
    \end{align*}
    where $T(W)$ is the sufficient statistic of $W|Z$ and $\epsilon_{t\cdot \delta, T|Z} = T(W|Z) - \E_{t\cdot \delta}[T(W|Z)]$.
\end{restatable}
To build intuition, in~\cref{subsec:example-of-bound} we give a scenario where this bound can be simplified. In particular, we consider a \enquote{covariate shift} setting \citep{Quionero-Candela2009-hg}
where $X$ is standard Gaussian, $Y = f_0(X) + \epsilon$ with a noise term independent of $X$ and we consider a shift $\delta$ in the mean of $X$. 
When evaluating a predictor $f(X)$ with the loss $\ell$ being the squared error, the bound in~\cref{thm:taylor-approximation-bound} depends on how the modelling error $g(X) = f_0(X) - f(X)$ behaves over the domain.  In particular, the bound scales as the supremum (over $t \in [0, 1]$) of $\sqrt{\var(g(X + t\cdot \delta)^2 - g(X)^2)}$.  As a simple corollary, if our predictor is off by an additive constant factor, $f = f_0 + C$, then the bound is zero, and the approximation is exact for any $\delta$.  On the other hand, if the squared modelling error $g(X)^2$ at one point $X$ tends to be a poor predictor of the squared modelling error at another point $X + t \cdot \delta$, then this variance will be large, and the approximation will be loose.

In exchange for considering a second-order approximation of the loss, we gain two benefits: Variance reduction and tractable optimization.  First, the variance of $\taylor$ is $O(\norm{\delta}^4)$ for large $\norm{\delta}$, while the variance of $\ipw$ can be much larger: We give a simple case in~\Cref{sec:variance-ipw-vs-taylor} where $\var(\taylor) = O(\delta^4)$ while $\var(\ipw) = O(\delta^2 \exp(\delta^2))$. Second, maximizing $\taylor$ over the set $\norm{\delta} \leq \lambda$ can be solved in polynomial time by exploiting the quadratic structure, while maximizing $\ipw$ over the constraints is generally hard, and may be infeasible in high dimensions. 

\subsection{Identifying worst-case parametric shifts}%
\label{sec:identifying_worst_case_parametric_shifts}

For $\lambda > 0$, we can locally approximate the worst-case loss over all distributions $\P_\delta$ where $\norm{\delta}_2 \leq\lambda$ by finding the worst-case loss in the Taylor approximation
\begin{equation}\label{eq:second-order-maximization}
  \sup_{\|\delta\|_2\leq \lambda} \E[\ell] + \delta^\top \sg^1 + \tfrac{1}{2}\delta^\top \sg^2 \delta.
\end{equation}
Since $\sg^2$ is generally not negative definite, the maximization objective is non-concave.  However, this particular problem is an instance of the `trust region problem'\footnote{Not to be confused with the `trust region \emph{method}', which repeatedly solves the trust region \emph{problem}.} which is well-studied in the optimization literature~\citep{conn2000trust}, and can be solved in polynomial time by specialized algorithms (see \citet[Section~8.1]{polik2007survey} for an example). This follows from the fact that strong duality holds, so that the optimal solution $\delta^*$ can be characterized in terms of the Karush-Kuhn-Tucker conditions \citep[Section~5.2]{boyd2004convex}. 
For this problem, we use the \texttt{trsapp} routine from NEWUOA \citep{Powell2006-rz}, as implemented in the python package \texttt{trustregion}.
Depending on the application and prior knowledge, one may choose constraint sets that differ from $\|\delta\|\leq\lambda$. 
In particular, the strong duality of \Cref{eq:second-order-maximization} also holds when $\|\delta\|_2 \leq \lambda$ is replaced by any single quadratic constraint $\delta^\top A \delta + \delta^\top b \leq \lambda$, allowing for e.g., larger shifts in some directions than in others.

\section{Experiments}%
\label{sec:experiments}
\subsection{Illustrative example: Laboratory testing}%
\label{sec:illustrative_example_worst_case_parametric_shifts}

\begin{wrapfigure}{r}{0.2\textwidth}
\begin{center}
\vspace{-0.4cm}
\begin{tikzpicture}[
  obs/.style={circle, draw=gray!90, fill=gray!30, very thick, minimum size=5mm}, 
  uobs/.style={circle, draw=gray!90, fill=gray!10, dotted, minimum size=5mm}, 
  shiftobs/.style={circle, draw=gray!90, fill=red!30, very thick, minimum size=5mm}, 
  bend angle=30]
  \node[obs] (Y) {$Y$};
  \node[shiftobs] (O) [right=0.5cm of Y] {$O$} ;
  \node[obs] (L) [below=0.5cm of Y]  {$L$};
  \draw[-latex, thick] (Y) -- (O);
  \draw[-latex, thick] (Y) -- (L);
  \draw[-latex, thick] (O) -- (L);
\end{tikzpicture}
\vspace{-0.2cm}
\end{center}
\caption{}%
\label{fig:causal_graph_labtest}
\vspace{0.2cm}
\end{wrapfigure}
To build intuition, we illustrate our method in a simple generative model, similar to~\cref{ex:lab_testing_rates}, where lab tests are more likely to be ordered $(O)$ for sick patients $(Y)$, and lab values $(L)$ are predictive of $Y$.
\begin{align*}
  Y &\sim \Ber(0.5) & O |  Y &\sim \Ber(\sigma(\alpha + \beta Y)) & L |  (Y, O = 1) &\sim \cN(\mu_y, 1)
\end{align*}
where $\mu_1 = 0.5, \mu_0 = -0.5$, and we initialize with $\alpha = -1$, $\beta = 2$, so that $\P(O = 1 |  Y = 0) \approx 0.27$ and $\P(O = 1 |  Y = 1) \approx 0.73$, and the marginal probability of test ordering is $\P(O = 1) = 0.5$.  When $O = 0$, we set $L$ to a dummy value of $L = 0$.  The underlying causal graph is given in~\cref{fig:causal_graph_labtest}. The predictive model $f(O, L)$ is trained on data from $\P$ to predict $Y$ using all available features. If lab tests are not available ($O = 0$), this model predicts $Y$ based on the observed likelihood of $Y$ given $O = 0$, and otherwise uses a logistic regression model trained on cases where $O = 1$ in the training data.

\begin{figure}[t]
  \centering
  \begin{subfigure}[t]{0.45\textwidth}
    \centering
      \includegraphics[width=0.65\textwidth]{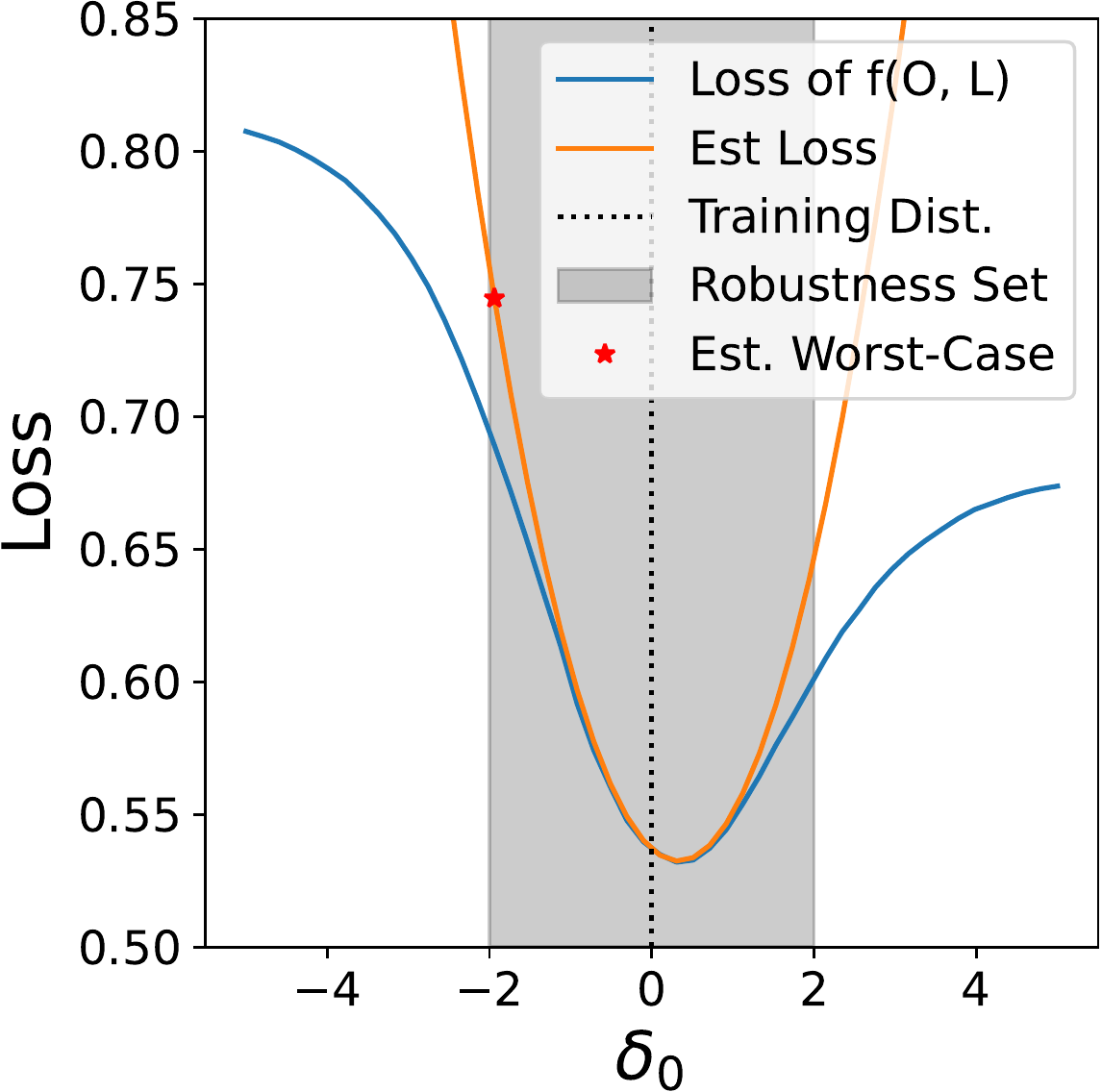}
  \end{subfigure}%
  \begin{subfigure}[t]{0.45\textwidth}
    \centering
      \includegraphics[width=0.65\textwidth]{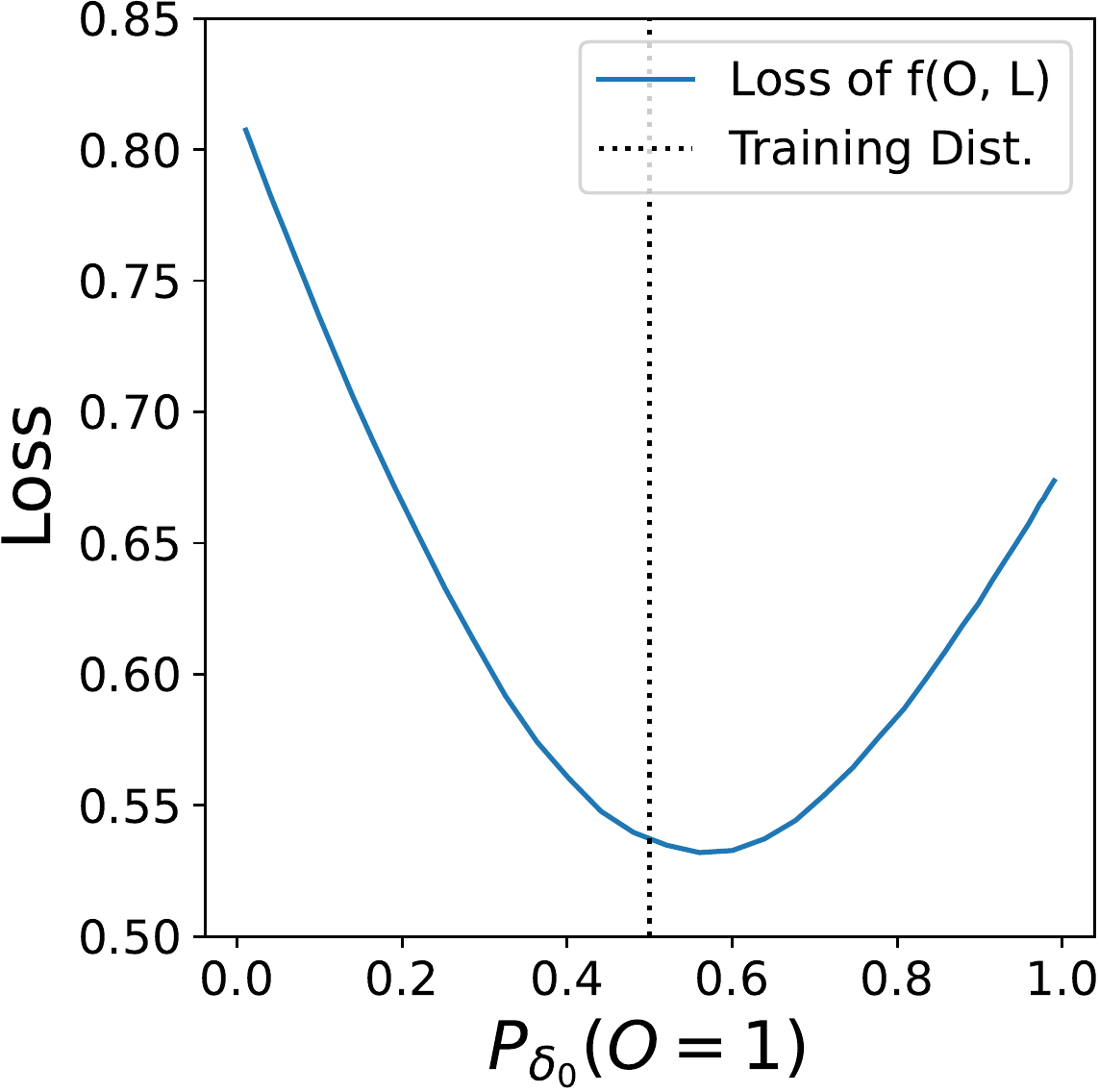}
  \end{subfigure}
  \caption{The blue line gives the (unobserved) cross-entropy loss under parametric shifts, plotted with respect to the parameter $\delta_0$ (left) and the resulting change in the marginal laboratory testing rate (right). We also provide the quadratic approximation (orange line), estimated using validation data, and the predicted worst-case shift (red star) for $\abs{\delta_0} < 2$ (region in grey).}%
  \label{fig:labtest_loss_under_intercept_shift}
  \vspace{-0.3cm}
\end{figure}

\textbf{Defining a shift function:} $\P(O |  Y)$ is a conditional exponential family with $\eta(Y) = \alpha + \beta Y$.  We consider the shift function $s(Y; \delta) = \delta_{0} + \delta_{1} Y$, where $\delta_0$ models an overall change in testing rate, and $\delta_1$ models an additional change in the likelihood of testing sick $(Y = 1)$ patients.

\textbf{Estimating the impact of shift using quadratic approximation:} To start, we keep $\delta_1 = 0$ fixed and vary only $\delta_0$, which uniformly increases or decreases testing.  In~\cref{fig:labtest_loss_under_intercept_shift}, we show the ground-truth cross-entropy loss of $f(O, L)$ under perturbed distributions $\P_{\delta_0}$. 
We observe that the \textbf{direction} of the shift matters: In~\cref{fig:labtest_loss_under_intercept_shift}, the model performance slightly increases under a small increase in testing rates, but degrades if testing increases too much; moreover, the loss under shift is generally asymmetric, as a decrease hurts more than an increase in testing. 
In~\cref{fig:labtest_loss_under_intercept_shift} (left), we demonstrate the use of the quadratic approximation described in~\cref{sec:approximating_the_shift_for_exp_fam}.  For illustration, we consider a robustness set of $\delta_0 \in [-2, 2]$, and see that the predicted worst-case shift coincides with the actual worst-case shift, and that the quadratic approximation is accurate for smaller values of $\delta$.  

In~\cref{sec:comparison_to_subpopulation_shift_app}, we allow both $\delta_0$ and $\delta_1$ to vary, and compare our approach to that of worst-case $(1 - \alpha)$ conditional subpopulation shifts \citep{Subbaswamy2020-vr}.  In the context of this example, we demonstrate that for any $1 - \alpha < 0.27$, the worst-case conditional subpopulation loss is achieved by having all healthy patients get tested, and no sick patients get tested.  We contrast this with an iterative approach to designing constraints that is made possible by considering parametric shifts, where end-users can restrict the degree to which the shift differs across sick and healthy populations.

\subsection{Detecting sensitivity to non-causal correlations}\label{subsec:celeba-gan}
A predictive model may pick up on various problematic dependencies in the data that may not remain stable under dataset shift.
To understand the impact of these dependencies, a model user may wish to understand which changes in distribution pose the greatest threats to model performance, and to measure the impact of these changes.
To illustrate this use-case, we make use of the CelebA dataset \citep{liu2015faceattributes}, which contains
images of faces and binary attributes (e.g., glasses, beard, etc.) encoding several features whose correlations may be unstable (e.g., the relation between gender and being bald). We consider the task of predicting gender ($Y$) from images of faces ($X$), and assess sensitivity to a shift in the distributions of attributes ($\bW$).\footnote{We do not endorse gender classification as an inherently worthwhile task. Nonetheless, gender classification is commonly studied in the context of understanding the implicit biases of machine learning models \citep{Buolamwini2018-yx,Schwemmer2020-qs}, and we consider the task with that context in mind.} 
\begin{figure}[t]
\centering
\scalebox{0.65}{\input{figures/celeba-graph}}
\caption{Causal graph over attributes in the synthetic CelebA dataset, where lightning bolts indicate changes in mechanisms.  All of these attributes are causal parents of the image $X$ (not shown here), which is generated by a GAN conditioned on these attributes.}\label{fig:celeba-scm}
\vspace{-0.4cm}
\end{figure}
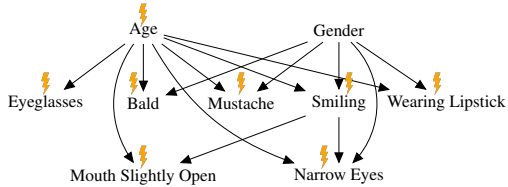

\textbf{Setup}: To obtain ground-truth shifts in distribution, we generate synthetic datasets of faces using CausalGAN \citep{kocaoglu2018causalgan}, trained on the CelebA data. We simulate attributes following the causal graph in \Cref{fig:celeba-scm}, and then simulate images from the GAN conditioned on those attributes. We draw a training sample from this distribution $\P$, and fit a gender classifier $f(X)$ using the image data alone, by finetuning a pretrained ResNet50 classifier \citep{hu2018does}. Each attribute $W_i$ is binary, so we consider shifts in the log-odds $\eta_i(Z_i)$ of each attribute $W_i$ given parents $Z_i$. Here, we use a maximally flexible shift function $s_i(Z_i;\delta_i) = \sum_{z\in\cZ_i} \delta_{i, z} \1{Z_i=z}$, such that for $Z_i \in \{0, 1\}^{k}$ there are $2^k$ parameters. Across all intervened variables, $\delta \in \R^{31}$. 
Due to the synthetic nature of our setup, we can simulate from $\P_\delta(X, \bW, Y)$ to evaluate the ground-truth impact of this shift, simulating first from the shifted attribute distribution, and then simulating images from the GAN conditional on those attributes.  We use the 0/1 loss $\ell = \1{f(X) \neq Y}$, and constrain $\delta$ by $\|\delta\|_2 \leq \lambda = 2$.

\textbf{Comparing importance sampling and Taylor across multiple simulations}: We simulate $K=100$ validation sets 
from $\P$, in each estimating the worst-case shifts $\delta_{\text{Taylor}}$ (via the approach in \cref{sec:identifying_worst_case_parametric_shifts}) and $\delta_{\text{IS}}$, where the latter corresponds to 
minimizing $\ipw$ using a standard non-convex solver from the \texttt{scipy} library \citep{scipy2020}. We simulate ground truth data 
from $\P_{\delta_{\text{IS}}}$ and $\P_{\delta_{\text{Taylor}}}$, to compare the two shifts.
First, we demonstrate that the Taylor approach finds more impactful shifts, 
when searching over the space of small, bounded shifts considered here. 
In~\cref{tab:celebA-table} (right), we compare the average drop in accuracy using the Taylor shifts (3.8\%) 
and the IS shifts (2.2\%). 
In \cref{fig:compare-ipw-taylor-optim} (right) we plot the differences in test accuracy $\E_{\delta_{\text{Taylor}}}[\1{f(X) = Y}]-\E_{\delta_{\text{IS}}}[\1{f(X) = Y}]$, where the Taylor approach finds a more impactful shift in $96\%$ of cases.
Second, the Taylor approach has an average run-time of $0.01s$, versus $2.14s$ for the IS approach. 
Third, when \textbf{only} used to evaluate the shift $\delta_{\text{Taylor}}$, the IS estimator is comparable to the Taylor estimator, with a near-identical average bias (shown in Table 1 (right)) and RMSE (0.0191 and 0.0192 respectively). 
Finally, however, in~\cref{tab:celebA-table} (right) we observe that $\hat{E}_{\delta_{\text{IS}}, \text{IS}}$ is strongly biased in predicting $\E_{\delta_{\text{IS}}}$, yielding a mean absolute prediction error (MAPE) of $0.069$ (not shown in the table).  This can be contrasted with a MAPE of $0.015$ when using $\hat{E}_{\delta_{\text{Taylor}}, \text{Taylor}}$ to predict $\E_{\delta_{\text{Taylor}}}$. 
This may suggest that optimizing the IS objective is prone to \enquote{overfitting}, choosing a sub-optimal $\delta$ from a region of the search space that has high variance.
Here, where $\lambda = 2$, the drop in accuracy is relatively mild for the shifts found by both approaches. In~\cref{sec:impact_of_changing_lambda} we show that larger values of $\lambda$ correspond to more substantial drops in accuracy (e.g., an average drop of 23\% for $\lambda = 8$ using the Taylor approach).

\begin{table}[t]
    \centering
    \caption{%
      (Left) Top 5 components (by magnitude) of the example shift vector $\delta \in \R^{31}$ where $\P$ and $\P_\delta$ denote conditional probabilities. The full example shift vector can be found in~\cref{subsec:full-table-1}.
      (Right) Taylor and IS estimates vs.\ true accuracy for the $\delta_{\text{Taylor}}$ found by the Taylor approach, and IS estimate vs.\ true accuracy for the $\delta_{\text{IS}}$ found by the IS approach. Averages are taken over 100 simulations.\label{tab:celebA-table}
      }
      \vspace{0.2cm}
      \scalebox{0.65}{\input{tables/table1_left}}\qquad
      \scalebox{0.65}{\input{tables/table1_right}}
\end{table}
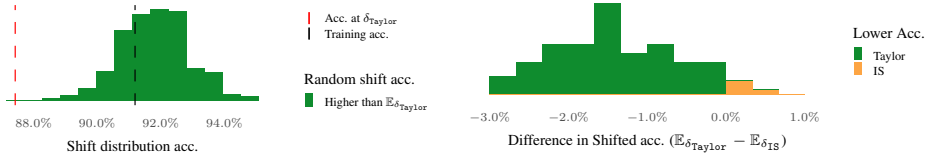
\begin{figure}[!t]
    \centering
    \begin{subfigure}{0.45\linewidth}
        \centering
        \scalebox{0.75}{\input{figures/figure5_left}}
    \end{subfigure}
    \begin{subfigure}{0.45\linewidth}
        \centering
        \scalebox{0.75}{\input{figures/figure5_right}}
    \end{subfigure}
    \caption{(Left) Model accuracy at randomly drawn shifts. (Right) Difference in accuracy in the worst-case shifts identified by Taylor and importance sampling approaches. The Taylor method identifies a more adversarial shift than importance sampling in $96\%$ of simulations (green).\label{fig:compare-ipw-taylor-optim}}
\end{figure}

\textbf{Examining a single shift}: To illustrate the type of shift found by our approach, we consider the $\delta_{\text{Taylor}}$ (over the $K$ runs) which yields the $\P_{\delta}$ with median test accuracy.
We display the largest components of that $\delta$ in \cref{tab:celebA-table} (left). Among others, this shift entails a $5\%$ increase in the probability of an older woman being bald, and a 5\% decrease in the probability of a young woman wearing lipstick. This suggests that the learned classifier $f$ relies on these associations in the images for prediction. We validate that this shift leads to a decrease in accuracy of around $3.8\%$, using simulated data from $\P_\delta$. To validate that this drop in accuracy is a non-trivial occurrence, we simulate $K=400$ random shifts $\delta_k$ where $\norm{\delta_k} = \lambda$ and evaluate the model accuracy in $\P_{\delta_k}$
(\cref{fig:compare-ipw-taylor-optim}, left). As expected, the chosen $\delta$ yields a lower accuracy (red line) than all of the random shifts.

\section{Conclusion}%
\label{sec:conclusion}
We argue for considering parametric shifts in distribution, to evaluate model performance under a set of changes that are interpretable and controllable. For parametric shifts in conditional exponential family distributions, we derive a local second-order approximation to the loss under shift.  This approximation enables the use of efficient optimization algorithms (to find the worst-case shift), and empirically provides realistic estimates of the resulting loss. In a computer vision task, this approach finds more impactful shifts (in far less time) than optimizing a reweighted objective, and the estimates of shifted accuracy under the chosen shift are substantially more reliable.

Of course, our method is not without limitations. Our definition of parametric shifts and resulting approximation relies on the relevant mechanisms $\P(W |  Z)$ being a conditional exponential family, and that the relevant variables are observed.  As illustrated in our experiments, this can be used to model changes in the causal relationships \textbf{between} attributes of an image, but does not immediately extend to modelling changes in the distribution of images given a fixed set of attributes.
As with any method that provides worst-case evaluation, there is potential for misuse and false confidence: If the specified shifts fail to capture important real-world changes, the resulting worst-case loss may be overly optimistic and misleading. Even if used correctly, our approach examines a narrow measure of model performance, and a small worst-case error should not be used to claim that a model is free of problematic behavior. For example, implicit dependence on certain attributes (e.g., race in medical imaging \citep{Banerjee2021-gy}) may be problematic based on ethical grounds, even if it does not lead to major issues with predictive performance under small shifts in distribution. 

\subsection*{Acknowledgements}%
\label{sec:acknowledgements}
We thank Jonas Peters, Tommi Jaakkola, Chandler Squires, and Stefan Hegselmann for helpful feedback and discussion, and Irene Chen and Christina X Ji for providing comments on an earlier draft.  MO and DS were supported in part by Office of Naval Research Award No. N00014-21-1-2807. NT was supported by a research grant (18968) from VILLUM FONDEN. 

\clearpage
\bibliographystyle{abbrvnat}
\bibliography{main.bbl}
\clearpage 

\appendix
\section*{Appendix}%
\label{sec:appendix}

This appendix is structured as follows:
\begin{itemize}
  \item In~\cref{app:details_of_labtest_example_fig1}, we provide details on the synthetic lab testing example, including how we generate the loss landscape in~\cref{fig:lab_testing_shift} (right).
  \item In~\cref{app:users_guide_to_defining_parametric_shifts}, we provide a \enquote{user's guide} to defining and interpreting parametric shifts, including worked examples for many common conditional distributions, as well as guidance on how to define and interpret the shift functions $s(Z; \delta)$.
  \item In~\cref{sec:considerations_for_evaluation_of_the_worst_case_loss}, we provide additional details on the worst-case optimization problem, as well as comparisons of the reweighting-based approach to the Taylor approximation approach.  We also demonstrate that the quadratic approximation is exact, for particularly simple structural causal models.
  \item In~\cref{sec:comparison_to_subpopulation_shift_app}, we compare our approach to that of worst-case conditional subpopulation shifts, in the context of a simpler laboratory testing example where we can explicitly compute the worst-case conditional subpopulations. Here, we demonstrate that our approach can capture more realistic intuition regarding which shifts are plausible in practice.
  \item In~\cref{app:celeba}, we give additional experimental details, as well as illustrative samples from the generative model, for the CelebA experiment described in~\cref{sec:experiments}.
  \item In~\cref{sec:related_work_app}, we give an extended discussion of related work.
  \item In~\cref{app:proofs}, we give proofs for all the results in the main paper.
\end{itemize}

\section{Details of \texorpdfstring{\cref{fig:lab_testing_shift}}{Figure 1}}%
\label{app:details_of_labtest_example_fig1}

In~\cref{fig:lab_testing_shift} (right), we consider the following, artificial, generative model, which resembles the setup in~\cref{sec:illustrative_example_worst_case_parametric_shifts}, but with the addition of age as a continuous variable.
\begin{align*}
    \text{Age}&\sim\cN(0, 0.5^2) \\
    \P(\text{Disease} = 1 |  \text{Age}) &= \operatorname{sigmoid}(0.5 \cdot \text{Age} - 1) \\
    \P(\text{Order} = 1 |  \text{Disease, Age}) &= \operatorname{sigmoid}(2 \cdot \text{Disease} + 0.5 \cdot \text{Age} - 1)\\
    \text{Test Result} |  \text{Order} = 1, \text{Disease} &\sim 
    \cN(-0.5 + \text{Disease}, 1)
\end{align*}
where if $\text{Order} = 0$, the test result is a placeholder value of zero.  In~\cref{fig:lab_testing_shift} (right), we consider a simple predictive model: If lab tests are not available ($\text{Order} = 0$), this model predicts disease based on an unregularized logistic regression model, which uses age to predict disease.  If a lab test is available, then it uses both age and the lab test for prediction.  This model is trained on $100{,}000$ samples from the training distribution.  To construct the loss landscape shown in~\cref{fig:lab_testing_shift} (right), we first observe that 
\begin{equation*}
  \P(O = 1 |  \text{Disease}, \text{Age}) = \operatorname{sigmoid}(\eta(\text{Disease}, \text{Age})),
\end{equation*}
where
\[
  \eta(\text{Disease},\text{Age}) = 2 \cdot \text{Disease} + 0.5\cdot \text{Age} - 1.
\]
We construct shifts using the shift function $s(\text{Disease}, \text{Age}; \delta) = \delta_0 \cdot (1 - \text{Disease}) + \delta_1 \cdot \text{Disease}$, and for a grid of values for $(\delta_0, \delta_1) \in {[-5, 5]}^2$ we consider perturbed distributions with a different conditional distribution of testing,
\begin{align*}
  \P_{\delta}(O = 1 |  \text{Disease, Age}) &= \operatorname{sigmoid}\bigg(\eta(\text{Disease}, \text{Age}) + \delta_0\cdot (1 - \text{Disease}) + \delta_1\cdot \text{Disease}\bigg),
\end{align*}
but where all other parts of the generative model are fixed.  For each value of $(\delta_0, \delta_1) \in {[-5, 5]}^2$, we draw $10{,}000$ samples from the corresponding distribution, and compute the negative log-likelihood of the original predictive model under this new distribution.  The resulting surface is plotted in~\cref{fig:lab_testing_shift} (right).

\section{A user's guide to defining parametric shifts}%
\label{app:users_guide_to_defining_parametric_shifts}

In this section, we discuss practical considerations in designing parametric shift functions for different distributions.  
\begin{itemize}
  \item In~\cref{appsub:examples_of_conditional_exponential_family_models}, we give examples of conditional exponential families, illustrative shift functions, and how to interpret them.
  \item In~\cref{appsub:adding_causal_edges_to_the_graph}, we formalize the idea that one can choose shift functions which depend on additional variables, other than the causal parents of a variable $W_i$.
  \item In~\cref{appsub:domain_preserving_parameterizations_of_shift} we give guidance on how to define shift functions when the parameters $\eta(Z)$ are constrained to lie in a particular domain, which is relevant for considering shifts such as changing the variance of a conditional Gaussian.
\end{itemize}

\subsection{Conditional exponential family models and interpretations of shifts}%
\label{appsub:examples_of_conditional_exponential_family_models}

In this section, we give examples of exponential families and their sufficient statistics, and discuss design considerations in specifying the shift function $s(Z; \delta)$.  Here, we restrict attention to shifts in a single variable, for ease of notation.  In~\cref{tab:examples-expontial-families} we give examples of conditional exponential families, along with their typical parameterizations. 
In the examples below, we review how shift functions $s(Z; \delta)$ impact these parameters, and how they can also be interpreted on the scale of more commonly considered parameters (e.g., conditional means and variances).

\begin{table}[t]
\caption{Examples of conditional exponential family distributions.}
\label{tab:examples-expontial-families}
\centering
\resizebox{\textwidth}{!}{%
\begin{tabular}{@{}lllll@{}}
\toprule
Distribution & Parameter space & Sufficient statistic & Inverse parameter map \\
\midrule
Binary($p$)       & $\eta(Z) \in \R$ & $T(W) = W$ & $p(W = 1 |  Z) = \text{sigmoid}(\eta(Z))$ \\
Categorical($p_1, \ldots, p_k$) & $\eta(Z) \in \R^{k}$ & $[T(W)]_i = \1{W = i}$ & $\P(W = i |  Z) = [\text{softmax}(\eta(Z))]_i$ \\
Poisson$(\lambda)$ & $\eta(Z) \in \R$ & $T(W) = W$ & $\lambda = \exp(\eta(Z))$ \\
Gaussian($\mu, \sigma^2$)  & $\eta(Z)_1 \in \R, \eta(Z)_2 < 0$ & $T(W) = (W, W^2)$ & $\mu(Z) = -\frac{\eta(Z)_1}{2 \eta(Z)_2}, \sigma^2(Z) = - \frac{1}{2 \eta(Z)_2}$ \\
Gamma$(\alpha, \beta)$ & $\eta(Z)_1 > -1, \eta(Z)_2 < 0$ & $T(W) = (\log W, W)$ & $\alpha(Z) = \eta(Z)_1 + 1, \beta(Z) = -\eta(Z)_2$\\
\bottomrule
\end{tabular}%
}
\end{table}

\begin{example}[Log-odds shift in a binary variable]\label{ex:logodds_shift}
Consider the distribution of a binary variable $W$ conditioned on variables $Z$.  Without loss of generality, we can write that 
\begin{equation*}
  \P(W = 1 |  Z) = \sigma(\eta(Z))
\end{equation*}
where $\sigma$ is the sigmoid function, and $\eta(Z)$ is an arbitrary measurable function of $Z$, taking on values in the extended real line $\eta(Z) \in \R \cup \{-\infty, +\infty\}$.  This can be written in canonical form as 
\begin{equation*}
  \P(W |  Z) = \exp\bigg\{\eta(Z)\cdot W  - \log(1 + \exp^{\eta(Z)})\bigg\}
\end{equation*}
where $\eta(Z)$ is the canonical parameter (the log-odds ratio), $T(W) = W$ is the sufficient statistic, and $h(\theta) = \log(1 + \exp^{\eta(Z)})$ is the normalizing constant. We can consider shifts $\eta_{\delta}(Z) \coloneqq \eta(Z) + \delta$, yielding the new conditional distribution
\begin{equation*}
  \P_{\delta}(W = 1 |  Z) = \sigma(\eta(Z) + \delta),
\end{equation*}
which is well-defined for any $\delta \in \R$. 
\end{example}

Here, we note that these shifts occur on the \enquote{natural} parameter scale $\eta(Z)$ (e.g., the log-odds), which at first glance may seem difficult to interpret: Why should we care about changes on the log-odds scale, instead of on the original probability scale? In addition to mathematical convenience, we argue that in some settings, working with natural parameters is advantageous for retaining a common scale across across multiple variables.

For instance, consider shifts in the two independent variables $W_1$ and $W_2$, where $V_i \sim \text{Bernoulli}(p_i)$, with $p_1 = 10^{-4}$ and $p_2 = 0.6$.  Suppose we wished to consider an additive shift on the probability scale, e.g., $p_1' = p_1 + 0.1, p_2' = p_2 + 0.1$.  Setting aside the inconvenience that we need to ensure $p_1', p_2' \in [0, 1]$, we argue that these shifts are not truly of a comparable scale.  In particular, this shift in $p_1$ may seem implausible in magnitude, while the same shift in $p_2$ seems more reasonable. On the other hand, an additive shift in the log-odds captures some aspect of this idea.

Of course, there is some flexibility to incorporate prior expectations of shifts in absolute probabilities. For instance, in binary variable with no causal parents, we can always construct a one-to-one map of $\delta$ to a change in the marginal probability.  For conditional shifts, we can similarly construct a one-to-one map between the value of $\delta$ in a shift $s(Z; \delta) = \delta$ and the resulting marginal probability of $W_i$, as formalized below.
\begin{restatable}{proposition}{LogOddsMarginaltoShift}\label{prop:log_odds_marginal_to_shift}
Consider a binary random variable $W$ with conditional distribution 
\begin{equation*}
  \P_{\delta}(W = 1 |  Z) = \sigma(\eta(Z) + \delta)
\end{equation*}
for an arbitrary measurable function $\eta(Z)$ whose range is the extended real numbers $\eta(Z) \in \R \cup \{+ \infty, - \infty\}$. Let $p_{+} \coloneqq \P(\eta(Z) = +\infty)$, $p_{-} \coloneqq \P(\eta(Z) = - \infty)$, and assume that $p_{+} + p_{-} < 1$. Then, the marginal probability 
\begin{equation*}
  p_{\delta} = \P_\delta(W = 1)
\end{equation*}
is a strictly monotonically increasing function of $\delta \in \R$ whose range is $(p_{+}, 1 - p_{-})$, 
\end{restatable}
\Cref{prop:log_odds_marginal_to_shift} states that, for any achievable marginal probability $p_{\delta} = \P_{\delta}(W = 1)$, there exists a unique value of $\delta$ that achieves this probability.  Because this relationship is strictly monotonic, we can hope to efficiently find such a value by e.g., binary search.  In the laboratory testing example of~\cref{ex:lab_testing_rates}, this would allow us to specify a plausible strength for the conditional shift $\delta$ in terms of an impact on the overall testing rate, e.g., modelling a scenario where the testing rate decreases from 20\% to 15\%.

Similar to the binary case, we can (if desired) directly parameterize shifts in terms of the conditional mean of a Gaussian distribution, as illustrated in~\cref{ex:gaussian_shift}, which operates on the scale of $\mu(Z)$ alone.
\begin{example}[Mean shift in a conditional Gaussian]\label{ex:gaussian_shift}
Consider the distribution of a multi-variate Gaussian variable $W$ conditioned on a binary variable $Z$, where we write 
\begin{equation*}
  p(w |  z) \eid \cN(w; \mu(z), \Sigma(z))
\end{equation*}
where $\cN(w; \mu(z), \Sigma(z))$ denotes the Gaussian density with mean $\mu(z)$ and covariance $\Sigma(z)$.  This can be written as an exponential family model with natural parameters $\eta(Z) = [{\Sigma(Z)}^{-1} \mu(Z), -\frac{1}{2}{\Sigma(Z)}^{-1}]$ and sufficient statistic $T(W) = [W, WW^\top]$.  Here, a shift in the mean can be parameterized by $s(Z; \delta) = [{\Sigma(Z)}^{-1} \delta, 0]$, such that 
\begin{equation*}
  p_{\delta}(w |  z) \eid \cN(w; \mu(z) + \delta, \Sigma(z)).
\end{equation*}
\end{example}

However, shifts of the same magnitude in the conditional mean may not be comparable.  Suppose that 
\begin{equation*}
    \P(W |  Z = 0) \eid \cN(0, 1) \qquad\text{and}\qquad \P(W |  Z = 1) \eid \cN(0, 0.001),
\end{equation*}
such that $\delta = 1$ in~\cref{ex:gaussian_shift} corresponds to 
\begin{equation*}
    \P_{\delta=1}(W |  Z = 0) \eid \cN(1, 1) \qquad\text{and}\qquad \P_{\delta=1}(W |  Z = 1) \eid \cN(1, 0.001).
\end{equation*}
While it may seem plausible that the mean of $W |  Z = 0$ can increase by $1$, it may seem unrealistic for $W |  Z = 1$. 
Here, it may be more reasonable to consider a different parameterization of $s(Z; \delta)$, where the impact of the shift in a direction is proportional to the variance in that direction; we discuss this in the next example.

\begin{example}[Variance-scaled mean shift in a conditional Gaussian]\label{ex:gaussian_shift_variance_scaled}
Consider the distribution of a multi-variate Gaussian variable $W$ conditioned on variables $Z$, where we write 
\begin{equation*}
  p(w |  z) \eid \cN(w; \mu(z), \Sigma(z))
\end{equation*}
where $\cN(w; \mu(z), \Sigma(z))$ denotes the Gaussian density with mean $\mu(z)$ and covariance $\Sigma(z)$.  This can be written as an exponential family model with natural parameters $\eta(Z) = [{\Sigma(Z)}^{-1} \mu(Z), -\frac{1}{2}{\Sigma(Z)}^{-1}]$ and sufficient statistic $T(W) = [W, WW^\top]$.  Here, a shift in the mean can be parameterized by $s(Z; \delta) = [\delta, 0]$, such that 
\begin{equation*}
  p_{\delta}(w |  z) \eid \cN(w; \mu(z) + \delta^\top \Sigma(Z), \Sigma(z)).
\end{equation*}
\end{example}

In~\cref{ex:gaussian_shift_variance_scaled}, the parameter $\delta$ has a different interpretation, as a variance-scaled mean-shift.  If $W$ is one-dimensional, we can see that this becomes 
\begin{equation*}
  p_{\delta}(w |  z) \eid \cN(w; \mu(z) + \delta \sigma^2(Z), \sigma^2(z)).
\end{equation*}
As we demonstrate in~\cref{sec:example-estimate-eta-not-needed}, this particular example of a parameterization has other benefits: For instance, for estimation of shift gradients and Hessians at $\delta = 0$ can be done without knowledge of $\Sigma(Z)$.

\subsection{Adding causal edges to the graph}%
\label{appsub:adding_causal_edges_to_the_graph}
In \cref{sec:defining_parametric_robustness_sets}, we consider the case where the shift function $s(Z; \delta)$ alters a conditional $\P(W|Z)$ by a shift function $s(Z;\delta)$. We now discuss shift functions that use a larger set $Z'$. In particular, we consider the setting where $Z$ represents the parents in a graph $\cG$ (that is, $Z \coloneqq \PA_{\cG}(W)$), and consider shift functions that correspond to adding additional parents in that causal graph. Our definitions and results immediately extend to measuring the impact of shifts that \textbf{add edges} to the graph, in the form of shift functions that depend on non-descendants of $W$.  

\textbf{Building intuition with a simple example:} To build intuition, consider the causal graph given in~\cref{fig:add_causal_edge}.  We consider a shift in $X_2$, with a shift function which depends not only on the causal parent $Y$, but also on $X_1$.
\begin{figure}[t]
\begin{center}
\begin{tikzpicture}[
  obs/.style={circle, draw=gray!90, fill=gray!30, very thick, minimum size=10mm}, 
  int/.style={rectangle, draw=red!10, fill=red!30, minimum size=5mm}, 
  uobs/.style={circle, draw=gray!90, fill=gray!10, dotted, minimum size=10mm}, 
  bend angle=30]
  \node[obs] (X1) {$X_1$};
  \node[obs] (Y) [right=of X1]  {$Y$};
  \node[obs] (X2) [right=of Y] {$X_2$} ;
  \node[int] (s) [above=of X2] {$s(X_1, Y; \delta)$} ;
  \draw[-latex, thick] (X1) -- (Y);
  \draw[-latex, red] (X1) -- (s);
  \draw[-latex, red] (Y) -- (s);
  \draw[-latex, thick] (Y) -- (X2);
  \draw[-latex, red] (s) -- (X2);
\end{tikzpicture}
\end{center}
\caption{Illustrative example of an intervention $s(X_1, Y; \delta)$, and modified causal graph, which creates a dependence between $X_1$ and $X_2$ that bypasses $Y$.}%
\label{fig:add_causal_edge}
\end{figure}
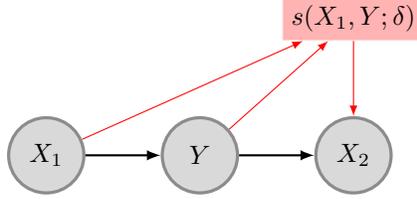
Suppose that the distribution $\P(X_2 |  Y)$ is a conditional exponential family, given by 
\begin{equation*}
  \P(X_2 |  Y) = g(X_2) \exp({\eta(Y)}^\top T(X_2) - h(\eta(Y))).
\end{equation*}
Using that $X_2\indep X_1 | Y$, we have $\P(X_2|Y) = \P(X_2|Y, X_2)$, and the joint probability factorizes as
\begin{equation*}
  \P(X_1, X_2, Y) = \P(X_2 |  Y) \P(Y |  X_1) \P(X_1) = \P(X_2 |  Y, X_1) \P(Y |  X_1) \P(X_1). 
\end{equation*}
This enables us to consider $Z = (Y, X_1)$ as the conditioning set in the context of~\cref{assmp:cef_factorization}.  This is useful, because it allows us to consider shift functions that depend on $Z$, which includes $X_1$ in addition to $Y$. The $\delta$-perturbation of this conditional distribution under the shift function $s(Y, X_1; \delta)$ is given by 
\begin{equation*}
  \P_{\delta}(X_2 |  Y, X_1) = g(X_2) \exp\bigg({\{\eta(Y) + s(Y, X_1; \delta)\}}^\top T(X_2) - h\big(\eta(Y) + s(Y, X_1; \delta)\big)\bigg), 
\end{equation*}
and we can observe that under both graphs, the distribution factorizes in the same fashion, where 
\begin{equation*}
  \P_{\delta}(X_1, X_2, Y) = \P_{\delta}(X_2 |  Y, X_1) \P(Y |  X_1) \P(X_1),
\end{equation*}
keeping the same convention that $s(Y, X_1; \delta = 0) = 0$, such that $\P_{0} = \P$.  This is one example of how our results can be applied with shift functions that effectively add edges to the causal graph.  Of course, not all edges are permitted, so we give a more general treatment below.

\newcommand{\ND}{\texttt{nd}}
\textbf{General guidelines for adding edges}: Allowing for the use of non-causal parents in the shift functions is straightforward, and can be done safely as follows, without violating~\cref{assmp:cef_factorization}: Given knowledge of the directed acyclic graph $\cG$ which generates the observed distribution $\P$, we can \textbf{add} edges to the graph, as long as they do not create cycles.  

Formally, let $\cG = (\bV, E)$ denote the causal DAG which generates the distribution $\P$, where $\bV$ denotes variables and $E$ denotes the set of edges, where we denote a directed edge by $e = (V_i, V_j)$, going from $V_i$ to $V_j$.  Let $\cG' = (\bV', E')$ denote another DAG (of our creation) with the constraint that we can only add edges, and that the graph must remain acyclic, such that $E' \supseteq E$, and $\bV' = \bV$.

For any variable $W_i \in \bV$, this implies that $\PA_{\cG'}(W_i) \supseteq \PA_{\cG}(W_i)$.  Moreover, any new causal parent $V_i$ of $W_i$ in $\cG'$ must have been a non-descendant of $W_i$ in the original graph, as otherwise the graph $\cG'$ would have a cycle from $W_i \rightarrow V_i \rightarrow W_i$. For ease of notation, let $N(W_i) \coloneqq \PA_{\cG'}(W_i) \setminus \PA_{\cG}(W_i)$ denote the set of new causal parents of $W_i$ in $\cG'$.  For any variable $W_i$ such that $N(W_i) \neq \varnothing$, we can write that 
\begin{equation}\label{eq:add_parents_cond_indep}
  W_i \indep_{\cG} N(W_i) |  \PA_{\cG}(W_i)
\end{equation}
by the rules of d-separation \citep{pearl2009causality}. As in~\cref{assmp:cef_factorization}, we use $\bW = \{W_1, \ldots, W_m\}$ to denote the set of variables to be intervened upon, and accordingly will assume that in the causal graph $\cG'$, we have not added new parents to any other variables, i.e., $N(V_i) = \varnothing$ for any $V_i \subsetneq \bW$. 

By~\cref{eq:add_parents_cond_indep}, we can write that the distribution $\P$ factorizes as 
\begin{equation*}
  \P(\bV) = \left(\prod_{W_i \in \bW} \P(W_i |  \PA_{\cG'}(W_i) )\right) \prod_{V_i \in \bV \setminus \bW} \P(V_i |  \PA_{\cG}(V_i))
\end{equation*}
because $\P(W_i |  \PA_{\cG'}(W_i)) = \P(W_i |  \PA_{\cG}(W_i)$, and if $\P(W_i |  \PA_{\cG}(W_i))$ is a conditional exponential family satisfying~\cref{def:conditional_exponential_family}, then $\P(W_i |  \PA_{\cG}(W_i))$ also satisfies this definition, where the function $\eta(\PA_{\cG}(W_i), N(W_i))$ is constant with respect to fluctuation in the variables $N(W_i)$.  Thus, taking $Z_i \coloneqq \PA_{\cG'}(W_i)$ as the conditioning set satisfies Assumption~\ref{assmp:cef_factorization}, and the rest of our results hold, where the corresponding $\delta$-perturbations in Definition~\ref{def:multiple_shift} are given by
\begin{equation*}
  \P_{\delta}(\bV) = (\prod_{W_i \in \bW} \P_{\delta_i}(W_i |  \PA_{\cG'}(W_i) )) \prod_{V_i \in \bV \setminus \bW} \P(V_i |  \PA_{\cG}(V_i))
\end{equation*}
with shift function $s_i(\PA_{\cG'}(W_i); \delta_i)$ that are parametric functions of causal parents in the modified graph $\cG'$.

\subsection{Domain-preserving parameterizations of shift}%
\label{appsub:domain_preserving_parameterizations_of_shift}

For both of the examples considered above, we did not need to restrict the magnitude of the additive change to $\eta(Z)$.  However, in some cases, such as changing the variance of a conditional Gaussian, we have the restriction that $\eta_{\delta}(Z) = \eta(Z) + s(Z; \delta)$ must lie in the proper domain, e.g., we cannot consider a shift which causes the conditional variance to become negative.  For a conditional Gaussian, we can consider unrestricted shifts in $\eta(Z)_1$, which controls the mean, because the mean has unrestricted domain.  On the other hand, $\eta(Z)_2 = (-2 \sigma^2(Z))^{-1}$ controls the variance, and must remain negative, such that $\eta(Z)_2 + s(Z; \delta)_2 < 0$ for the shifts we consider.

This can be resolved in one of two ways.  First, one can consider parameterizations of $s(Z; \delta)$ which are guaranteed to preserve the correct domain with an additional constraint on the values of $\delta$, such as the multiplicative shift below, which is sign-preserving for $\delta > -1$
\begin{align*}
\eta_{\delta}(Z)_2 &= \eta(Z)_2 + \underbrace{\delta \eta(Z)_2}_{s(Z; \delta)} = (1 + \delta) \eta(Z)_2.
\end{align*}
To handle the general case, at the expense of some additional complexity in the gradients of $s(Z; \delta)$, one can define the shifts as follows for parameters $\eta(Z)$ that have a lower bound $L$, with an equivalent formulation for shifts where the parameters have an upper bound, for any desired shift function $s'(Z; \delta)$
\begin{equation*}
  \eta(Z) + \underbrace{s'(Z; \delta) \cdot \text{sigmoid}(\gamma \cdot [(\eta(Z) + s'(Z; \delta)) - (L + \epsilon)])}_{s(Z; \delta)}
\end{equation*}
where $\text{sigmoid}(\gamma \cdot (x - (L + \epsilon)))$ is a smooth relaxation of the indicator function $\1{x > L + \epsilon}$, for a sufficiently large temperature parameter $\gamma > 0$ and a small $\epsilon > 0$.  This transformation preserves the twice-differentiable nature of $s(Z; \delta)$.  In practice, however, we typically evaluate the gradient of $s(Z; \delta)$ at $\delta = 0$, where $\eta(Z)$ does not lie at the boundary of allowable parameter space, such that we can consider simpler parameterizations like
\begin{equation*}
  \eta(Z) + \underbrace{s'(Z; \delta) \cdot \1{\eta(Z) + s'(Z; \delta) > L + \epsilon}}_{s(Z; \delta)}
\end{equation*}
as long as $\epsilon$ is taken sufficient small such that $\eta(Z) > L + \epsilon$ almost everywhere in $\P$.

\section{Considerations and additional results for evaluation of the worst-case loss}%
\label{sec:considerations_for_evaluation_of_the_worst_case_loss}
In this section, we present additional results on the Taylor approximation and compare how the Taylor approximation compares to the reweighting approach in evaluation and worst-case optimization of the shifted loss.
\begin{itemize}
    \item In \cref{appsec:algorithm_calculation_theorem1} we give a full treatment of how shift gradients and Hessians are estimated from samples, following~\cref{thm:sg-cond-cov-several-shift}.
    \item In \cref{sec:example-estimate-eta-not-needed}, we demonstrate in some cases, one does not need to estimate all of $\eta(Z)$, but only the parts of $\eta(Z)$ that is shifting.
    \item In~\cref{sec:linear-model-exact}, we demonstrate that the second-order Taylor expansion is exact in a linear-Gaussian setting, which gives a conceptual connection between this work and that of Anchor Regression \citep{rothenhausler2021anchor}, which considered a restricted type of additive shift intervention in a globally linear structural causal model.
    \item In \cref{sec:shift_gradient_and_hessian_binary}, we work out the expression for the shift gradient and Hessian when we condition on binary variables.
    \item In \cref{ex:lab_test_variance_plots,sec:variance-ipw-vs-taylor,experiment:compare-ipw}, we provide experiments that compare the variance of the importance sampling estimate $\ipw$ (see \cref{eq:ipw-estimate-of-mean}) to the variance of the Taylor estimate $\taylor$ (see \cref{eq:taylor-approx}) of the loss in a shifted distribution.
    \item In \cref{subsec:example-of-bound}, we consider the bound in \cref{thm:taylor-approximation-bound} in a covariate shift setting, and give an explicit expression for this under additional assumptions.
\end{itemize}

\subsection{Algorithm for Estimation of Shift Gradients and Hessians}
\label{appsec:algorithm_calculation_theorem1}

Here, we recall the form of the shift gradients and Hessians in~\cref{thm:sg-cond-cov-several-shift}, and demonstrate how to compute them in practice using a set of auxiliary regression functions fit to the validation data.
\SgConvCovSeveralShift* 

\textbf{Notation and Dimensions}: Let $\bW = \{W_1, \ldots, W_m\}$ denote the set of $m$ intervened variables, and let $\bZ = \{Z_1, \ldots, Z_m\}$ denote the conditioning sets.  Note that for a single $W_i \in \R^{d_{W_i}}$, we will generally have it that $Z_i \in \R^{d_Z}$, where $d_W$ is the dimension of $W$ (typically 1) and $d_Z$ is the number of conditioning variables, and when considering $n$ samples, $W_i$ will be a matrix in $\R^{n \times d_{W}}$, and $Z_i$ will be a matrix $\R^{n \times d_Z}$.  The sufficient statistic $T_i(W_i)$ maps from $\R^{d_W}$ to $\R^{d_T}$, where $d_T$ is the dimension of the sufficient statistic.  For many common distributions, $T_i(W_i) = W_i$, the identity function.  For others, like the conditional multi-variate Gaussian, $T_i(W_i) = [W_i, W_i W_i^\top]$, where $W \in \R^{d_W}$ and $W_i W_i^\top \in \R^{d_W \times d_W}$.  In these cases, we squeeze $T_i(W_i)$ to be a single vector, so in this case $d_T = d_W + d_W^2$.

\newcommand{\predw}{\hat{\mu}_{W_i}}
\newcommand{\predl}{\hat{\mu}_{\ell}}
\newcommand{\residl}{\hat{\epsilon}_{\ell |  Z_i}}
\newcommand{\residtzi}{\hat{\epsilon}_{T_i |  Z_i}}
\newcommand{\residtzj}{\hat{\epsilon}_{T_j |  Z_j}}
\textbf{Auxiliary models}: To estimate the shift gradients and Hessians, we first learn auxiliary predictive models, which are required for computing the relevant conditional covariances. For simplicity, we do not consider sample-splitting in the algorithm given below, but one could employ sample-splitting to learn these predictive models on an independent validation sample.  
\begin{itemize}
    \item For each $W_i$, we learn $\predw(Z_i)$ as a regression model for $\E[T_i(W_i) |  Z_i]$.  Because $T_i(W_i)$ may have multiple dimensions, this is a function from $\R^{d_Z}$ to $\R^{d_T}$.
    \item For each conditioning set $Z_i$, we learn $\predl(Z_i)$ as a regression model for $\E[\ell |  Z_i]$.  Because the loss is one-dimensional, this is a function from $\R^{d_Z}$ to $\R$.
\end{itemize}

We then construct the following, which are defined for each data point in the sample.
\begin{itemize}
    \item For each $W_i$, we construct $\residtzi \coloneqq T_i(W_i) - \predw(Z_i)$, which is a vector of length $d_{T_i}$.
    \item For each conditioning set $Z_i$, for the loss $\ell$, we construct $\residl \coloneqq \ell - \predl(Z_i)$, which is a real number.
    \item For each conditioning set $Z_i$, we compute $D_{i, 1}(Z_i)$ as $\nabla_{\delta_i} s_i(Z_i; \delta_i) \big|_{\delta = 0}$,  which is a matrix of size $d_T \times d_{\delta_i}$, and a function of $Z_i$ that we can evaluate on each sample.
    \item For each conditioning set $Z_i$, we compute $D_{i, 2}(Z_i)$ as $\nabla^2_{\delta_i} s_i(Z_i; \delta_i) \big|_{\delta = 0}$,  which is a tensor of size $d_T \times d_{\delta_i} \times d_{\delta_i}$, and a function of $Z_i$ that we can evaluate on each sample.
\end{itemize}

\textbf{Estimating shift gradients}
The shift gradient and Hessian in \cref{thm:sg-cond-cov-several-shift} are expressed as conditional covariance. Since $\E[\cov(A, B|C)] = \E[\epsilon_{A|C}\epsilon_{B|C}]$ where $\epsilon_{A|C}:=A - \E[A|C]$ and $\epsilon_{B|C}:=B - \E[B|C]$, we can use the estimated conditional means above, to compute the shift gradient and Hessian.
Suppose that we observe $N$ samples, $n \in \{1, \ldots, N\}$. 
For each index $i \in [m] := \{1, \ldots, m\}$, 
\begin{equation*}
   \hat{\sg}_i^1 = \frac{1}{N} \sum_{n = 1}^{N} \residl^{(n)} \cdot {D_{i, 1}(Z_i^{(n)}) }^\top \residtzi^{(n)}
\end{equation*}
which yields a vector of length $d_{\delta_i}$, and these are concatenated together for each $i$ to yield the entire shift gradient.  The shift Hessian is constructed block-wise, for each index $i, j \in [m] \times [m]$ as follows:  If $i = j$, then we construct the corresponding $d_{\delta_i} \times d_{\delta_i}$ block as 
\begin{equation*}
   \hat{\sg}^2_{i,i} = \frac{1}{N} \sum_{n = 1}^{N} \residl^{(n)} \cdot \left[\left({D_{i, 1}(Z_i^{(n)})}^\top \residtzi^{(n)}\right)^{\otimes 2} - D_{i, 2}(Z_i^{(n)})^\top \residtzi \right]
\end{equation*}
where $v^{\otimes 2}$ denotes the outer product so that $v^{\otimes 2} = vv^\top$, and the transpose of $D_{i, 2}$ refers to a transpose which has dimension $d_{\delta_i} \times d_{\delta_i} \times d_T$.   On the other hand, if $i \neq j$ we have 
\begin{equation*}
   \hat{\sg}^2_{i,j} = \frac{1}{N} \sum_{n = 1}^{N} (\ell^{(n)} - \bar{\ell}) \cdot \left(D_{i, 1}(Z_i^{(n)})^\top \residtzi^{(n)}\right) \left(D_{j, 1}(Z_j^{(n)})^\top \residtzj^{(n)}\right)^\top
\end{equation*}
where $\bar{\ell}$ is the average value of $\ell$ in the validation sample.

\subsection{Shifts where estimating all of \texorpdfstring{$\eta(Z)$}{} is not necessary for estimating shift gradient and Hessian}%
\label{sec:example-estimate-eta-not-needed}
The following example shows that when a shift occurs in an exponential conditional distribution with parameter $\eta(Z)$, we do not necessarily need to model all of $\eta(Z)$ in order to compute the shift gradient and Hessian. In particular, we only need to model the parts of $\eta(Z)$ that shift. This is different from estimating the shifted loss using importance sampling, where $\eta(Z)$ needs to be evaluated to evaluate \cref{eq:ipw-weights-exponential-family}.
\begin{example}\label{ex:gaussian-only-estimate-mean}
  Consider the distribution of $W$ conditioned on variables $Z$ that is a multi-variate Gaussian variable,
  \begin{equation*}
      W |  Z = \cN(\mu(Z), \Sigma(Z)),
  \end{equation*}
  for unknown functions $\mu, \Sigma$.
  The sufficient statistic for the multivariate Gaussian distribution is $T(W) = (W, WW^\top)$ and the canonical parameter is $\eta(Z) = (\Sigma(Z)^{-1}\mu(Z), -\frac{1}{2}\Sigma(Z)^{-1})$.\footnote{Or, more formally, $T(W) = \big(W, \vectorize(WW^\top)\big)$ and $\eta(Z) = \big(\sigma(Z)^{-1}\mu(Z), -\frac{1}{2}\vectorize({\mu(Z)})\big)$, where $\vectorize$ denotes the vectorization operation. For a detailed walk through of the exponential family parameterization of multivariate Gaussian distributions, see \url{https://maurocamaraescudero.netlify.app/post/multivariate-normal-as-an-exponential-family-distribution/}.}
  The first component of $\eta(Z)$ is a signal-to-variance ratio and the second is the inverse covariance matrix. 
  For a shift $(\delta, 0)$ that only affects the first component, we show that we do not need to model $\Sigma(Z)$, but only $\mu(Z)$. This is beneficial, since estimating a conditional covariance from data can be challenging, especially if $W$ is high-dimensional. 

  For $\delta \in \R^{d_W}$, let $s(Z; \delta) = (\delta, 0)^\top$, and suppose that we wish to estimate $\E_\delta[\ell]$ using \cref{eq:taylor-approx}. 
  The derivative of $s$ is given by
  \begin{align*}
    D_1 = \nabla_\delta^2 s(Z;\delta) = 
    \begin{pmatrix}
    \begin{bmatrix}
    1 & 0 & \cdots & 0 \\
    0 & 1 & \cdots & 0 \\
    \vdots & \vdots &\ddots & \vdots \\
    0 & 0 & \cdots & 1
    \end{bmatrix}
    \begin{bmatrix}
    0 & 0 & \cdots & 0 \\
    0 & 0 & \cdots & 0 \\
    \vdots & \vdots &\ddots & \vdots \\
    0 & 0 & \cdots & 0
    \end{bmatrix}
    \end{pmatrix},
  \end{align*}
  where the first block is a $d_W \times d_W$ diagonal matrix, and the second is a $d_{W} \times d_W^2$ matrix of zeros. The second derivative of $s$ is $D_2 = 0$.
  Hence, using \cref{thm:sg-cond-cov-several-shift}, the shift gradient is 
  \begin{align*}
      \sg^1 = \E[D_1 \cov(\ell, (W, WW^\top)|Z)] = \E[\cov(\ell, W|Z)],
  \end{align*}
  and 
  \begin{align*}
    \sg^2 &= \E\left[D_1 \cov(\ell, \bigg(W - \E[W|Z], WW^\top - \E[WW^\top|Z]^{\otimes 2}\bigg)|Z) D_1^\top\right]\\
    &= \E\left[\cov(\ell, \big(W - \E[W|Z]\big)^{\otimes 2}|Z)\right].
  \end{align*}
  Conditional covariances can be computed by only residualizing one of the variables: $\E[\cov(A, B|C)] = \E[A(B-\E[B|C])]$. Thus, if we only residualize $\ell$, we get
  \begin{align*}
      \sg^1 = \E[(\ell-\E[\ell|Z]) W ] \qquad\text{and}\qquad \sg^2 = \E[(\ell - \E[\ell|Z])\cdot(W - \mu(Z))^{\otimes 2}].
  \end{align*}
  Therefore, given data from $\P$, we can estimate the shift gradients by plugging in estimators $\hat\mu(Z)$ of $\E[W|Z]$ and $\hat{L}(Z)$ of $\E[\ell|Z]$. It follows that we do not need to model $\Sigma(Z)$ in order to estimate the shift gradients and Hessian at $\delta = 0$. 

  The story is different for a reweighting based estimator that seeks to estimate $\E_\delta[\ell]$ using importance sampling (see \cref{sec:modelling_shifted_loss_using_ipw}), where the weights are given by
  \begin{equation*}
      w_{\eta,\delta}(Z) = (W - \mu(Z))^\top\delta - \tfrac{1}{2}\delta^\top \Sigma(Z) \delta, 
  \end{equation*}
  and hence estimating $w_{\eta,\delta}(Z)$ requires estimation of $\Sigma(Z)$.
\end{example}

\subsection{The quadratic approximation is exact, for mean shifts in linear models}%
\label{sec:linear-model-exact}

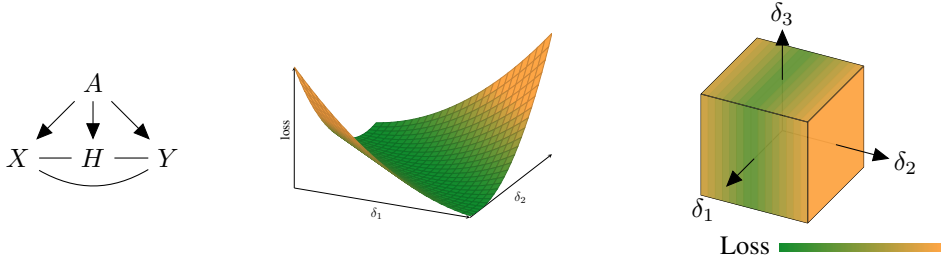
\begin{figure}[t]
  \centering
  \begin{subfigure}{0.2\textwidth}
      \centering
      \begin{tikzpicture}
          \node (A) at (0, 1) {$A$};
          \node (X) at (-1, 0) {$X$};
          \node (H) at (0, 0) {$H$};
          \node (Y) at (1, 0) {$Y$};
          \draw[->] (A) -- (X);
            \draw[->] (A) -- (H);
            \draw[->] (A) -- (Y);
            \draw[-] (X) -- (H);
            \draw[-] (H) -- (Y);
            \draw[-] (X) edge[bend right=30] (Y);
        \end{tikzpicture}
    \end{subfigure}
    \begin{subfigure}{0.4\textwidth}
        \centering
            \begin{tikzpicture}[scale=0.5]
                \begin{axis}[
                    xlabel=$\delta_1$,ylabel=$\delta_2$,
                    zlabel=loss,
                    mesh/interior colormap name=mycolors,
                    colormap={mycolors}{color=(mygreen) color=(myred) color=(myred)},
                    ticks=none,
                    axis x line = bottom,
                    axis y line = left,
                    axis z line = left
                    ]
                    \addplot3 [domain=-0.5:0.5, surf, shader=faceted] {(0.7*x + 0.9*y)^2};
                \end{axis}
            \end{tikzpicture}
    \end{subfigure}
    \begin{subfigure}{0.35\textwidth}
        \centering
            \input{figures/loss-cube}
        \end{subfigure}
    \caption{(Left) Graphical model assumed by \cref{eq:linear-scm}. The undirected edges represent either any directed configuration of directed edges or the dependence structures arising due to an acyclic SCM \citep{bongers2021foundations}.
    (Middle) Plotting $\E_{\delta}[(Y-\gamma^\top X)^2]$ as a function of $\delta\in \R^2$ for a fixed predictor $\gamma$. 
    (Right) Plotting $\E_{\delta}[(Y-\gamma^\top X)^2]$ as a function of $\delta\in \R^3$, with the loss indicated by the color. The loss only varies with changes in $\delta_2$ (corresponding in \cref{lemma:ar-quadratic-loss} to $v_\gamma \propto (0, 1, 0)^\top$).
    }
    \label{fig:ar-graph}
\end{figure}

We now consider data generated by a linear model, and show that the shifted loss is a quadratic function of $\delta$, meaning that the Taylor approximation $\taylorPop$ is globally exact. 
Suppose that data is sampled from a linear structural causal model, and a shift in mean occurs in an variable $A$ that does not have any causal parents. In particular, let $A$ have a normal distribution with mean $\mu$ and finite variance and let
\begin{equation}
  \label{eq:linear-scm}
  \begin{pmatrix} X \\ Y \\ H \end{pmatrix} = B \begin{pmatrix} X \\ Y \\ H \end{pmatrix} + MA + \epsilon.
\end{equation}
This is the model assumed by \citet{rothenhausler2021anchor}, and the corresponding graphical model is shown in \cref{fig:ar-graph} (left).
We consider the linear predictor $f_\gamma(X) = \gamma^\top X$ and the mean squared loss $\ell(f_\gamma(X), Y) = (Y - f(X))^2$. Due to the linearity of the model, the loss under a mean shift in $A$ is quadratic \citep{rothenhausler2021anchor}. 
\begin{restatable}{lemma}{AnchorRegQuadLoss}
\label{lemma:ar-quadratic-loss}
    Suppose $A \sim \mathcal{N}(\mu, \Sigma)$ and that $(X, Y, H)$ are generated according to \cref{eq:linear-scm}. 
    For $\gamma\in\R^{d_X}$ define $\ell:= (Y-\gamma^\top X)^2$. Then there exist $v_\gamma, u_{\mu,\gamma} \in \R^{d_A}$ such that for all shifts $\delta\in\R^{d_A}$:
    \begin{align*}
        \E_{\delta}[\ell] = \E[\ell] + \delta^\top u_{\mu,\gamma} + \tfrac{1}{2}\delta^\top v_\gamma v_\gamma^\top \delta,
    \end{align*}
    where $\E_{\delta}$ corresponds to taking the mean in the distribution where $A \sim \mathcal{N}(\mu + \delta, \Sigma)$.
    Further $u_{\mu,\gamma} = 0$ if $\mu = 0$.
\end{restatable}
\Cref{prop:reconciliation-ar} elicits two properties of this linear model: First the loss is described by a quadratic function globally, i.e. also for very large $\delta$. 
In \cref{fig:ar-graph} (middle), we plot $\E_{\delta}[\ell]$ as a function of $\delta$. We observe a `valley' in the loss, in which the expected loss does not at all change with $\delta$. This is a consequence of \cref{lemma:ar-quadratic-loss}, and particularly that if $\delta$ is orthogonal to both $u_{\mu,\gamma}$ and $v_\gamma$ then $\E_{\delta}[\ell] = \E[\ell]$. In higher dimensions $d_A > 2$, since $v_\gamma v_\gamma^\top$ has rank $1$, the `valley' persists in that the loss does not grow at all in $d_A - 2$ dimensions (or $d_A - 1$ if $A$ has mean $\mu=0$), see \cref{fig:ar-graph} (right).

We now show that coefficients in the quadratic form in \cref{lemma:ar-quadratic-loss} is equal to the shift gradient and Hessian. We use that the Gaussian distribution with known variance $\Sigma$ can be parameterized as an exponential family with sufficient statistic $T(A) = \Sigma^{-1}A$ and parameter $\eta = \mu$.\footnote{It can also be parameterized as $T(A) = \Sigma^{-1/2}A, \eta = \Sigma^{-1/2}\mu$, which would yield the same result.}
\begin{restatable}{proposition}{ReconAR}
\label{prop:reconciliation-ar}
    Suppose $A \sim \mathcal{N}(\mu, \Sigma)$ and that $(X, Y, H)$ are generated according to \cref{eq:linear-scm}.
    Then the shift gradient and Hessian are given by
    \begin{align*}
        \sg^1 = \cov(\ell, \Sigma^{-1}A) \qquad\text{and}\qquad \sg^2 = \cov(\ell, \Sigma^{-1}(A - \mu)(A-\mu)^\top \Sigma^{-\top})
    \end{align*}
    and the loss under a mean shift of $\delta$ in $A$ is given by
    \begin{align*}
        \E_{\delta}[\ell] = \E[\ell] + \delta^\top \sg^1 + \tfrac{1}{2}\delta^\top \sg^2\delta,
    \end{align*}
    where $\ell := (Y-\gamma^\top X)^2$ and $\E_{\delta}$ corresponds to taking the mean in the distribution where $A \sim \mathcal{N}(\mu + \delta, \Sigma)$.
\end{restatable}
This elicits a connection to anchor regression \citep{rothenhausler2021anchor}: 
Under the generative model \cref{eq:linear-scm} and using the quadratic loss $\ell=(Y-\gamma^\top X)^2$ for $\gamma\in\R^{d_X}$, they show that for any $\lambda \geq 0$, the worst-case loss $\E_\delta[\ell]$ over a set $\Delta = \{\delta|\delta\delta^\top \preceq \lambda \E[AA^\top]\}$ equals the objective $\ell_{\text{AR}} = \E[\ell] + \lambda \E[\E[Y - \gamma^\top X|A]^2]$, which is computable from the observed distribution. 

Because of \cref{prop:reconciliation-ar}, $\ell_{\text{AR}}$ \emph{also} equals the solution of the optimization problem \cref{eq:second-order-maximization} over the constraint set $\Delta$. Therefore minimizing the anchor regression objective over $\gamma$ or minimizing \cref{eq:second-order-maximization} over $\gamma$ will lead to the same estimator. 
Since our proposed Taylor approximation in \cref{eq:second-order-maximization} does not assume linearity, one could use the approximation to extend the rationale of anchor regression of minimizing the worst-case loss to non-linear models. This however comes at the cost of not optimizing the exact worst-case loss, but rather an approximation, whose quality is given by \cref{thm:taylor-approximation-bound}. Further, this would involving a minimax problem, minimizing \cref{eq:second-order-maximization} over models $f$, and there are questions, such as convexity and tractability, which would need to be solved.

\subsection{Estimating the shift gradient and Hessian for conditional on binary variables}%
\label{sec:shift_gradient_and_hessian_binary}
To build intuition for the shift gradient and Hessian, we here give an example where we condition on variables $Z$ that take a finite number of values and write out explicit expressions for the shift gradient and Hessian. However, we emphasize, that in most practical scenarios, one will not have to work out the shift gradient and Hessian explicitly, but can simply estimate them as covariances from the data (\cref{thm:sg-cond-cov-several-shift}).

\begin{example}[Shift Function of Discrete Parents]
    Consider a conditional distribution $W|Z$ where $Z$ takes values in a finite set $\cZ$. This is for instance the case if $Z = (Z_1, \ldots, Z_d)$ where each $Z_i$ is binary, so $|\cZ| = 2^d$. Instead of a shift $\eta(Z) + \delta$, where the parameter increases by the same amount for all values of $Z$, we may consider a shift $\eta(Z) + s(Z;\delta)$ where $s(Z;\delta) = \sum_{z\in\cZ} \delta_z 1_{Z=z}$, meaning that the shift is different in each category $Z$. Since $\eta(Z)$ only takes a finite number of variables, this shift corresponds to an arbitrary change in $\eta(Z)$.
    
    $s(Z;\delta)$ is a differentiable function in $\delta$, and if $d_T = 1$ the shift gradient is a $(1\times 2^d)$-row vector, $\nabla_{\delta}s(Z;\delta) = (1_{Z=z})_{z\in \cZ}$, and the shift Hessian vanishes, $\nabla^2_{\delta}s(Z;\delta) = 0$. 
    Enumerating $\cZ = \{1, \ldots, 2^d\}$, the $i$'th entry in the shift gradient becomes
    \begin{equation*}
        (\sg^1)_i = \E\left[1_{Z = i}\cov\bigg(\ell, T(W)\bigg| Z\bigg)\right] = \P(Z = i)\cov(\ell, T(W)|Z=i),
    \end{equation*}
    and the $i,j$'th entry of the shift Hessian becomes $0$ if $j\neq i$ and else
    \begin{equation*}
        (\sg^2)_{i,i} = \E\left[1_{Z=i} \cov_\delta\left(\ell, \epsilon_{T|Z}^{\otimes 2}\bigg|Z\right)\right] = \P(Z = i)\cov(\ell, \epsilon_{T|Z}^{\otimes 2}|Z=i).
    \end{equation*}
    
    Consider for example the case where both $W$ and $Z$ are binary. Then $T(W) = W$ and $s(Z;\delta) = 1_{Z=0}\delta_0 + 1_{Z=1}\delta_1$ and $s^{(1)} = (1_{Z=0}, 1_{Z=1})$ and $s^{(2)} = 0$.
    The conditional covariance can be evaluated by residualizing only one of the variables, $\E[\cov(A, B|C)] = \E[A(B-\E[B|C])]$, so we can chose to residualize only $W$ (for $\sg^1$) or $(W - \E[W|Z=i])^2$ (for $\sg^2$). Finally, if we let $p_i = \P(W=1|Z=i)$ and use that $\E[W|Z=i] = p_i$ and $\E[(W - p_i)^2|Z=i] = \var(W|Z=i) = p_i(1-p_i)$, we get that
    \begin{align*}
        \sg^1 &= \E\left[\begin{pmatrix} 
        p_0 \cdot \ell \cdot (W - p_0)\\
        p_1 \cdot \ell \cdot (W - p_1)\\
        \end{pmatrix}\right],
    \end{align*}
    and
    \begin{align*}
        \sg^2 &= \E\left[\begin{pmatrix}\ell p_0 \big\{(W - p_0)^2 - p_0(1-p_0)\big\} & 0 \\ 0 & \ell p_1 \big\{(W - p_1)^2 - p_1(1-p_1)\big\} \end{pmatrix}\right].
    \end{align*}
\end{example}

\subsection{Comparison of variance of reweighting and Taylor estimates in the lab ordering example}
\label{ex:lab_test_variance_plots}
To compare the bias and variance of the Taylor and the importance sampling estimates of the shifted loss, we simulate data from the following, artificial, generative model (which is the same generative model that was used to construct the loss landscape in~\cref{fig:lab_testing_shift} (right)). 
\begin{align*}
    \text{Age}&\sim\cN(0, 0.5^2) \\
    \P(\text{Disease} = 1 |  \text{Age}) &= \operatorname{sigmoid}(0.5 \cdot \text{Age} - 1) \\
    \P(\text{Order} = 1 |  \text{Disease, Age}) &= \operatorname{sigmoid}(2 \cdot \text{Disease} + 0.5 \cdot \text{Age} - 1)\\
    \text{Test Result} |  \text{Order} = 1, \text{Disease} &\sim 
    \cN(-0.5 + \text{Disease}, 1)
\end{align*}
where if $\text{Order} = 0$, the test result is a placeholder value of zero.
\begin{figure}[t]
    \centering
    \resizebox{\textwidth}{!}{\input{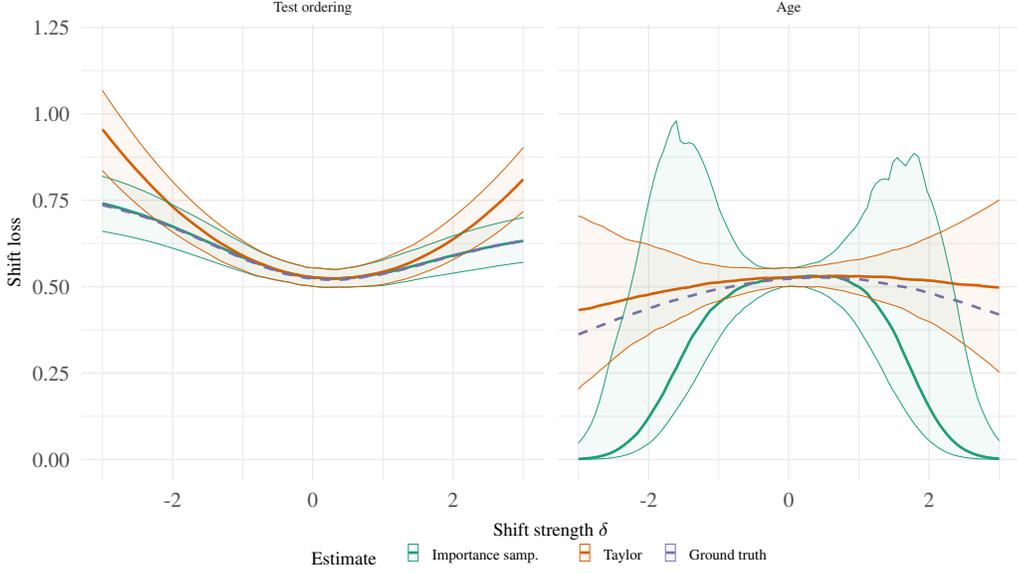}}
    \caption{We plot the mean and confidence intervals of $\taylor$ and $\ipw$ when the shifted loss as in the lab test ordering example \cref{ex:lab_testing_rates}. (Left) We consider a shift in the logits of ordering lab tests from $\eta(Z)$ to $\eta(Z) + \delta_0$. (Right) We consider a shift in the mean of Age. In the observed distribution $\eta = \mu/\sigma = 0$ and we shift to a mean of $\eta = \delta$. 
    }
    \label{fig:compare-variances-ipw-and-taylor}
\end{figure}

We consider either a shift in the logits of ordering lab tests $\eta_\delta(Z) = \eta(Z) + \delta$ (\cref{fig:compare-variances-ipw-and-taylor} left) or a mean shift in the Gaussian distribution of age $\eta_\delta = \delta$ (\cref{fig:compare-variances-ipw-and-taylor} right). 
For each $\delta$ in a grid, we compute estimates $\ipw$ and $\taylor$ of the loss under a shift of size $\delta$, 
We repeat this $n=1{,}000$ times, and plot the mean and point-wise prediction intervals (the pointwise $0.05$ and $0.95$ quantiles) for $\ipw$ and $\taylor$.
We also simulate ground truth data from $\P_\delta$, to compute the actual loss under shift.

For shifts in the binary variable (\cref{fig:compare-variances-ipw-and-taylor}, left), both estimates capture the loss well for small shifts, but as $\delta$ gets larger, the quadratic approximation increasingly deviates from the true mean; the importance sampling estimate remains very close to the ground truth shifted loss.
On the contrary, for the Gaussian mean shift (\cref{fig:compare-variances-ipw-and-taylor}, right), the importance sampling weights are ill-behaved, and the variance dramatically increases as $\delta$ becomes larger. This supports the intuition, that while importance sampling tends to work well for binary variables, the variance can be large in continuous distributions, such as the Gaussian distribution.

\subsection{Comparison of theoretical variance of reweighting and Taylor estimates}\label{sec:variance-ipw-vs-taylor}
\begin{example}\label{ex:variance-ipw-vs-taylor}
    To demonstrate the reduction in variance obtained from using the Taylor approximation of the importance weights, we consider a simple example where $\P(X) \sim \cN(0, 1)$ and $\P_\delta(X) \sim \cN(\delta, 1)$ and we wish to estimate $\E_\delta[\ell(X)]$ for some loss function $\ell(X)$.\footnote{In practice one would not use importance sampling estimation for such a simple shift, but use other approaches, such as analytically work out an estimate of $\E_\delta[\ell]$.}
    The importance sampling weights are given by $w_\delta(X) = \exp(-\tfrac{1}{2}\delta^2 + X\cdot \delta)$, and the shift gradient and Hessians are $\sg^1 = \E[\ell(X)X]$ and $\sg^2 = \E[\ell(X) X^2]$. 
    
    Therefore samples $X_1, \ldots, X_n$ from $\P$ consider the estimators, for any loss function $\ell(X)$, two estimators of $\E_\delta[\ell]$ are
    \begin{align*}
        \hat\mu_{\text{IS}} = \frac{1}{n}\sum_{i=1}^n w_\delta(X_i)\ell(X_i) \quad\text{and}\quad
        \hat\mu_{\text{Taylor}} = \frac{1}{n}\sum_{i=1}^n \ell(X_i) + \delta\cdot \ell(X_i) X_i + \tfrac{1}{2} \delta^2 \ell(X_i)X_i^2, 
    \end{align*}
    and the variances of the estimators are
    \begin{align*}
        \var(\hat\mu_{\text{IS}}) &= \frac{\E[\{\ell(X + 2\delta)\}^2]}{n}\exp(\delta^2) \\
        \var(\hat\mu_{\text{Taylor}}) &= \frac{\var\left(\ell(X) + \delta X\ell(X) + \tfrac{1}{2}\delta^2 X^2 \ell(X)\right)}{n}.
    \end{align*}
    The variance of $\hat\mu_{\text{Taylor}}$ grows like $\delta^4$ and the variance of 
    $\hat\mu_{\text{IS}}$ grows exponentially fast (unless $\E[\{\ell(X + 2\delta)\}^2]$ also diminishes exponentially fast, which is generally not the case), and so except for small $\delta$, the variance of the importance sampling estimator will be orders of magnitude larger than the variance of the estimator using the Taylor approximation. 
    While, $\hat\mu_{\text{IS}}$ is an unbiased estimator of $\E_\delta[\ell(X)]$ and $\hat\mu_{\text{Taylor}}$ is a biased, the overall mean squared error will be smaller for the Taylor approximation, unless the bias of the Taylor approximation also grows exponentially. 
    
    For the sake of analysis, consider the simple example $\ell(X) = X$. In this case, the Taylor estimate is unbiased because $\E_\delta[X] = \delta$ is a linear function of $\delta$, so the quadratic approximation is adequate. Further, the variances are given by
    \begin{align*}
        \var(\hat\mu_{\text{IS}}) = \frac{\exp(\delta^2)(1 + 4\delta^2)-\delta^2}{n} \quad\text{and}\quad \var(\hat\mu_{\text{Taylor}}) =\frac{1 + 5\delta^2 + \tfrac{15}{4}\delta^4}{n}. 
    \end{align*}
    In particular, the variance of the importance sampling estimate grows like $\exp(\delta^2)$ while that of the Taylor estimate grows like $\delta^4$.
\end{example}

\subsection{Comparison of variance of reweighting and Taylor estimates in a simple synthetic example}\label{experiment:compare-ipw}
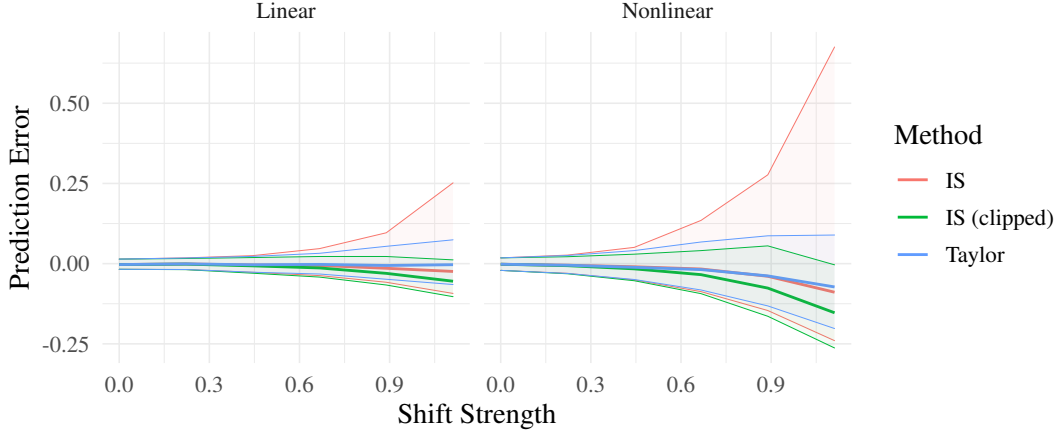
\begin{figure}[t]
    \centering
    \resizebox{\textwidth}{!}{\input{figures/compare_ipw}}
    \caption{Median and quantiles of the error in predicting $\E_\delta[\ell]$ under a shift $\delta$. }
    \label{fig:compare-ipw}
\end{figure}
In this experiment, we compare the variance of importance sampling and Taylor estimates in a simple synthetic example. 
We simulate data from $\P$ where $X \in \R^{3}$ and $Y\in\R^{1}$ depend either linearly or quadratically on $W\in\R^{3}$,
\begin{equation*}
    W \sim \mathcal{N}(0, \operatorname{Id}_3) \qquad \text{and}\qquad 
    \begin{pmatrix} X \\ Y \end{pmatrix} = (\operatorname{Id}_4 - B)^{-1}M(W + \alpha (W \odot W) + \epsilon),
\end{equation*}
where $\odot$ refers to entrywise multiplication, $\epsilon \sim \mathcal{N}(0, \operatorname{Id}_4)$, $\alpha$ is either 0 (linear) or $\frac{1}{2}$ (nonlinear) and
\begin{align*}
    B := \begin{pmatrix}
       2 & 1 & 0 & 1\\
       2 & 2 & 0 & 3\\
       3 & 3 & 0 & 2\\
       4 & 2 & 4 & 0
    \end{pmatrix} \qquad\text{and}\qquad
    M := \begin{pmatrix}
       2 & 1 & 0 \\
       2 & 1 & 1 \\
       2 & 2 & 0 \\
       4 & 1 & 1
    \end{pmatrix}.
\end{align*}

On the simulated data from $\P$, we then fit a linear predictor $f(X)$ of $Y$, and consider a shift in the mean of $W$ from $\P(W) \sim \mathcal{N}(0, \operatorname{Id}_3)$ to $\P_\delta(W) \sim \mathcal{N}(\delta, \operatorname{Id}_3)$, where $\delta = [s, s, s]^\top$ for some shift strength $s > 0$. We then compute the shift gradient $\sg^1 = \cov(\ell,W)$ and Hessian $\sg^2 = \cov(\ell,WW^\top)$, and approximate $\E_\delta[\ell]$ by $\taylor$ (see \Cref{eq:taylor-approx}).
In the linear data, the Taylor approximation is exact (see \Cref{sec:linear-model-exact}), such that any prediction error can be attributed to finite-sample fluctuation, whereas both model misspecification and finite-sample fluctuation contribute to the error in the nonlinear setting. 

Similarly, we estimate $\E_\delta[\ell]$ by importance sampling, $\E_\delta[\ell] = \E[w_\delta(W) \ell] \approx \tfrac{1}{n} \sum w_\delta(W) \ell$, where $w_\delta(W) = \frac{\P_\delta(W)}{\P(W)} = \delta^\top W - \frac{1}{2}\delta^\top \delta$, and compare this to ground truth data sampled from $\P_\delta$; we do the same for an importance sampling estimator with weights `clipped' at the $99\%$ quantile. 

We compare the predicted loss $\E_\delta[\ell]$ by actually simulating data from $\P_\delta$ and evaluating $\E_\delta[\ell]$ (where $\ell$ is still the model trained on data from $\P$). We then compute the prediction error, as the difference $\E_\delta[\ell] - \taylor$ or $\E_\delta[\ell] - \ipw$. 

For a number of different shift strengths $s$, we repeat this procedure $M=1{,}000$ times, and in \cref{fig:compare-ipw} we plot the median and a confidence interval defined by the $2.5$ and the $97.5\%$ quantiles of the prediction error.

In the linear case, both the importance sampling and the Taylor approximation retains a median error close to $0$, with the variance of $\ipw$ being larger than $\taylor$. The clipped importance sampling estimate has a smaller variance than that of ordinary importance sampling, though the median deviates further from $0$, and the variance is not smaller than that of the Taylor estimate. 

In the non-linear cases, all three models underestimate the shifted loss. For $\taylor$, this happens because as the mean of $W$ shift, the mean shift is amplified by the non-linearity, such that the quadratic approximation of the loss is an underestimate. 
While the variance of the clipped importance sampling is smaller than the variance of the ordinary importance sampling estimate and comparable to the variance of the Taylor estimate, this prediction is further from $0$ than the Taylor estimate.

Since importance sampling methods are known to produce very large outliers, the use of the median and quantiles, as opposed to the mean an confidence intervals based on the standard deviation, is favouring importance sampling; the Taylor method looks even more favourable if we instead plot the mean and standard deviations. 

\subsection{The bound in \texorpdfstring{\cref{thm:taylor-approximation-bound}}{} under covariate shift}%
\label{subsec:example-of-bound}
The bound in \cref{thm:taylor-approximation-bound} is in a general form that applies to any shift in the CEF framework. In concrete cases, the bound can be made simpler, as we now demonstrate.

Suppose that $X$ is a covariate that is Gaussian distributed $\mathcal{N}(0, 1)$.
Also consider a prediction target $Y \coloneqq f_0(X) + \epsilon$ for some function $f_0$ and noise variable $\epsilon$ that is independent of $X$.

Suppose we consider a predictor $\hat{Y} = f(X)$ and apply our proposed methodology to estimate the mean squared prediction error when predicting $Y\approx f(X)$ under a mean shift of size $\delta\in\R$ to $X$. When we only consider shifts in the mean (and not the variance), the sufficient statistic is $T(X)=X$.
We can use \cref{thm:taylor-approximation-bound} to bound the prediction error.
In this setting, 
\begin{align*}
  \ell = (Y - \hat{Y})^2 = (f_0(X) - f(X) + \epsilon)^2 \quad\text{and}\quad \epsilon_{t\cdot \delta T} = X - t\cdot\delta,
\end{align*}
such that the bound in \cref{thm:taylor-approximation-bound} becomes
\begin{align*}
    &\bigg|\E_\delta[\ell] - \taylorPop\bigg| \\ 
    &\leq \tfrac{1}{2} \sup_{t\in[0,1]} \bigg|\cov_{t\cdot\delta}\big((f_0(X) - f(X) + \epsilon)^2, (X - t\cdot\delta)^2\big) \\
    &\qquad
    - \cov\big((f_0(X) - f(X) + \epsilon)^2, (f_0(X) - f(X) + \epsilon)^2, X^2\big)\bigg| \cdot \delta^2.
\end{align*}
The subscript $\cov_{t\cdot\delta}$ indicates that the covariance is taken in the distribution $\mathcal{N}(t\cdot\delta, 1)$; instead we can write this in the observed distribution, and add $t\cdot\delta$ to $X$. Further, the terms relating to $\epsilon$ disappear, as they are independent of $X$. Thus, if we define the modelling error $g(x) = f_0(x) - f(x)$, we can write
\begin{align*}
  \bigg|\E_\delta[\ell] - \taylorPop\bigg| \leq \tfrac{1}{2} \sup_{t\in[0,1]} \bigg|\cov\big(g(X+t\cdot\delta)^2 - g(X)^2, X^2\big)
  \bigg| \cdot \delta^2.
\end{align*}
We can bound the covariance using the inequality $\cov(A, B) \leq \sqrt{\var(A)\var(B)}$,
\begin{align*}
  \bigg|\E_\delta[\ell] - \taylorPop\bigg| \leq \tfrac{1}{2} \sup_{t\in[0,1]} \bigg|\sqrt{\var\left((g(X+t\cdot\delta)^2 - g(X)^2\right)}\bigg|\cdot \bigg| \sqrt{\var(X^2)}\bigg| \cdot \delta^2.
\end{align*}
The first term on the right hand side is the variance of the difference of approximation error in $X$ and in $X + t\delta$.
If we are willing to make assumptions on the quality of the approximation $f$, we can simplify this further. 
For example, we can assume that $|g(x)^2 - g(y)^2| \leq C\cdot |x - y|^2$, meaning that the squared error of $f_0(x) - f(x)$ does not change faster than quadratically in $x$. In that case, we get
\begin{align*}
  \bigg|\E_\delta[\ell] - \taylorPop\bigg| \leq \tfrac{1}{2} C \bigg| \sqrt{\var(X^2)}\bigg| \cdot \delta^4.
\end{align*}
In some cases, one can sharpen this bound by using prior knowledge about the data generating mechanism (for example, the data generating function $f_0$ may be bounded). 

\section{Limitations of worst-case conditional subpopulation shift for defining plausible robustness sets}%
\label{sec:comparison_to_subpopulation_shift_app}

For the example in~\cref{sec:illustrative_example_worst_case_parametric_shifts}, we can contrast the type of shift we consider with the worst-case $(1 - \alpha)$-conditional subpopulation shift considered by \citet{Subbaswamy2020-vr}.

In this section, we will make the following points: First, worst-case conditional $(1 - \alpha)$-subpopulation shifts can be too pessimistic, with even moderate values of $\alpha$ leading to implausible conditional distributions. Second, we will argue that parametric robustness sets enable more fine-grained control over the set of plausible shifts, leading to more informative estimates of worst-case risk. Overall, we argue that the two approaches are complementary, with different strengths.

Before we proceed, we define a conditional $(1 - \alpha)$ subpopulation shift. A $(1 - \alpha)$ subpopulation shift in the conditional distribution $\P(O |  Y)$ is defined by a weighting function $h: \cO \times \cY \mapsto [0, 1]$, which has the property that $\E[h(O, Y) |  Y] = 1 - \alpha$ for all values of $Y$.  This can be used to construct a worst-case objective, which measures the worst-case loss under such a shift:
\begin{align}
  \sup_{h: \{0, 1\}^2 \mapsto [0, 1]} \qquad & \frac{1}{(1 - \alpha)} \E[h(O, Y) \mu(O, Y)]\label{eq:cond_subpopulation_original} \\
  \text{s.t.} \qquad & \E[h(O, Y) |  Y = y] = 1 - \alpha, \quad \text{for } y \in \{0, 1\} \nonumber
\end{align}
where $\mu(O, Y) \coloneqq \E[\ell(Y, f) |  O, Y]$, for a predictor $f$ and loss $\ell$.  This has the effect of leaving the distribution $\P(Y)$ untouched, while changing the conditional distribution $\P(O |  Y)$. Throughout this section, we will use the same predictor $f(O, L)$ described in~\cref{sec:illustrative_example_worst_case_parametric_shifts}. The rest of this section is structured as follows:

In~\cref{sec:feasible_conditional_subpopulations}, we derive the feasible set of conditional distributions $\P(O |  Y)$ implicitly considered by this objective in the simple generative model of Section~\ref{sec:illustrative_example_worst_case_parametric_shifts}, which only involves variables $O, L$ and $Y$.  We do so by showing that (for discrete $O, Y$), maximizing~\cref{eq:cond_subpopulation_original} over $h$ is equivalent to solving a linear program, where we can characterize the constraints on $h$ exactly, and translate them into constraints on $\P(O = 1|  Y = 1), \P(O = 1 |  Y = 0)$.  Here, we show that the resulting feasible set is quite large, even for moderately large subpopulations.  In particular, whenever  $(1 - \alpha) < \min \{\P(O = 1 |  Y = 0), \P(O = 0 |  Y = 1) \}$, all conditional distributions are possible.  

In~\cref{sec:worst_case_conditional_subpopulations_in_section_sec}, we derive the value of $h$ that maximizes~\cref{eq:cond_subpopulation_original}, and show that, as we vary $\alpha$, the worst-case shift is always in the same \enquote{direction} probability space:  Healthy patients $(Y = 0)$ are tested more, and sick patients $(Y = 1)$ are tested less, and for $\alpha < 0.27$, the worst-case subpopulation shift is the (unrealistic) scenario where healthy patients are always tested, and sick patients are never tested.

In~\cref{sec:iterating_with_domain_experts}, we illustrate how this type of behavior can be avoided with our approach.  We first give a parameterized shift function $s(Z; \delta_0, \delta_1)$ such that we can reach any conditional distribution of $\P(O |  Y)$, for sufficiently large values of $\delta_0, \delta_1$.  We then demonstrate how an iterative process might play out with domain experts, where we consider different constraint sets until we find a constraint set that contains plausible shifts.

\subsection{Feasible conditional subpopulations in~\texorpdfstring{\cref{sec:illustrative_example_worst_case_parametric_shifts}}{}}%
\label{sec:feasible_conditional_subpopulations}

For the simple example in~\cref{sec:illustrative_example_worst_case_parametric_shifts}, we give a self-contained derivation of the feasible region for $1 - \alpha$ conditional subpopulations in the distribution $\P(O |  Y)$.   The advantage of working with this simple generative model is that the conditional distribution can be described by only two numbers, $\P(O = 1 |  Y = 1)$ and $\P(O = 1 |  Y = 0)$, and so we can visualize the resulting conditional distribution.

Because $O, Y$ are discrete, the worst-case subpopulation in this simple example can be solved via a linear program, for a fixed $\alpha$.  We have an optimization problem in two variables, since $h_{11}\P(O = 1 |  Y = 1) + h_{01}\P(O = 0 |  Y = 1) = 1 - \alpha$, and likewise for $h_{10}, h_{00}$, where $h_{ij} = h(O = i, Y = j)$.  We also have the constraint that each variable must live in $[0, 1]$.  Meanwhile, the loss to maximize is a linear function, as an expectation of $\E[h(O, Y) \mu(O, Y)]$, where $\mu(O, Y)$ takes on four possible values, where we write $p_{ij} = \P(O = i |  Y = j)$, and $\mu_{ij}$ similarly.
\begin{align}
  \max_{h \in \R^{2 \times 2}} \quad & h_{00} \mu_{00} + h_{10} \mu_{10} + h_{01} \mu_{01} + h_{11} \mu_{11} \label{eq:cond_subpopulation_lp} \\
  \text{s.t.,}\quad & h_{11} p_{11} + h_{01} (1 - p_{11}) = 1 - \alpha \nonumber \\
       & h_{10} p_{10} + h_{00} (1 - p_{10}) = 1 - \alpha \nonumber \\
       & 0 \leq h_{ij} \leq 1, \forall i, j  \nonumber 
\end{align}
This linear program is simple enough to solve by hand, and we will do here to build intuition.  In this section, we begin by characterizing the feasible region of $h$, and then translating that into a feasible region for $\P_h(O |  Y)$, which we can plot in two dimensions.

\textbf{Characterizing feasible values of $h$}: Here, we focus on characterizing the feasible set that $h$ can lie in, as a way of characterizing the feasible set for $\P(O |  Y)$. From the constraints, we can write that 
\begin{align*}
  h_{11} p_{11} + h_{01} (1 - p_{11}) &= 1 - \alpha &\implies&& h_{01} &= \frac{1 - \alpha - h_{11} p_{11}}{1 - p_{11}}\\
  h_{10} p_{10} + h_{00} (1 - p_{10}) &= 1 - \alpha & \implies&& h_{00} &= \frac{1 - \alpha - h_{10} p_{10}}{1 - p_{10}}
\end{align*}
There are only two constraints on $h_{11}$: Those directly imposed by $0 \leq h_{11} \leq 1$, and those which are imposed by the equality constraint with $h_{01}$ and the fact that $0 \leq h_{01} \leq 1$.  For the latter, with some algebra we can write that 
\begin{align*}
  0 &\leq \frac{1 - \alpha - h_{11}p_{11}}{1 - p_{11}} \leq 1
                              &\implies&&  \frac{p_{11} - \alpha}{p_{11}} &\leq h_{11} \leq \frac{1 - \alpha}{p_{11}}
\end{align*}
So that the constraints on $h_{11}$ become
\begin{equation}
  \max\left\{0, \frac{p_{11} - \alpha}{p_{11}}\right\} \leq h_{11} \leq \min\left\{1, \frac{1 - \alpha}{p_{11}}\right\} \label{eq:cond_subpopulation_h11_constraint}
\end{equation}
which recovers our intuition that if $\alpha = 0$, it must be that $h_{11} = 1$ and $h_{01} = 1$.

\textbf{Bounding feasible values of $\P_{h}(O |  Y)$} 
The parameters $h$ can be understood as importance weights whose expectation is $1 - \alpha$ instead of $1$, that reweight $\P$ to a new distribution $\P_h$ when appropriately normalized.  To compute conditional probabilities $\P_h(O = i |  Y = j)$ under the new distribution, we can compute the expectation of $\1{O = i, Y = j}$, and normalize by $\P(Y = j)$.
\begin{align*}
  \P_{h}(O = i, Y = j) &= \frac{1}{1 - \alpha} \E[h(O, Y) \1{O = i, Y = j}] = \frac{h_{ij}}{1 - \alpha} \P(O = i, Y = j) \\
  \implies \P_{h}(O = i |  Y = j) &= \frac{h_{ij}}{1 - \alpha} \P(O = i |  Y = j)
\end{align*}
where the implication follows from the fact that $\P_{h}(Y) = \P(Y)$.  This allows us to translate bounds on $h_{ij}$ directly into bounds on $\P_h(O = i |  Y = j)$.  Making use of~\cref{eq:cond_subpopulation_h11_constraint}, we can write that
\begin{equation*}
  \max\left\{0, \frac{p_{11} - \alpha}{p_{11}}\right\} \cdot \frac{p_{11}}{1 - \alpha} \leq \P_h(O = 1 |  Y = 1) \leq \min\left\{1, \frac{1 - \alpha}{p_{11}}\right\} \cdot \frac{p_{11}}{1 - \alpha}
\end{equation*}
which yields 
\begin{equation*}
  \max\left\{0, \frac{p_{11} - \alpha}{1 - \alpha}\right\} \leq \P_h(O = 1 |  Y = 1) \leq \min\left\{\frac{p_{11}}{1 - \alpha}, 1 \right\}
\end{equation*}

We can apply a similar logic to $h_{10}$, which is identical except for $p_{11}$ being replaced by $p_{10}$, yielding 
\begin{equation*}
  \max\left\{0, \frac{p_{10} - \alpha}{1 - \alpha}\right\} \leq \P_h(O = 1 |  Y = 0) \leq \min\left\{\frac{p_{10}}{1 - \alpha}, 1 \right\}
\end{equation*}

\textbf{Visualizing the constraint set}: \Cref{fig:subpopulation_worst_case_shift} gives feasible conditional distributions under different values of $\alpha$.  We can observe that when $\alpha = 0.8$, all conditional distributions are feasible, including the distribution where $\P(O = 1 |  Y = 0) = 1$ and $\P(O = 1 |  Y = 1) = 0$, representing the case where every healthy patient gets tested, and no sick patients receive a test.  This is generally possible in this example whenever $1 - \alpha < \min \{\P(O = 1 |  Y = 0), \P(O = 0 |  Y = 1)\}$, as it permits the following subpopulation function, which yields this result.
\begin{equation*}
  h(O = o, Y = y) = \frac{1 - \alpha}{\P(O = o |  Y = y)} \1{o \neq y}
\end{equation*}

\subsection{Worst-case conditional subpopulation shifts}%
\label{sec:worst_case_conditional_subpopulations_in_section_sec}

\begin{figure}[t]
\centering
  \begin{subfigure}[t]{0.3\textwidth}
  \centering
    \includegraphics[width=\textwidth]{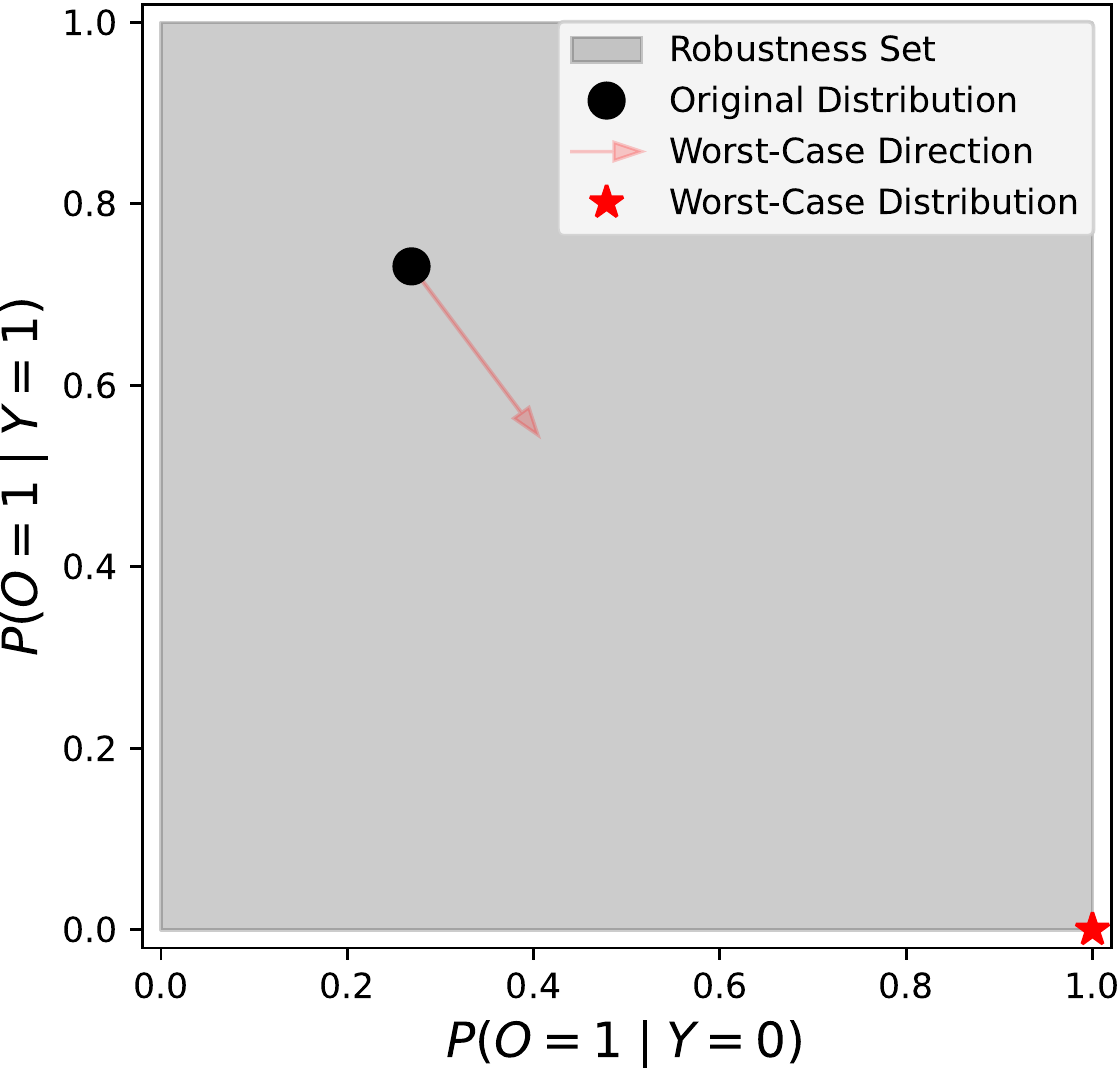}
  \caption{$(1 - \alpha) = 0.2$}%
  \label{subfig:subpopulation_worst_case_alpha80}
  \end{subfigure}
  \begin{subfigure}[t]{0.3\textwidth}
  \centering
    \includegraphics[width=\textwidth]{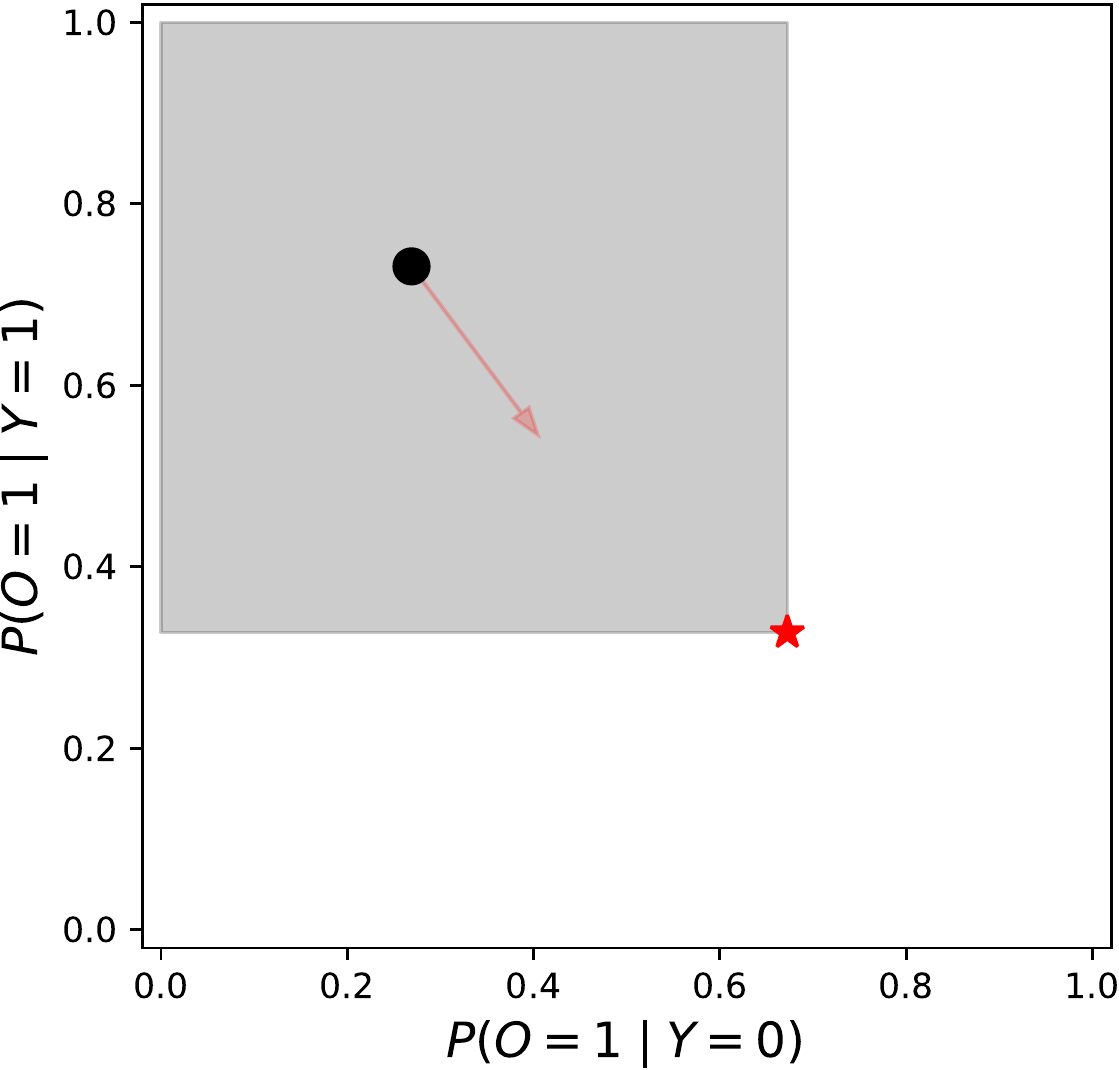}
  \caption{$(1 - \alpha) = 0.4$}%
  \label{subfig:subpopulation_worst_case_alpha60}
  \end{subfigure}
  \begin{subfigure}[t]{0.3\textwidth}
  \centering
    \includegraphics[width=\textwidth]{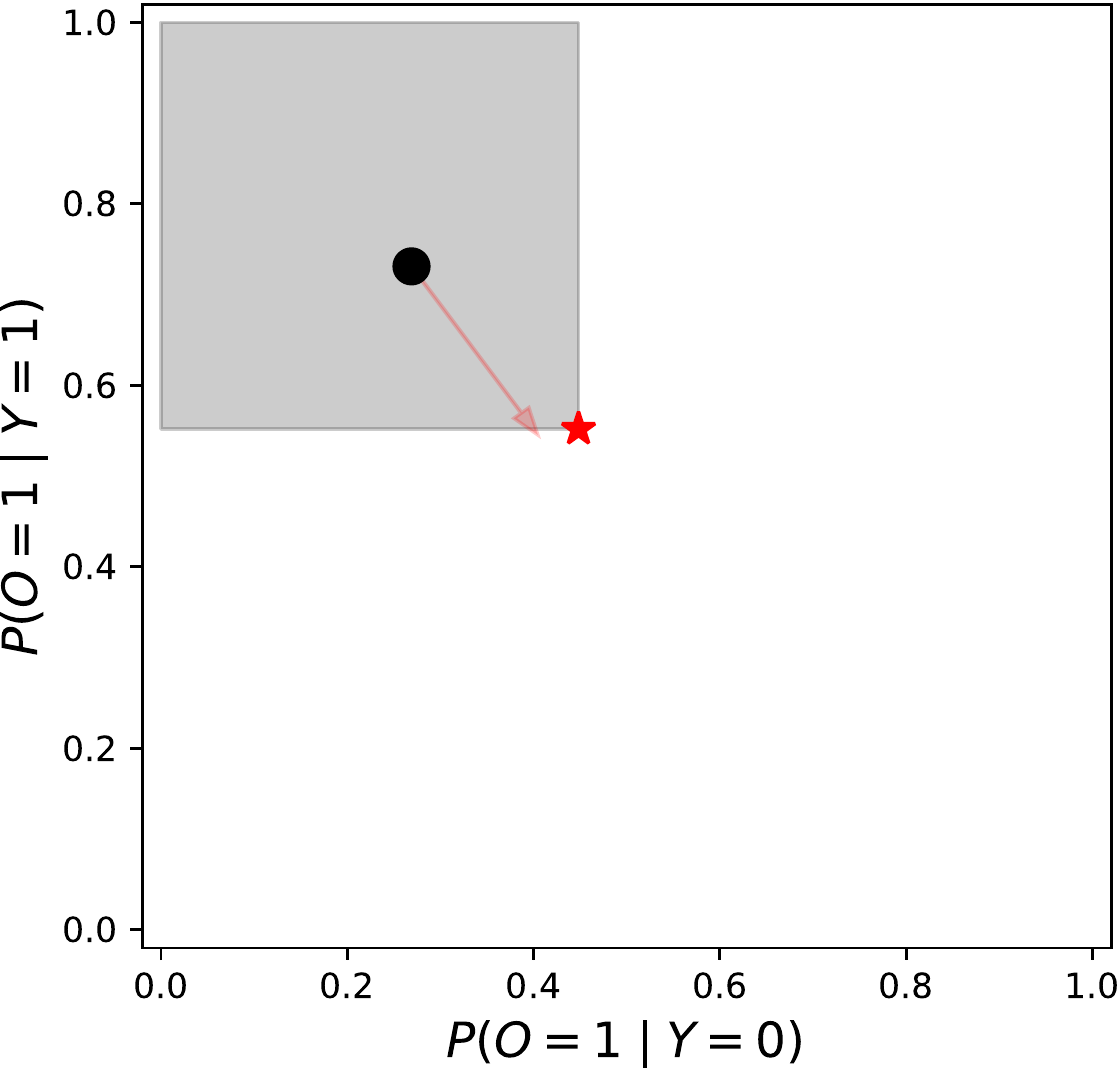}
  \caption{$(1 - \alpha) = 0.6$}%
  \label{subfig:subpopulation_worst_case_alpha40}
  \end{subfigure}
  \caption{Feasible sets, worst-case directions, and worst-case solutions for a $(1 - \alpha)$ subpopulation shift in the conditional distribution $\P(O |  Y)$ for differing values of $\alpha$. Worst-case directions are computed using~\cref{eq:cond_subpopulation_worst_case_direction}, as unit-norm vectors re-scaled to fit in the plot, and the colored dots give the worst-case solutions, all of which lie in the lower-right corner of the constraint set. The original conditional distribution is given by the black dot.}%
\label{fig:subpopulation_worst_case_shift}
\end{figure}

Given the constraint set which describes the feasible set of conditional distributions under the $(1 - \alpha)$-conditional subpopulation objective, we can derive the worst-case conditional distribution. Here, since $Y, O$ are both binary, the expected loss under a new distribution $\P_h$ is given by 
\begin{equation*}
  \E_h[\ell] = \sum_{y, o} \mu(o, y) \P_{h}(O = o |  Y = y) \P(Y = y)
\end{equation*}
which we can write in terms of the constrained probabilities $\P_{h}$ as follows, where $q_{11} \coloneqq \P_{h}(O = 1 |  Y = 1)$ and $q_{10} \coloneqq \P_{h}(O = 1 |  Y = 0)$
\begin{equation*}
  \P(Y = 1) [\mu(1, 1) q_{11} + \mu(0, 1)(1 - q_{11})] + \P(Y = 0)[\mu(1, 0) q_{10} + \mu(0, 0) (1 - q_{10})]
\end{equation*}
which also gives us a direction in which the loss is maximized, since the loss is given by 
\begin{equation}\label{eq:cond_subpopulation_worst_case_direction}
  \E_h[\ell] = q_{11} \cdot \P(Y = 1) \cdot (\mu(1, 1) - \mu(0, 1)) + q_{10} \P(Y = 0) \cdot (\mu(1, 0) - \mu(0, 0)) + C
\end{equation}
where $C = \P(Y = 1) \mu(0, 1) + \P(Y = 0) \mu(0, 0)$. Since $q_{11}, q_{10}$ can be optimized independently, the worst-case solution is given by taking the maximum value of $q_{11}$ if $\mu(1, 1) > \mu(0, 1)$ and the minimum value if $\mu(1, 1) < \mu(0, 1)$, and likewise taking the maximum value of $q_{10}$ if $\mu(1, 0) > \mu(0, 0)$, and the minimum value otherwise.  If $\mu(1, 1) = \mu(0, 1)$ or $\mu(1, 0) = \mu(0, 0)$, then the objective is unaffected by the choice of $q_{11}$ or $q_{10}$ respectively.

\textbf{Visualizing the worst-case conditional distributions} The worst-case directions on the probability scale, and the resulting worst-case conditional distribution obtained by solving~\cref{eq:cond_subpopulation_lp}, are given in~\cref{fig:subpopulation_worst_case_shift}.  The red line arrow visualizes the direction from~\cref{eq:cond_subpopulation_worst_case_direction}, and the worst-case distribution is the point which is furthest in this direction in the constraint set.  Here, we are finding the worst-case accuracy of the same predictive model $f(O, L)$ described in \cref{sec:illustrative_example_worst_case_parametric_shifts}.  We can observe that the worst-case loss is obtained by seeking to reverse the correlation between $Y$ and $O$, decreasing the probability that a sick patient $(Y = 1)$ gets a test ordered, and increasing the probability that a healthy patient $(Y = 0)$ gets a test ordered.

\subsection{Iterating with domain experts to define realistic parametric robustness sets}%
\label{sec:iterating_with_domain_experts}

\begin{figure}[t]
\centering
  \begin{subfigure}[t]{0.33\textwidth}
  \centering
    \includegraphics[width=0.95\textwidth]{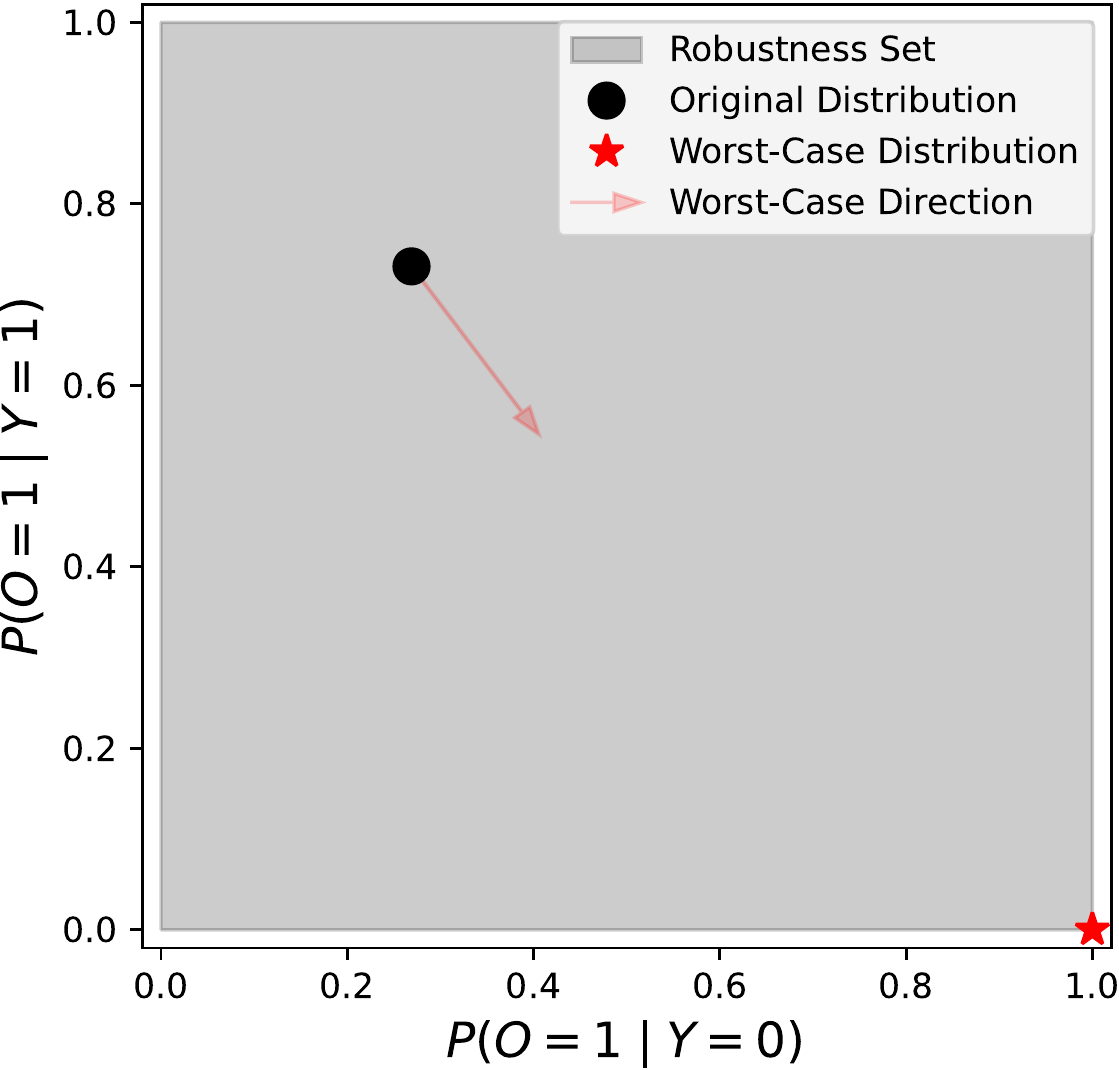}
  \caption{}%
  \label{subfig:user_story_0}
  \end{subfigure}%
  \begin{subfigure}[t]{0.33\textwidth}
  \centering
    \includegraphics[width=0.95\textwidth]{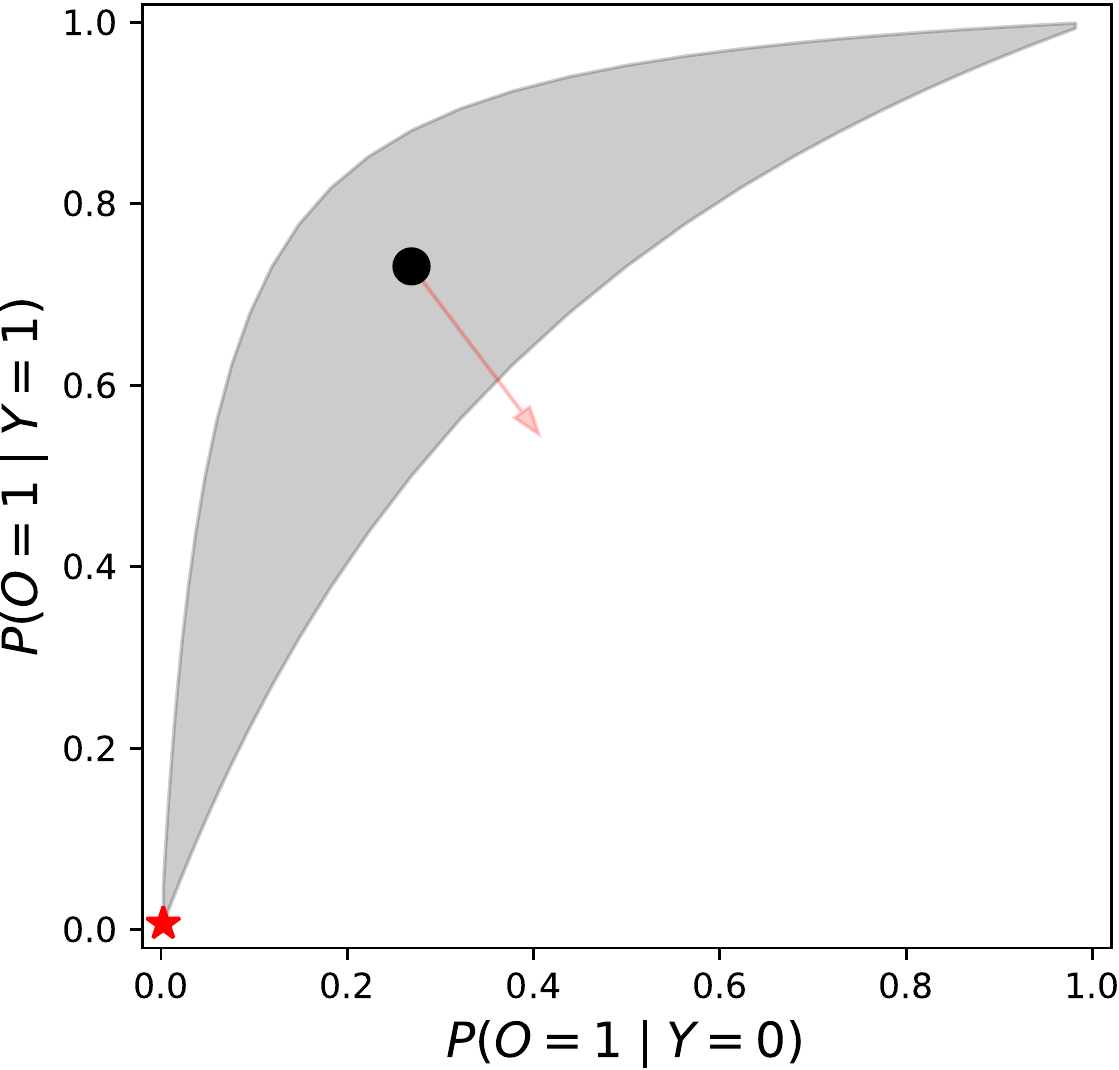}
  \caption{}%
  \label{subfig:user_story_1}
  \end{subfigure}%
  \begin{subfigure}[t]{0.33\textwidth}
  \centering
    \includegraphics[width=0.95\textwidth]{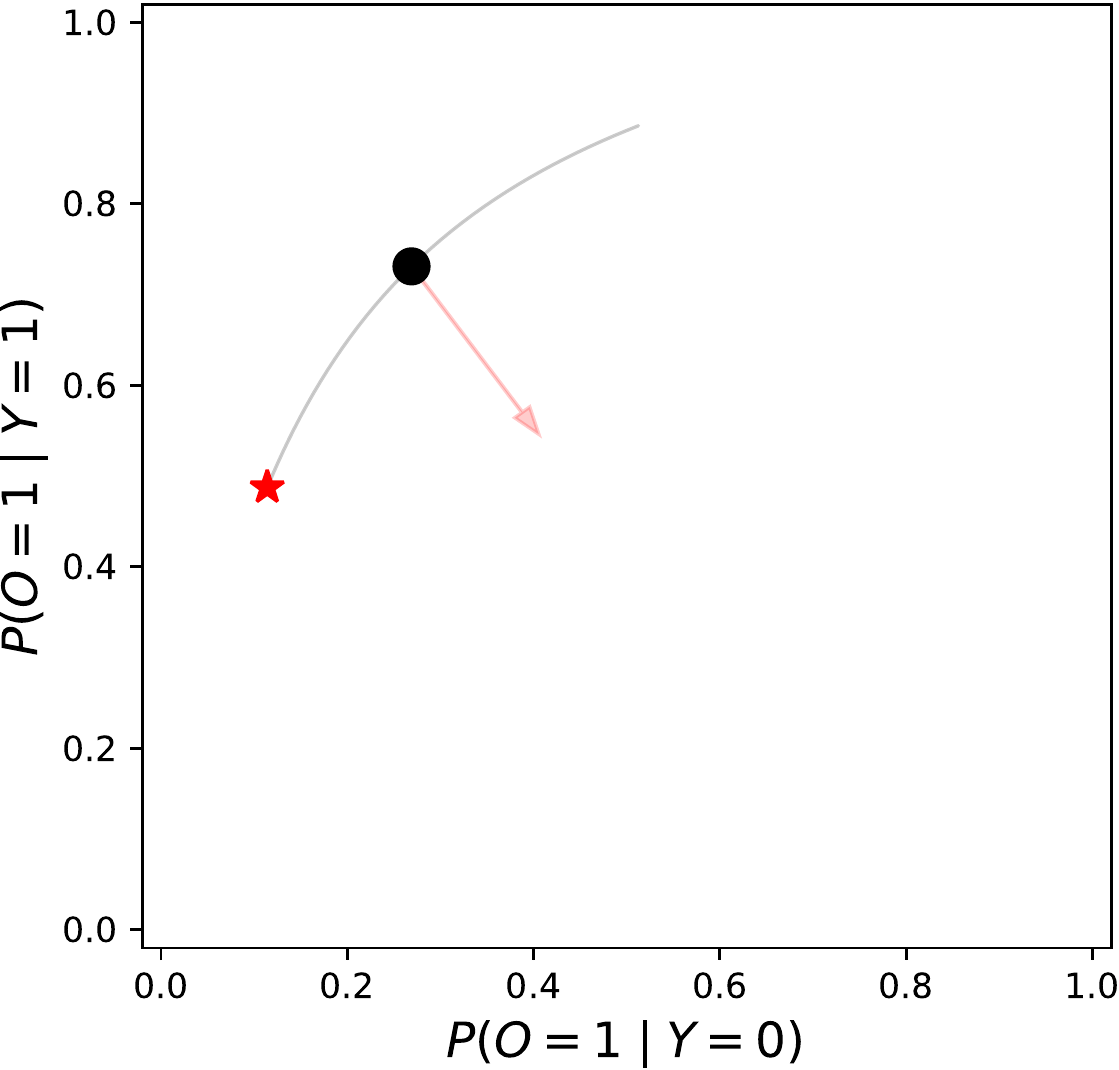}
  \caption{}%
  \label{subfig:user_story_2}
  \end{subfigure}%
  \caption{Each figure shows the set of conditional probability distributions (\enquote{CPDs}) $\P(O |  Y)$ that can be represented by a shift of $(\delta_0, \delta_1) \in \Delta_0 \times \Delta_1$, along with the worst-case distribution (given by the red star) for the 0--1 loss.  In this example, the expected loss under $\P_{\delta}$ is a linear function of the two conditional probabilities (see~\Cref{sec:worst_case_conditional_subpopulations_in_section_sec}), where the loss increases along the red arrow. (\subref{subfig:user_story_0}) captures (nearly) all conditional probability distributions, with $\Delta_0, \Delta_1$ unconstrained. (\subref{subfig:user_story_1}) shows a set of CPDs with $\Delta_0$ unconstrained, and $\Delta_1 = [-1, 1]$, with resulting worst-case accuracy of 50\%.  (\subref{subfig:user_story_2}) shows a more restrictive set of shifts, where $\Delta_0 = [-1.05, 1.05], \Delta_1 = \{0\}$.  The worst-case accuracy in this case is 69\%, comparable to the accuracy of 75\% on the original distribution.}%
\label{fig:user_story}
\end{figure}

In the previous sections, we saw that $(1 - \alpha)$-conditional subpopulation shift does not always produce realistic worst-case conditional distributions.  Moreover, given only the parameter $\alpha$, there is limited ability to control the nature of the resulting worst-case conditional distribution $\P(O |  Y)$.  In this section, we contrast this limitation with the finer-grained control enabled by considering parametric robustness sets.  In particular, we argue that parametric shifts allow for end-users to customize robustness sets, ruling out shifts that represent unrealistic changes.

In practice, we imagine that the following iterative process could be a useful tool in model development:
\begin{enumerate*}[label= (\roman*)]
  \item Define a class of shifts with an appropriate $s(Z;\delta)$ and constraint set $\Delta$, and search for a worst-case shift $\delta$.
  \item Present to domain experts \textbf{both} the worst-case shift $\delta$ (in terms of summary statistics of the resulting distribution $\P_{\delta}$) alongside the associated estimate of the worst-case loss.  For instance, report both the worst-case loss, as well as corresponding rate of testing among sick and healthy patients.
  \item If the shift itself is unrealistic, further the constrain parameter set or shift function, and repeat the process.
\end{enumerate*}

In~\cref{fig:user_story}, we give a concrete example. Each sub-figure shows the set of conditional probability distributions $\P(O |  Y)$ that can be represented by a shift of $(\delta_0, \delta_1) \in \Delta_0 \times \Delta_1$, along with the worst-case conditional distribution (given by the red star) for the 0--1 loss.  Recall that we use the shift function $s(Y; \delta) = \delta_{0} + \delta_{1} Y$, where $\delta_0$ controls a general increase or decrease in testing, while $\delta_1$ controls a shift in the testing rate for only sick patients, and allows for a different change in the testing rate of sick vs healthy patients.

\textit{Iteration 1:} We might imagine starting with a relatively unconstrained robustness set, where $\delta_0$ and $\delta_1$ are unconstrained. \Cref{subfig:user_story_0} shows the resulting robustness set of conditional distributions, and finds a shift with with a worst-case accuracy of 16\%, compared to accuracy of 75\% on the original distribution. 
However, the corresponding $\delta$-perturbation $\P_{\delta}$ is unrealistic, where all healthy patients (and no sick patients) are tested.  Luckily, because we have parameterized the shift, we can constrain the robustness set to exclude these types of results.

\textit{Iteration 2:} A benefit of our approach is that we can refine the robustness set, with this type of feedback in mind.  In~\cref{subfig:user_story_1}, we restrict the support of $\delta_1$ to $[-1, 1]$, to avoid large changes in the relative probability of testing sick vs healthy patients.  Here, the resulting worst-case accuracy is much higher (50\%), but the corresponding worst-case conditional probability distribution is perhaps still unrealistic:  No patients undergo laboratory testing at all!

\textit{Iteration 3:} Finally, we consider only shifts that affect all patients in a similar way, generally raising or lowering the conditional probability of a lab test, represented by shifts in $\delta_0$ alone. This may correspond to a more realistic scenario where (in a new hospital) laboratory testing use is more or less constrained. Additionally, we can specify that this shift should decrease testing rates by at most 20\%, which translates directly into a lower-bound on $\delta_0$.\footnote{In~\Cref{prop:log_odds_marginal_to_shift}, we prove that for binary random variables with a shift $\eta(Z) + \delta$, there is a one-to-one mapping between a new marginal distribution ($\P(O =1)$ in this case) and the value of the parameter $\delta$.}  \Cref{subfig:user_story_2} shows the resulting robustness set of distributions, where the worst-case shift may seem more plausible: A reduction in testing rates for both populations. The worst-case accuracy in this case is 69\%, comparable to the accuracy of 75\% on the original distribution.

\section{CelebA: Experiment details and additional results}\label{app:additional-experiments}\label{app:celeba}
In this section, we give details of the computer vision experiment in \cref{subsec:celeba-gan}. 

\subsection{Details for the experiment}
\paragraph{Creating the training distribution}
To construct the training distribution $\P$, we use the conditional GAN in \citet{kocaoglu2018causalgan}.
In particular, we use their CausalBEGAN, which is an extends the boundary equillibrium GAN \citep{berthelot2017began} to also take attributes as inputs. 
We train the CausalBEGAN using the default hyper parameters in the implementation provided by \citet{kocaoglu2018causalgan}, available under the MIT license. The model is trained for $250{,}000$ iterations on a single GPU, taking around approximately 16 hours. 

Similar to \citet{kocaoglu2018causalgan}, we use the CelebA dataset \citep{liu2015faceattributes}, which contains approximately $200{,}000$ images of faces, along $40$ binary attributes. Of those, we use the following $9$ attributes \{Male, Young, Wearing Lipstick, Bald, Mustache, Eyeglasses, Narrow Eyes, Smiling, Mouth Slightly Open\}. 
The CelebA dataset is licensed for non-commercial research purposes only, and consists of publicly available images of celebrities, which were collected from the internet. 
Although the data set has been widely used, \citet{liu2015faceattributes} do not make any mention of consent by the individuals to have the images included in the data set, and it is therefore likely that those celebrities did not provide consent. 

\paragraph{Training distribution over attributes}
\begin{figure}[t]
\centering
\begin{subfigure}{0.4\linewidth}
\centering
\input{figures/celeba-graph}
\end{subfigure}
\hspace{5cm}
\caption{Causal graph over attributes, where lightning bolts indicate changes in mechanisms. Also displayed in \cref{fig:celeba-scm}.}\label{fig:celeba-scm-appendix-copy}
\end{figure}
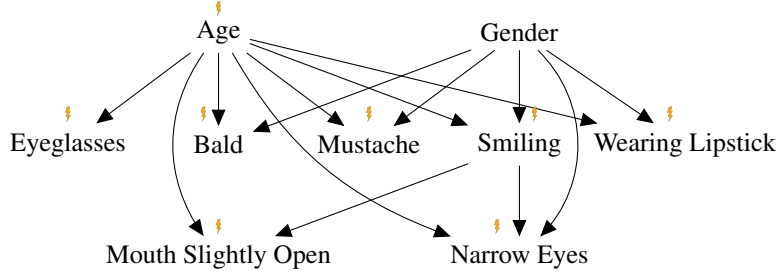

For the training distribution, we simulate binary attributes according to the structural causal model in \cref{fig:celeba-scm} (for convenience also copied to \cref{fig:celeba-scm-appendix-copy}), where the model parameters are
\begin{align*}
    \P(\text{Young} = 1) &= \sigma(0.0) \\
    \P(\text{Male} = 1) &=    \sigma(0.0) \\
    \P(\text{Eyeglasses} = 1 | \text{Young}) &= \sigma(0.0 -0.4\cdot \text{Young}) \\
    \P(\text{Bald} = 1 | \text{Young, Male}) &=    \sigma(-3.0 + 3.5\cdot \text{Male}-\text{Young}) \\
    \P(\text{Mustache} = 1 | \text{Young, Male}) &=    \sigma(-2.5 + 2.5\cdot\text{Male}-\text{Young}) \\
    \P(\text{Smiling} = 1 | \text{Young, Male}) &= \sigma(0.25- 0.5\cdot\text{Male} + 0.5\cdot\text{Young}) \\
    \P(\text{Wearing Lipstick} = 1 | \text{Young, Male}) &=    \sigma(3.0 - 5.0\cdot\text{Male} - 0.5\cdot\text{Young}) \\
    \P(\text{Mouth Slightly Open} = 1 | \text{Young, Smiling}) &= \sigma(-1.0 + 0.5\cdot\text{Young} + \text{Smiling}) \\
    \P(\text{Narrow Eyes} = 1 | \text{Male, Young, Smiling}) &= \sigma(-0.5 + 0.3\cdot\text{Male} + 0.2\cdot\text{Young} + \text{Smiling}),
\end{align*}
where each variable either takes the value $0$ or $1$ and $\sigma$ indicates the sigmoid.
To generate data, we first simulate attributes from this binary Bayesian network, which we then pass as inputs to the GAN to simulate images (in addition to the random noise used by the GANs to simulate different images). In \cref{fig:celeb-gan-examples-train-test,fig:celeb-gan-examples-characteristics}, we plot examples of the training images that were generated.

\paragraph{Predictive model}
We simulate a training set of $12{,}000$ attribute-image pairs, and a validation set of $2{,}000$ pairs. The training set is used to fit a classifier $f$, and the validation set is used for model selection.
To build a classifier $f$, we use the ResNet-50 \citep{he2016deep} model implemented in the python package \texttt{torch}. We add a final fully connected layer to adapt the ResNet model to a binary classification task, and fine-tune the model on the training data by (only) learning the weights and bias of the final layer. The model is trained using the negative log-likelihood criterion and an ADAM optimizer. The model is trained for $25$ epochs and we select the model which after a full epoch had the best validation set performance. Given the learned model $f$, we simulate a separate validation dataset of $n=1{,}000$ samples, and make model predictions $f(X)$. We then compute the model accuracy as $\ell = \1{f(X) = Y}$, which is the input to computing the shift gradient and Hessian.

\paragraph{Estimation of shifted loss}
We apply the methods in \cref{sec:approximating_the_shift_for_exp_fam} to estimate the worst-case shift to the distribution $\P$ (given by the binary probabilities above). 
For each conditional $\P(W_i|\PA(W_i))$, we consider a shift $\eta_{\delta_i}(\PA(W_i)) = \eta(\PA(W_i)) + \sum_{z\in\cZ} \1{\PA(W_i) = z}\delta_i$, which corresponds to arbitrarily shifting the conditional distribution (see \cref{sec:shift_gradient_and_hessian_binary}).
For example, for $W_i = \text{Bald}$, where $\eta(\text{Young, Male}) = -3.0 + 3.5\cdot\text{Male} - 1.0\cdot\text{Young}$, the shift would be
\begin{equation}\label{eq:example-of-shifted-eta-celeb-a}
    \eta_{\delta_{\text{Bald}}}(\text{Young, Male})  = \eta(\text{Young, Male}) + \begin{cases}
        \delta_{\text{Bald}, 0}, & \text{Young}=0, \text{Male} = 0 \\
        \delta_{\text{Bald}, 1}, & \text{Young}=0, \text{Male} = 1 \\
        \delta_{\text{Bald}, 2}, & \text{Young}=1, \text{Male} = 0 \\
        \delta_{\text{Bald}, 3}, & \text{Young}=1, \text{Male} = 1.
    \end{cases}
\end{equation}
For each $W_i$, this means that $\delta_i$ is $\R^{2^{|\PA(W_i)|}}$, and in total $\delta = (\delta_1, \ldots, \delta_8)\in \R^{31}$ (we do not consider shifts in the distribution of gender, since this is the label we are predicting). 

We compute the shift gradient and Hessian using \cref{thm:sg-cond-cov-several-shift}. In particular, since $W_i$ is binary, the sufficient statistic is $T(W_i) = W_i$, so the shift gradients and Hessians given by \cref{sec:shift_gradient_and_hessian_binary}. See \cref{appsec:algorithm_calculation_theorem1} for a detailed walk through of computing the shift gradient and Hessian from a sample. 

For any given $\delta$, the shifted distribution of $W_i$ is given by $\P_\delta(W_i = 1|\PA(W_i)) = \sigma(\eta_{\delta_i})$, where $\eta_{\delta_i}$ is computed similar to \cref{eq:example-of-shifted-eta-celeb-a}, and $\sigma$ is the sigmoid function. Then the importance sampling weights are given by
\begin{equation*}
    w_\delta = \prod_{i=1}^8 \frac{\sigma(\eta_{\delta_i}(\PA(W_i)))}{\sigma(\eta(\PA(W_i))}. 
\end{equation*}
Using these weights, for any $\delta$, we can estimate $\E_\delta[\ell]$ by $\ipw$ and $\taylor$ using \cref{eq:ipw-estimate-of-mean,eq:taylor-approx-finie-sample}, respectively.

\subsection{Full table of worst-case shift in \texorpdfstring{\Cref{subsec:celeba-gan}}{}}%
\label{subsec:full-table-1}
In \cref{subsec:celeba-gan}, we find the worst-case shift $\delta$, and display the $5$ largest components. In \cref{tab:worst-case-delta-full}, we display the full vector $\delta\in\R^{31}$, sorted by absolute value of the size of the component. 
\begin{table}[!ht]
    \centering
    \input{tables/table3}
    \caption{Worst case shift in the $\delta\in\R^{31}$ identified by the Taylor approach in \cref{subsec:celeba-gan}. Each entry corresponds to a shift in a conditional distribution given a particular outcome, and the squared sum of the entries equal $\lambda^2 = 4$.}
    \label{tab:worst-case-delta-full}
\end{table}

\subsection{Sample images from training distribution in \texorpdfstring{\Cref{subsec:celeba-gan}}{}}
\label{subsec:example-images-from-gan}
\begin{figure}[t]
    \centering
    \begin{subfigure}{0.49\textwidth}
        \includegraphics[width=0.9\textwidth]{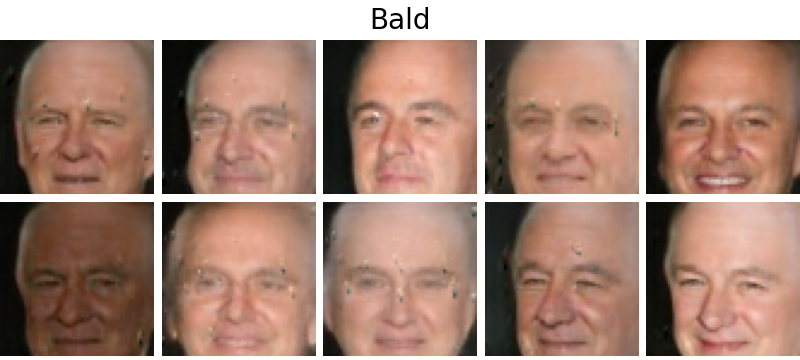}
    \end{subfigure}
    \begin{subfigure}{0.49\textwidth}
        \includegraphics[width=0.9\textwidth]{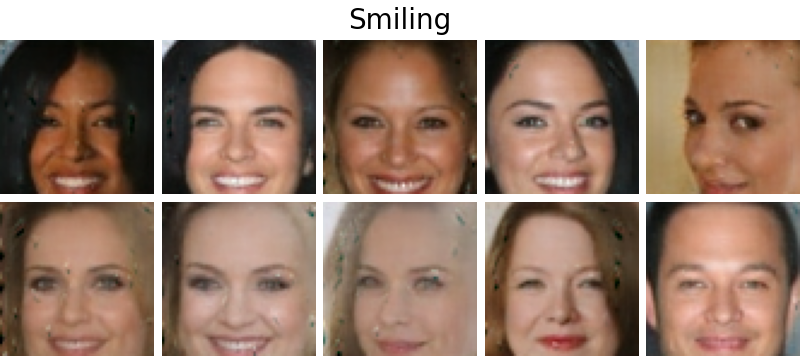}
    \end{subfigure}
    \begin{subfigure}{0.49\textwidth}
        \includegraphics[width=0.9\textwidth]{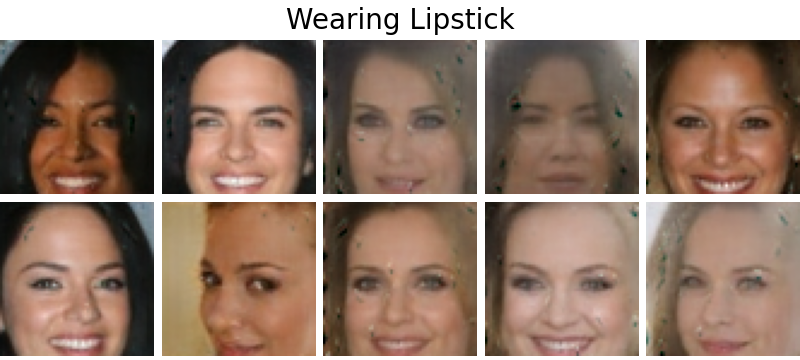}
    \end{subfigure}
    \begin{subfigure}{0.49\textwidth}
        \includegraphics[width=0.9\textwidth]{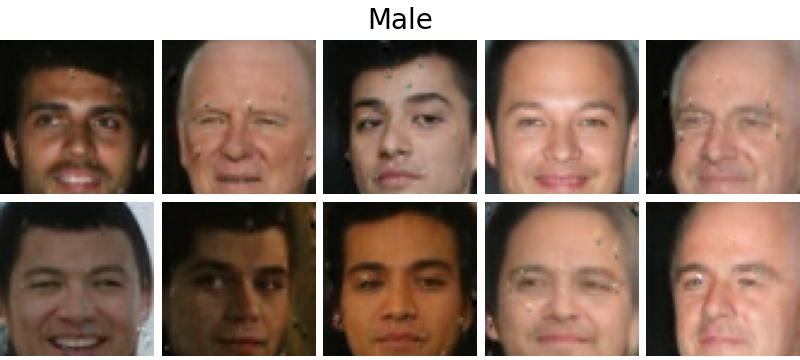}
    \end{subfigure}
    \caption{Examples of images from the training distribution $\P$. Each of the four groups (Bald, Smiling, Wearing Lipstick, Male) show training images who have that characteristic.}
    \label{fig:celeb-gan-examples-characteristics}
\end{figure}%
\begin{figure}[!t]
    \centering
    \begin{subfigure}{0.49\textwidth}
        \includegraphics[width=0.9\textwidth]{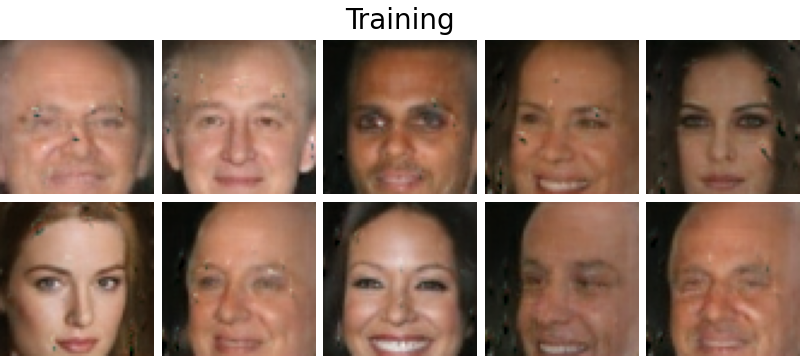}
    \end{subfigure}
    \begin{subfigure}{0.49\textwidth}
        \includegraphics[width=0.9\textwidth]{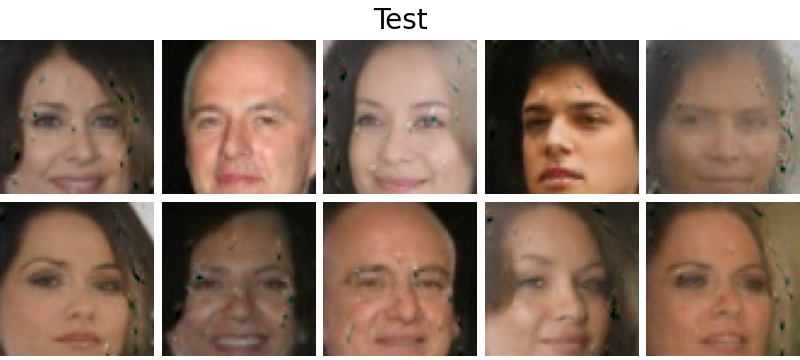}
    \end{subfigure}
    \caption{Examples of images from the training distribution $\P$ and the test distribution $\P_\delta$ that is characterized by the worst-case shift $\delta$, see \cref{fig:celeba-scm}. }
    \label{fig:celeb-gan-examples-train-test}
\end{figure}
In \cref{fig:celeb-gan-examples-characteristics}, for the $4$ attributes $\{\text{Bald, Smiling, Wearing Lipstick, Male}\}$, we display images generated from the training distribution $\P$ (i.e. by the GAN) with that particular attribute. In \cref{fig:celeb-gan-examples-train-test} we show 10 randomly drawn images from the training distribution $\P$ as well as the test distribution $\P_\delta$ corresponding to the worst-case $\delta$ found in \cref{subsec:celeba-gan}. 

\clearpage
\subsection{Impact of changing \texorpdfstring{$\lambda$}{lambda}}%
\label{sec:impact_of_changing_lambda}

The shift considered in the main text yields a relatively small drop in accuracy.  To demonstrate that larger drops in accuracy are possible, we repeated our experimental setup over the same 100 initial validation datasets, while varying the size of the constraint $\norm{\delta}_2 \leq \lambda$. We report results in~\cref{tab:lambda_comparison} for $\lambda \in [2, 4, 6, 8, 10]$, where $\lambda = 2$ corresponds to the setting of~\cref{tab:celebA-table} (right). 
\begin{table}[h]
  \centering
  \caption{Performance of the Taylor and IS approaches over different values of $\lambda$, where $\lambda = 2$ corresponds to the setting of~\cref{tab:celebA-table} (right).  Averages taken over 100 simulations.}%
  \label{tab:lambda_comparison}
  \begin{tabular}{lrrrrr}
  \toprule
  {} &  $\lambda = 2$ &  $\lambda = 4$ &  $\lambda = 6$ &  $\lambda = 8$ &  $\lambda = 10$ \\
  \midrule
  Original Acc. ($\E[\mathbf{1}\{f(X) = Y\}]$)                                  &          0.912 &          0.912 &          0.912 &          0.912 &           0.912 \\
  \midrule
  Acc.\ under Taylor shift ($\E_{\delta_{\text{Taylor}}}[\mathbf{1}\{f(X) = Y\}]$)       &          0.874 &          0.812 &          0.736 &          0.681 &           0.648 \\
  IS est.\ of acc.\ under Taylor shift ($\hat{E}_{\delta_{\text{Taylor}}, \text{IS}}$)         &          0.863 &          0.795 &          0.715 &          0.658 &           0.625 \\
  Taylor est.\ of acc.\ under Taylor shift ($\hat{E}_{\delta_{\text{Taylor}}, \text{Taylor}}$) &          0.863 &          0.798 &          0.711 &          0.601 &           0.466 \\
  \midrule
  Acc.\ under IS shift ($\E_{\delta_{\text{IS}}}[\mathbf{1}\{f(X) = Y\}]$)               &          0.889 &          0.830 &          0.746 &          0.670 &           0.596 \\
  IS est.\ of acc.\ under IS Shift ($\hat{E}_{\delta_{\text{IS}}, \text{IS}}$)                 &          0.821 &          0.670 &          0.463 &          0.264 &           0.130 \\
  \bottomrule
  \end{tabular}
\end{table}

Recall that we have two complementary goals: First, we would like to \textbf{find} a shift that results in a large drop in accuracy.  Second, we would like to reliably \textbf{evaluate} the impact of the shift that we find, using only the training data.  These two goals can be tackled with different approaches, such as using the Taylor approximation to find a shift, but using importance sampling (IS) to estimate the loss under that shift.  \Cref{tab:lambda_comparison} allows us to compare three different strategies: (i) using the Taylor approximation for both finding and evaluating the shift, (ii) using IS for both finding and evaluating, and (iii) using Taylor to find, but IS to evaluate the shift.

From~\cref{tab:lambda_comparison}, we can observe that \textbf{using Taylor to find, but IS to evaluate, consistently performs best in terms of reliable evaluation} (i.e., predicting the shifted accuracy), across all values of $\lambda$.  For $\lambda = 2$, the bias in evaluation is 1\% (predicting 86\% vs ground truth of 87\% on average), and for $\lambda = 10$, the bias of this approach is still only 2\% (predicting 63\% vs ground truth of 65\% on average).  In contrast, for $\lambda = 10$, the first strategy (using Taylor to find and evaluate) over-predicts the impact by 18\%, and the second strategy (using IS to find and evaluate) over-predicts the impact by 47\%.

This strategy \textbf{also tends to find the most impactful shifts, for moderate values of $\lambda$}.  For $\lambda \leq 6$, the shifts found by the Taylor approach are more impactful than those found by the IS approach.  Moreover, the drop in accuracy remains substantial (e.g., a drop of around 17\% at $\lambda=6$).  For $\lambda > 6$, the story is more subtle: The third approach (using IS to find and evaluate shifts) finds more impactful shifts, but (as noted previously) dramatically over-estimates their impact.

\section{Relationship to other approaches}%
\label{sec:related_work_app}

In this section, we give a more detailed discussion of how our work relates to other approaches for evaluation of distributional robustness and learning of robust models.  Much of the content from~\cref{sec:related_work} is duplicated here, but expanded upon to include other relevant work and detailed discussion.

\textbf{Distributionally Robust Optimization/Evaluation with divergence measures}: Distributionally robust optimization (DRO) seeks to learn models that minimize objectives of the form of~\cref{eq:maximum_loss} \citep{Duchi2018-ju, Duchi2020-at, sagawa2019distributionally}.  We focus on proactive worst-case evaluation of a fixed model, not optimization, similar to \citet{Subbaswamy2020-vr, Li2021-qe}, but major differences between our work and prior work lie in the \textbf{definition of the set of plausible future distributions $\cP$}, often called an \enquote{uncertainty set} in the optimization literature, where the goal is to specify a set that captures expected shifts, without being overly conservative.

\textit{Shifts in $\P(X, Y)$}: A conservative approach is to include all joint distributions $\P(X, Y)$ within a certain neighborhood of the training distribution. Many coherent risk measures can be written as a worst-case loss of this form. For instance, the Entropic Value-at-Risk (EVaR), with confidence level $1 - \alpha$, corresponds to the worst-case loss over a set of distributions $\cP = \{ P \ll P_0 : D_{KL}(P \| P_0) \leq - \ln \alpha \}$, where $P_0$ is the original distribution \citep{Ahmadi-Javid2012-yw}.  Similarly, the Conditional Value-at-Risk (CVaR) with parameter $\alpha$ can be seen as the worst-case loss over an uncertainty set obtained from a limiting $f$-divergence (see Example 3 of \citet{Duchi2018-ju}), including all $\alpha$-fractions of the original distribution. These measures are appealing, in that they are straightforward to compute, but can be very conservative.

Indeed, such measures often reduce to only considering the distribution of the loss itself. CVaR, for instance is equivalent to sorting the training examples by their loss, and taking the average loss of the top $\alpha$-fraction.
To illustrate these limitations, it is straightforward to see that, using the 0-1 loss and a classifier with 80\% accuracy, the worst-case loss under both of these measures is $1.0$ for any $\alpha \leq 0.2$. This is intuitive for CVaR (since over 20\% of samples are misclassified in the original distribution), and follows for EVaR from the fact that the binary distribution with probability $q = 1$ has a KL-divergence to the original distribution $p = 0.2$ of $-\ln 0.2$.

\citet{Lam2016-xx} consider a more general problem of estimating the worst-case performance of stochastic systems over infinitesimal changes in distribution, measured by Kullback-Leibler divergence. Their approach is applicable beyond machine-learning settings, and generalizes to e.g., worst-case waiting times in a queueing system. They demonstrate that for a sufficiently small neighborhood of distributions, this worst-case performance can be well-approximated by a Taylor expansion whose coefficients can be estimated from the original distribution.

\textit{Shifts in $\P(X)$ alone}: Partially due to this overly-conservative behavior, there has been a line of work incorporating additional restrictions on the allowable shift (i.e., adding more assumptions).  For instance, \citet{Duchi2020-at} considers learning predictive models that optimize a worst-case loss similar to CVaR (a \enquote{worst-case subpopulation shift}), but where only $\P(X)$ is allowed to change, and $\P(Y \mid X)$ is assumed to be constant. For similar shifts, \citet{Li2021-qe} considers only the task of evaluation, but provides a novel estimation procedure with dimension-free finite-sample guarantees.  However, many real-world shifts do not fit this framework: In~\cref{ex:lab_testing_rates}, both $\P(X)$ and $\P(Y \mid X)$ are changing, where $X = (A, O, L)$, as a result of a shift in $\P(O \mid Y, A)$.

\textit{Shifts in a conditional distribution}: Closer to our work is \citet{Subbaswamy2020-vr} who consider evaluating the loss under worst-case changes in a conditional distribution, but while we consider parametric shifts, they estimates the loss under worst-case $(1 - \alpha)$ conditional subpopulation shifts. However, it is not obvious how to choose an appropriate level of $\alpha$: in some settings, seemingly plausible values of $\alpha$ (e.g., a 20\% subpopulation) correspond to entirely implausible shifts. We give a simple lab-testing example in \cref{sec:comparison_to_subpopulation_shift_app}, where the worst-case subpopulation is one where healthy patients are always tested, and sick patients never tested. 

In contrast to these methods, our approach uses explicit parametric perturbations to define shifts, as opposed to distributional distances or subpopulations. In addition, our approach allows for shifts in multiple marginal or conditional distributions simultaneously: In~\cref{ex:lab_testing_rates}, for instance, we can model a simultaneous change in both the marginal distribution of age, as well as the conditional distribution of lab testing, while other conditionals are unchanged. Our main requirement is that each shifting distribution is exponential family, and that the shift can be represented via the natural parameters: For continuous variables this is a non-trivial restriction, but for discrete variables it is true by definition.

\textbf{Causality-motivated methods for learning robust models:} Several approaches seek to learn models that perform well under arbitrarily large causal interventions (which result in arbitrary changes in selected conditional distributions). Several approaches proactively specify shifting mechanisms/conditional distributions, and then seek to learn predictors that have good performance under arbitrarily large changes in these mechanisms \citep{subbaswamy2019preventing,Veitch2021-cu,Makar2022-eq,Puli2022-dn}.  Other approaches use auxiliary information, such as environments \citep{magliacane2018domain,Rojas-Carulla2018-qe,arjovsky2019invariant} or identity indicators \citep{Heinze-Deml2021-mz} to learn models that rely on invariant conditional distributions.  The worst-case optimality of these approaches is often restricted to cases where the shifts are arbitrarily large: In~\cref{ex:lab_testing_rates}, worst-case optimality under arbitrarily large shifts would correspond to minimizing the worst-case loss under all possible lab testing policies.

However, when the causal interventions (i.e., changes in causal mechanisms) are bounded (i.e., not arbitrary), then these approaches are not necessarily optimal. Closest to our work in motivation is prior work on robustness to bounded shift interventions in linear causal models \citep{rothenhausler2021anchor, oberst2021regularizing, kook2021distributional}. Our work can be seen as extending those ideas to general non-linear causal models, where our focus is on evaluation rather than learning robust models. We discuss this point in more detail in~\cref{sec:why_consider_bounded_shifts} below.

Our work can serve as an aid to deploying these causality-motivated methods in a few ways, by comparing their worst-case performance under bounded shifts: First, our work can inform whether such methods should be deployed at all, as for sufficiently small shifts, it may be the case that standard training yields better performance.  Second, our work can inform hyperparameter selection for several of these approaches, which include regularization terms that implicitly trade off between robustness and in-distribution performance.  More broadly, our approach is useful for probing (and comparing) the reliability of specific learned models under shift, regardless of the algorithm that produced them.

\textbf{Evaluating out-of-distribution performance with unlabelled samples}: A recent line of work has focused on predicting model performance in out-of-distribution settings, where unlabelled data is available from the target distribution \citep{Garg2022-xp,Jiang2022-pn,Chen2021-pd}.  In contrast, our method operates using only samples from the original source distribution, and seeks to estimate the worst-case loss over a set of possible target distributions. 

\subsection{The importance of considering restricted shifts in causal mechanisms}%
\label{sec:why_consider_bounded_shifts}

In~\cref{fig:labtest_loss_under_intercept_shift_app} we revisit~\cref{ex:lab_testing_rates}, adopting the perspective of a model developer, who is aware that laboratory testing policies (i.e., $P(O \mid A, Y)$) may change. 
As this change may impact the correlation between laboratory testing features $(O, L)$ and the label $Y$, how should the model developer proceed? 

From a causal perspective, one way to approach model development is to learn a predictive model that is \enquote{causal} in the sense that it only relies on the causal parents of the label $Y$. In this example, $A$ is the full set of causal parents of $Y$, and the conditional distribution $P(Y = 1 \mid A)$ does not change under changes in laboratory testing policy. This conditional distribution is an example of an \enquote{invariant} conditional distribution \citep{Rojas-Carulla2018-qe}, reflecting the unchanging causal mechanisms that generate $Y$ which are not affected by changes in laboratory testing policy.
With this in mind, we consider the choice between two models:
\begin{itemize}
  \item \textbf{Age-based model}: $f(A)\approx P(Y = 1 \mid A)$, predicting disease using age alone.\footnote{Details of how the full model $f(A, O, L)$ is trained are described in~\cref{app:details_of_labtest_example_fig1}. The model $f(A)$ is trained using unregularized logistic regression. Both models are trained on data drawn from the original distribution, where the marginal testing rate is 50\%.}
  \item \textbf{Full model}: $f(A, O, L) \approx P(Y = 1\mid A, O, L)$, predicting disease using all features.
\end{itemize}

We now demonstrate the utility of incorporating additional knowledge, considering not only \enquote{what} can change (i.e., $P(O \mid A, Y)$), but also considering \enquote{how} and \enquote{how much} it can change, and translating that knowledge into a quantitative comparison between these modelling choices.  The question of \enquote{how} corresponds to our choice of shift function, and \enquote{how much} corresponds to our choice of constraints on shift parameters. We consider changes in testing that correspond to a uniform increase/decrease in testing rates, parameterized as 
\begin{equation}
P_{\delta}(O = 1 \mid A, Y) = \operatorname{sigmoid}(\eta(A, Y) + \delta)
\end{equation}
Other details of the underlying distribution are given in~\cref{app:details_of_labtest_example_fig1}.

In~\cref{fig:labtest_loss_under_intercept_shift_app} (right), we plot the loss of each model under distributions\footnote{In this case, every choice of $\delta$ maps to a unique marginal testing rate in the distribution $P_{\delta}$ (see~\cref{prop:log_odds_marginal_to_shift}), so we plot the loss as a function of testing rate, instead of $\delta$ directly.} that correspond to different choices of $\delta$, and observe that despite having invariant performance, the age-based model only out-performs the full model under substantial changes in testing policy. In this case, the model $f(A)$ (throwing away laboratory testing information) yields better performance if testing rates drop substantially, but for a large set of changes in testing rates, the full model $f(A, O, L)$ is superior.

Considering the worst-case performance of each model can guide model selection. If a substantial change in testing rates is not plausible (which can be expressed as constraints on $\delta$), and the worst-case loss (over plausible changes) of $f(A, O, L)$ is lower than that of $f(A)$, the model developer may decide to use the full model $f(A, O, L)$ in any case.

\begin{figure}[t]
  \centering
  \begin{subfigure}[t]{0.45\textwidth}
    \centering
    \begin{tikzpicture}[
      obs/.style={circle, draw=gray!90, fill=gray!30, very thick, minimum size=5mm}, 
      int/.style={circle, draw=red!10, fill=red!30, minimum size=5mm}, 
      uobs/.style={circle, draw=gray!90, fill=gray!10, dotted, minimum size=5mm}, 
      bend angle=30]
      \node[obs] (A) {$A$} ;
      \node[obs] (Y) [below left=1cm of A]  {$Y$};
      \node[int] (O) [below right=1cm of A] {$O$} ;
      \node[obs] (L) [below right=1cm of Y] {$L$} ;
      \draw[-latex, thick] (A) -- (Y);
      \draw[-latex, thick] (A) -- (O);
      \draw[-latex, thick] (Y) -- (O);
      \draw[-latex, thick] (O) -- (L);
      \draw[-latex, thick] (Y) -- (L);
    \end{tikzpicture}
  \end{subfigure}%
  \begin{subfigure}[t]{0.40\textwidth}
  \centering
    \includegraphics[width=0.8\linewidth]{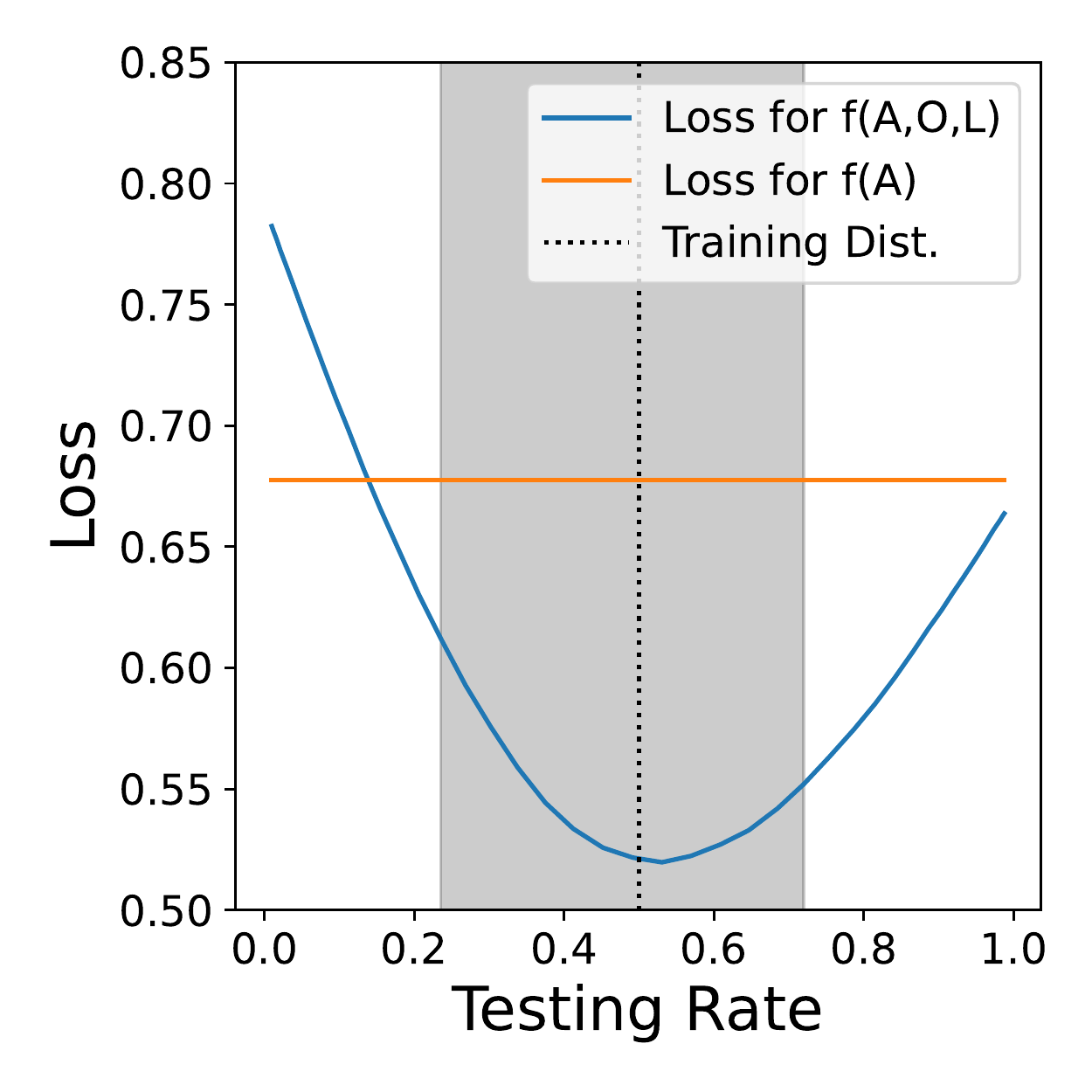}
  \end{subfigure}
  \centering
  \caption{(Left) Causal graph for~\cref{ex:lab_testing_rates}, where the variables are $Y \in \{0, 1\}$ for the label (Disease), $A \in \R$ for Age, $O \in \{0, 1\}$ for whether a laboratory test is ordered (Test Order), and $L \in \R$ for the lab result (Test Result), if available.
  (Right) Using the same generative model as in~\cref{app:details_of_labtest_example_fig1}, we contrast the performance of the full model $f(A,O,L)$ and a model $f(A)$ that only uses age, across distributions which differ in testing rates according to $P_{\delta}(O = 1 \mid A, Y) = \operatorname{sigmoid}(\eta(A, Y) + \delta)$. Comparing performance on a range of distributions where we vary $\delta$, we observe that $f(A)$ has invariant loss, but $f(A, O, L)$ has better performance for a wide range of shifts $\delta$. In particular, if we compare the worst-case loss under shifts $\abs{\delta} \leq 1.5$ (corresponding to marginal testing rates in the grey region), we can observe that the worst-case loss of $f(A,O,L)$ is lower than that of $f(A)$.}%
  \label{fig:labtest_loss_under_intercept_shift_app}
\end{figure}

\section{Proofs}\label{app:proofs}
\subsection{Proof of \texorpdfstring{\Cref{prop:ipw-weights-exp-family}}{}}
\CommonSupport*
\begin{proof}
    By \cref{def:multiple_shift,assmp:cef_factorization}, we have that
    \begin{align*}
        \P_\delta(\bV) &= \prod_{i=1}^m \P_{\delta_i}(W_i|Z_i) \prod_{V_j\in \bV\setminus \bW}\P(V_j|U_j)\\
        \P(\bV) &= \prod_{i=1}^m\P(W_i|Z_i) \prod_{V_j\in \bV\setminus \bW}\P(V_j|U_j).
    \end{align*}
    It follows that the supports of $\P_\delta$ and $\P$ are the same: Since the exponential family density is given by the base measure $g_i(W_i)$ times a exponential term (which is always strictly positive), and since the terms $\prod_{V_j\in \bV\setminus \bW}\P(V_j|U_j)$ are shared between $\P_\delta$ and $\P$, their supports agree. 
    
    To get the density ratio, we take the ratio of $\P_\delta(\bV)$ and $\P(\bV)$, and the terms $V_j\in \bV\setminus \bW$ cancel:
    \begin{align*}
        w_{\delta}(\bV) &= \frac{\P_\delta(\bV)}{\P(\bV)} \\
        &= \prod_{i=1}^m \frac{\P_{\delta_i}(W_i|Z_i)}{\P(W_i|Z_i)}.
    \end{align*}
    By \cref{def:multiple_shift,assmp:cef_factorization}, each $\P_{\delta_i}(W_i|Z_i)$ is a $\delta_i$-perturbation around the CEF distribution $\P(W_i|Z_i)$, so plugging in the exponential family densities, we get
    \begin{align*}
        w_{\delta}(\bV) &= \prod_{i=1}^m \frac{g(W_i)\exp\bigg(\{\eta_i(Z_i) + s_i(Z_i;\delta_i)\}^\top T_i(W_i) - h_i(\eta_i(Z_i) + s_i(Z_i;\delta_i))\bigg)}{g(W_i)\exp\bigg(\eta_i(Z_i)^\top T_i(W_i) - h_i(\eta_i(Z_i))\bigg)} \\
        &= \prod_{i=1}^m \exp\bigg(s_i(Z_i;\delta_i) T_i(W_i) - h_i(\eta_i(Z_i) + s_i(Z_i;\delta_i)) + h_i(\eta_i(Z_i))\bigg) \\
        &=\exp\bigg(\sum_{i=1}^ms_i(Z_i;\delta_i) T_i(W_i)\bigg)  \exp\bigg(\sum_{i=1}^m h_i(\eta_i(Z_i))-h_i(\eta_i(Z_i) + s_i(Z_i;\delta_i)) \bigg).
    \end{align*}
\end{proof}

\subsection{Proof of \texorpdfstring{\Cref{thm:sg-cond-cov-several-shift}}{}}
\SgConvCovSeveralShift*
\begin{proof}
  For simplicity throughout, we use $h^{(1)}_i$ to denote the gradient of the log-partition function $\nabla h_i(\cdot)$ with respect to the arguments, which is a column vector of length $d_{T_i}$, and we use $h^{(2)}_i$ to denote the Hessian $\nabla^2 h_i(\cdot)$, which is a matrix of size $d_{T_i} \times d_{T_i}$.  We also use $\eta_{\delta_i}(z_i)$ as short-hand for $\eta_i(z_i) + s_i(z_i; \delta_i)$.

  \textbf{Shift Gradient}: By~\cref{def:multiple_shift}, the probability density / mass function $\P_{\delta}$ factorizes as follows, where $\delta = (\delta_1, \ldots, \delta_m)$
  \begin{equation}\label{eqapp:factorization_of_pd}
        \P_\delta(\bV) = \left(\prod_{W_i \in \bW} \P_{\delta_i}(W_i|Z_i)\right)\left( \prod_{V_i \in \bV\setminus \bW} \P(V_i|\PA(V_i)) \right),
    \end{equation}
  and the gradient with respect to shift parameters $\delta_i$ is given by 
  \begin{equation*}
    \nabla_{\delta_i} p_{\delta}(v) = p_{\delta}(v) \nabla_{\delta_i} \log p_{\delta}(v) = p_{\delta}(v) \nabla_{\delta_i} \log p_{\delta_i}(w_i |  z_i)
  \end{equation*}
  where the last equality follows from additivity of the log-likelihood in the conditionals, the factorization above, and the fact that $\delta_i$ only enters into the given conditional distribution.  Given the assumed form of $\log p_{\delta_i}(w_i |  z_i)$ given in~\cref{def:delta-perturbation}, we can observe that 
  \begin{align}
    \nabla_{\delta_i} \log p_{\delta_i}(w_i |  z_i) &= \nabla_{\delta_i} \left[{(\eta_i(z_i) + s_i(z_i; \delta_i))}^\top T_i(w_i) - h_i(\eta(z_i) + s_i(z_i; \delta_i))\right]\nonumber\\
                                                      &= {(\nabla_{\delta_i} s_i(z_i; \delta_i))}^\top T_i(w_i) - {(\nabla_{\delta_i} s_i(z_i; \delta_i))}^\top \nabla h_i(\eta(z_i) + s_i(z_i; \delta_i)) \nonumber\\
                                                      &= {(\nabla_{\delta_i} s_i(z_i; \delta_i))}^\top (T_i(w_i) - h^{(1)}_i(\eta_{\delta_i}(z_i))) \label{eq:grad_log_prob_i}
  \end{align}
  where ${\nabla_{\delta_i} s_i(z_i; \delta_i)} \in \R^{d_{T_i} \times d_{\delta_i}}$, and $\nabla h_i(\eta(z_i) + s_i(z_i; \delta_i))$ is the gradient of the function $h_i: \R^{d_{T_i}} \rightarrow \R$, which is a column vector of length $d_{T_i}$.  It follows from known properties of the log-partition function \citep[Proposition~3.1]{wainwright2008graphical}, that $h^{(1)}_i(\eta_{\delta_i}(z_i)) = \E_{\delta}[T_i(W_i) |  z_i]$. This gives us that 
  \begin{align*}
    \nabla_{\delta_i} \E_{\delta}[\ell] &= \E_{\delta}\left[\ell \cdot {(\nabla_{\delta_i} s_i(Z_i; \delta_i))}^\top (T_i(W_i) - \E_{\delta}[T_i(W_i) |  Z_i])\right]\\
                                        &= \E_{\delta}\left[{(\nabla_{\delta_i} s_i(Z_i; \delta_i))}^\top \E_{\delta}[\ell \cdot (T_i(W_i) - \E_{\delta}[T_i(W_i) |  Z_i]) |  Z_i] \right]\\
                                        &= \E_{\delta}\left[{(\nabla_{\delta_i} s_i(Z_i; \delta_i))}^\top \cov_{\delta}(\ell, T_i(W_i) |  Z_i)\right],
  \end{align*}
  where the second equality follows from the tower property and $Z_i$-measurability of $\nabla_{\delta_i} s_i(Z_i; \delta_i)$, and the final equality follows from the definition of the conditional covariance. This expression, evaluated at $\delta = 0$, gives us the desired result, that 
  \begin{equation*}
    \sg_{i}^1 \coloneqq \nabla_{\delta_i} \E_{\delta}[\ell] \big|_{\delta = 0} = \E\left[D_{i, 1}^{\top} \cov(\ell, T_i(W_i) |  Z_i)\right],
  \end{equation*}
  where $D_{i, 1} = \nabla_{\delta_i} s_i(Z_i, \delta_i) |_{\delta = 0}$.  The result follows from the definition that gradients are taken entry-wise, giving $\sg^1 = (\sg^1_1, \ldots, \sg^1_m) \in \R^{d_{\delta_1} + \cdots d_{\delta_m}}$.

  \textbf{Shift Hessian (Diagonal)}:  
  For the shift Hessian, we first compute the diagonal entries of $\nabla^2_{\delta}\E_\delta[\ell]|_{\delta=0}$, which are blocks of size $\R^{d_{\delta_i} \times d_{\delta_i}}$.  We begin by computing the Hessian of the likelihood.
    \begin{align*}
    &\nabla_{\delta_i}^2 p_{\delta}(v) \\
    &= \nabla_{\delta_i} \bigg(p_{\delta}(v) \nabla_{\delta_i} \log p_{\delta_i}(w_i |  z_i)\bigg)\\
          &=p_\delta(v) \bigg((\nabla_{\delta_i} \log p_{\delta_i}(w_i|z_i))^{\otimes 2} + \nabla^2_{\delta_i} \log p_{\delta_i}(w_i|z_i) \bigg) \\
            &= p_\delta(v) \bigg(\{\nabla_{\delta_i} s_i(z_i; \delta_i)\}^\top\big(T_i(w_i) - h^{(1)}_i(\eta_{\delta_i}(z_i))\big)^{\otimes 2}\{\nabla_{\delta_i} s_i(z_i; \delta_i)\} \\
            & \qquad\qquad - \{\nabla^2_{\delta_i} s_i(z_i; \delta_i)\}^\top (T_i(w_i) - h^{(1)}_i(\eta_{\delta_i}(z_i))) \\
            &\qquad\qquad - \{\nabla_{\delta_i} s_i(z_i; \delta_i)\}^\top h^{(2)}_i(\eta_{\delta_i}(z_i))\{\nabla_{\delta_i} s_i(z_i; \delta_i)\}\bigg),\\
            &= p_\delta(v) \bigg(\{\nabla_{\delta_i} s_i(z_i; \delta_i)\}^\top\bigg(\big(T_i(w_i) - h^{(1)}_i(\eta_{\delta_i}(z_i))\big)^{\otimes 2} - h^{(2)}_i(\eta_{\delta_i}(z_i))\bigg) \{\nabla_{\delta_i} s_i(z_i; \delta_i)\} \\
            & \qquad\qquad - \{\nabla^2_{\delta_i} s_i(z_i; \delta_i)\}^\top \big(T_i(w_i) - h^{(1)}_i(\eta_{\delta_i}(z_i))\big) \bigg)
    \end{align*}
    where we use the notation $v^{\otimes 2} := vv^\top$, and we note that $\nabla_{\delta_i}^2 s(z_i; \delta_i)$ is a tensor of size $d_{T_i} \times d_{\delta_i} \times d_{\delta_i}$, and $\{\nabla^2_{\delta_i} s_i(z_i; \delta_i)\}^\top h^{(1)}_i(\cdot)$ is a matrix of size $d_{\delta_i} \times d_{\delta_i}$, where the $(m,n)$'th entry is $\{\tfrac{\partial}{\partial \delta_{im}}\tfrac{\partial}{\partial \delta_{in}} s(z_i; \delta_i)\}^\top h^{(1)}(\cdot)$. 

    Now, using the fact that $h^{(1)}(\eta_{\delta_i}(z)) = \E_{\delta}[T_i(W_i)|z_i]$ and $h^{(2)}(\eta_{\delta_i}(z_i)) = \var_{\delta}[T_i(W_i)|z_i]$ \citep[Proposition~3.1]{wainwright2008graphical}, and the definition $\epsilon_{T_i|Z_i} = T_i(W_i) - \E_{\delta}[T_i(W_i)|Z_i]$, we obtain
    \begin{align*}
      &\nabla_{\delta_i}^2 \E_{\delta}[\ell]\\
        &= \E_\delta\left[\ell \cdot \{\nabla_{\delta_i} s_i(Z_i; \delta_i)\}^\top\bigg(\epsilon_{T|Z_i}^{\otimes 2} - \var_{\delta}(T_i(W_i)|Z_i)\bigg)\{\nabla_{\delta_i} s_i(Z_i; \delta_i)\} \right] \\
        &\quad- \E_{\delta}\left[\ell \cdot \{\nabla^2_{\delta_i} s_i(Z_i; \delta_i)\}^\top \epsilon_{T_i|Z_i} \right]\\
        &=\E_\delta\left[\{\nabla_{\delta_i} s_i(Z_i; \delta_i)\}^\top \cov_{\delta}\left(\ell, \epsilon_{T_i|Z_i}^{\otimes 2}\bigg|Z_i\right)\{\nabla_{\delta_i} s_i(Z_i; \delta_i)\}\right]\\
        &\qquad\qquad- \E_\delta\left[\ell \cdot \{\nabla^2_{\delta_i} s_i(Z_i; \delta_i)\}^\top \epsilon_{T_i|Z_i}\right]
    \end{align*}
    which gives the desired result when we evaluate at $\delta = 0$.

\textbf{Shift Hessian (Off-Diagonal)} For $i \neq j$, we have that 
    \begin{align*}
    &\nabla_{\delta_i}\nabla_{\delta_j} p_{\delta}(v) \\
    &= \nabla_{\delta_i} (p_{\delta}(v) \nabla_{\delta_j} \log p_{\delta_j}(w_j |  z_j)) \\
                             &= \nabla_{\delta_i} (p_{\delta}(v) \nabla_{\delta_j} \log p_{\delta_j}(w_j |  z_j)) \\
                             &= p_{\delta}(v) \nabla_{\delta_i} \log p_{\delta_i}(w_i |  z_i) {\left(\nabla_{\delta_j} \log p_{\delta_j}(w_j |  z_j)\right)}^\top \\
                             &= p_{\delta}(v) \bigg({\{\nabla_{\delta_i} s_i(z_i; \delta_i)\}}^\top (T_i(w_i) - h^{(1)}_i(\eta_{\delta_i}(z_i)))\bigg) \\
                             &\qquad \qquad\qquad\bigg({\{\nabla_{\delta_j} s_j(z_j; \delta_j)\}}^\top (T_j(w_j) - h^{(1)}_j(\eta_{\delta_j}(z_j)))\bigg)^\top
    \end{align*}
    where the third line follows from the fact that $\nabla_{\delta_i} (\nabla_{\delta_j} \log p_{\delta_j}(w_j |  z_j)) = 0$, and the last line follows from the derivation of the gradient of the log-likelihood in~\cref{eq:grad_log_prob_i}.
    We can again use the fact that $h^{(1)}_i(\eta_{\delta_i}(Z_i)) = \E_{\delta}[T_i(W_i)|Z_i]$ and the shorthand $\epsilon_{T_i|  Z_i} \coloneqq T_i(W_i) - \E_{\delta}[T_i(W_i) |  Z_i]$ to write that 
    \begin{align*}
        &\nabla_{\delta_i}\nabla_{\delta_j} \E_{\delta}[\ell]\\
        &= \E_{\delta}\bigg[ \ell \cdot {\{\nabla_{\delta_i} s_i(z_i; \delta_i)\}}^\top \bigg((T_i(w_i) - h^{(1)}_i(\eta_{\delta_i}(z_i)))\bigg)\\
        &\qquad\qquad\qquad\qquad\qquad\bigg((T_j(w_j) - h^{(1)}_j(\eta_{\delta_j}(z_j)))\bigg)^\top{\{\nabla_{\delta_j} s_j(z_j; \delta_j)\}} \bigg]
    \end{align*}
    and when we evaluate this expression at $\delta = 0$, we obtain 
    \begin{align*}
        \nabla_{\delta_i}\nabla_{\delta_j} \E_{\delta}[\ell] \big|_{\delta=0} &= \E\left[ \ell \cdot D_{i, 1}^\top \epsilon_{T_i |  Z_i} (\epsilon_{T_j |  Z_j})^\top D_{j, 1} \right] = \cov(\ell, D_{i, 1}^\top\epsilon_{T_i|Z_i}\epsilon_{T_j|Z_j}^\top D_{j, 1}).
    \end{align*}
    Where the last equality follows because $\E[D_{i, 1}^\top \epsilon_{T_i|Z_i}\epsilon_{T_j|Z_j}^\top D_{j, i}] = 0$.  To see this, note that one of $W_i, W_j$ must be a non-descendant of the other, and we will assume without loss of generality that $W_j$ is a non-descendant of $W_i$ in the causal graph consistent with the factorization given in~\cref{eqapp:factorization_of_pd}, which implies that $Z_j$ (the parents of $W_j$ in the underlying graph) are also non-descendants of $W_i$.Thus, $W_i \indep (W_j, Z_j) |  Z_i$, because $(W_j, Z_j)$ are both non-descendants of $W_i$.  Then, observe that $D_{i, 1}$ is a function of $Z_i$, and $\epsilon_{T_i |  Z_i}$ is a variable with zero-mean conditioned on $Z_i$.  Thus, $\E[D_{i, 1}^\top \epsilon_{T_i |  Z_i} |  Z_i] = 0$, for all $Z_i$.  Moreover, given $Z_i$, we have that $D_{i, 1}^\top \epsilon_{T_i |  Z_i}$ is independent of $D_{j, 1}^\top \epsilon_{T_j |  Z_j}$.  As a result, we can write that 
    \begin{align*}
      \E[D_{i, 1}^\top \epsilon_{T_i|Z_i}\epsilon_{T_j|Z_j}^\top D_{j, 1}] 
      &= \E[\E[D_{i, 1}^\top \epsilon_{T_i|Z_i}\epsilon_{T_j|Z_j}^\top D_{j, 1} |  Z_i]]  \\
      &= \E[\E[D_{i, 1}^\top \epsilon_{T_i|Z_i} |  Z_i] \E[\epsilon_{T_j|Z_j}^\top D_{j, 1} |  Z_i]]  \\
      &= \E[ 0 \cdot \E[\epsilon_{T_j|Z_j}^\top D_{j, 1} |  Z_i]]  \\
        &= 0
    \end{align*}
\end{proof}

\subsection{Proof of \texorpdfstring{\Cref{thm:sg-cond-cov}}{}}
\SgCondCov*
\begin{proof}
    We have $\nabla_\delta s(Z;\delta) = \nabla_\delta \delta = 1$ and $\nabla_\delta^2 s(Z;\delta) = \nabla_\delta^2 \delta = 0$. The result now follows from \Cref{thm:sg-cond-cov-several-shift}.
\end{proof}

\subsection{Proof of \texorpdfstring{\Cref{thm:taylor-approximation-bound}}{}}
\TaylorApproxBound*
\begin{proof}
    The expectation is continuous and twice-differentiable with respect to $\delta$, because of the smoothness of the exponential family in the parameter, the fact that the shift function $s$ is twice-differentiable, and because the support does not change. Thus, applying Taylors remainder theorem to the function $t \mapsto \E_{t\cdot\delta}[\ell]$, it follows that there exist a $t_0 \in [0, 1]$ such that
    \begin{equation}\label{eq:proof-taylor-thm-1}
        \E_{1\cdot\delta}[\ell] - \E_{0\cdot\delta}[\ell] - \bigg(\tfrac{\mathrm{d}}{\mathrm{d}t} \E_{t\cdot\delta}[\ell]\bigg)\bigg|_{t=0} = \bigg(\tfrac{1}{2}\tfrac{\mathrm{d}^2}{\mathrm{d}^2t} \E_{t\cdot\delta}[\ell]\bigg)\bigg|_{t=t_0}.
    \end{equation}
    We have $\bigg(\tfrac{\mathrm{d}}{\mathrm{d}t} \E_{t\cdot\delta}[\ell]\bigg)\bigg|_{t=0} = \sg^1$ and by the same arguments (see the proof of \cref{thm:sg-cond-cov-several-shift}), it follows that $\bigg(\tfrac{1}{2}\tfrac{\mathrm{d}^2}{\mathrm{d}^2t} \E_{t\cdot\delta}[\ell]\bigg)\bigg|_{t=t_0} = \delta^\top\cov_{t_0\cdot\delta}(\ell, \epsilon_{t_0\cdot\delta, T|Z}^{\otimes 2})\delta$. Plugging this in, and subtracting $\tfrac{1}{2}\delta^\top\sg^2\delta$ on both sides of \cref{eq:proof-taylor-thm-1} yields
    \begin{align*}
      \bigg|\E_{\delta}[\ell] - E_{\delta, \text{Taylor}} \bigg|
        &= \tfrac{1}{2}\bigg|\delta^\top\bigg(\cov_{t_0\cdot\delta}(\ell, \epsilon_{t_0\cdot\delta, T|Z}^{\otimes 2}) - \cov(\ell, \epsilon_{0, T|Z}^{\otimes 2})\bigg)\delta\bigg| \\
        &\leq \tfrac{1}{2}\sup_{t\in[0,1]}\bigg|\delta^\top\bigg(\cov_{t\cdot\delta}(\ell, \epsilon_{t\cdot\delta, T|Z}^{\otimes 2}) - \cov(\ell, \epsilon_{0, T|Z}^{\otimes 2})\bigg)\delta\bigg|.
    \end{align*}
    Let $K := \bigg(\cov_{t\cdot\delta}(\ell, \epsilon_{t\cdot\delta, T|Z}^{\otimes 2}) - \cov(\ell, \epsilon_{0, T|Z}^{\otimes 2})\bigg)$. Since $K$ is symmetric and real valued, it is diagonalizeable, $K = U^\top \Lambda U$ for an orthonormal matrix $U$ and diagonal matrix $\Lambda = \operatorname{diag}(\alpha_1, \ldots, \alpha_d)$. 
    We then have
    \begin{align*}
        |\delta^\top K \delta| &= |\delta^\top U^\top \Lambda U \delta| \\
        &= |(\Lambda^{1/2}U\delta)^\top (\Lambda^{1/2}U\delta)| \\
        &= \|\Lambda^{1/2}U\delta\|_2^2 \\
        &\leq \|\Lambda^{1/2}\|_2^2 \|U\delta\|_2^2 \\
        &= \sigma(K) \|\delta\|_2^2,
    \end{align*}
    where $\Lambda^{1/2} = \operatorname{diag}(\sqrt{\alpha_1}, \ldots, \sqrt{\alpha_d})$, $\|\cdot\|_2$ denotes the supremum-norm when applied to matrices and the $2$-norm when applied to vectors and $\|U\delta\|_2 = \|\delta\|_2$ because $\|U\delta\|_2^2 = \delta^\top U^\top U \delta = \delta^\top \delta = \|\delta\|_2^2$, using orthonormality of $U$. Plugging in this inequality, we get that
    \begin{equation*}
      \bigg|\E_{\delta}[\ell] - E_{\delta, \text{Taylor}} \bigg| \leq \tfrac{1}{2}\sup_{t\in[0,1]} \sigma\bigg(\cov_{t\cdot\delta}(\ell, \epsilon_{t\cdot\delta, T|Z}^{\otimes 2}) - \cov(\ell, \epsilon_{0, T|Z}^{\otimes 2})\bigg) \|\delta\|_2^2,
    \end{equation*}
    which concludes the proof.
\end{proof}

\subsection{Proof of \texorpdfstring{\Cref{prop:log_odds_marginal_to_shift}}{}}
\LogOddsMarginaltoShift*
\begin{proof}
  Let $F$ denote the event that $\eta(Z)$ is finite (i.e., $\eta(Z) \not\in \{- \infty, + \infty\}$). Under $F$, the conditional probability function $\sigma(\eta(Z) + \delta)$ is a strictly monotonically increasing function of $\delta$, and if $\eta(Z) \in \{- \infty, + \infty\}$, then the conditional probability is a constant function of $\delta$ (zero or one, respectively). Hence, we can write that 
  \begin{equation*}
    \P_{\delta}(W = 1) = \P_{\delta}(W = 1 |  F) (1 - p_{+} - p_{-}) + p_{+}
  \end{equation*}
  and by assumption, $1 - p_{+} - p_{-} > 0$. The marginal probability $\P_{\delta}(W = 1 |  F)$ is a strictly monotonically increasing function of $\delta$, with a limit of $1$ as $\delta \rightarrow \infty$, and a limit of $0$ as $\delta \rightarrow - \infty$.  As a result, it is bounded in $(p_{+}, 1 - p_{-})$.
\end{proof}

\subsection{Proof of \texorpdfstring{\Cref{lemma:ar-quadratic-loss}}{}}
\AnchorRegQuadLoss*
\begin{proof}
    It follows from \cref{eq:linear-scm} that one can write $(X^\top, Y^\top, H^\top) = (1-B)^{-1}(MA + \epsilon)$, and for a given $\gamma$, there exist $b_\gamma, \kappa_\gamma$ such that $Y - \gamma^\top X = b_\gamma^\top A + \kappa_\gamma^\top \epsilon$ \citep{rothenhausler2021anchor}.
    In $\P_{\delta}$, we can write $A = \mu + \delta + \epsilon_A$, where $\epsilon_A \sim \mathcal{N}(0, \Sigma)$, for all values of $\mu$ and $\delta$. Plugging this in yields
    \begin{align*}
        \E_{\delta}[(Y-\gamma^\top X)^2] &= \E_{\delta}[(b_\gamma^\top A + \kappa_\gamma^\top \epsilon)^2] \\
        &= \E_{\delta}[(b_\gamma^\top (\mu + \delta + \epsilon_A) + \kappa_\gamma^\top \epsilon)^2] \\
        &=\E[(b_\gamma^\top (\mu + \epsilon_A) + \kappa^\top \epsilon)^2] + (2 b_\gamma^\top\mu) \delta^\top b_\gamma + \delta^\top b_\gamma b_\gamma^\top \delta \\
        &=\E[(Y - \gamma^\top X)^2] + (2 b_\gamma^\top\mu) \delta^\top b_\gamma + \delta^\top b_\gamma b_\gamma^\top \delta.
    \end{align*}
    where we do not put a subscript on the expectation in the third line because it is taking expectations over $\epsilon_A$ and $\epsilon$, both which do not depend on the choice of $\mu$ and $\delta$. The statement of the lemma follows by letting $u_{\mu,\gamma} = 2b_\gamma^\top \mu$ and $v_\gamma = \sqrt{2} b_\gamma$.
\end{proof}

\subsection{Proof of \texorpdfstring{\Cref{prop:reconciliation-ar}}{}}
\ReconAR*
\begin{proof}
    Similar to \cref{lemma:ar-quadratic-loss}, we rewrite $Y-\gamma^\top X = b_\gamma^\top A + \kappa^\top \epsilon$, and by rewriting $A = \mu + \delta + \epsilon_A$, where $\epsilon_A \sim \mathcal{N}(0, \Sigma)$, we obtain
    \begin{align}
        \E_{\delta}[(Y-\gamma^\top X)^2] &= \E(b_\gamma^\top (\mu + \epsilon_A) + \kappa^\top \epsilon)^2 \label{eq:ar-recon-1} \\
        &+ (2 b_\gamma^\top\mu) \delta^\top b \label{eq:ar-recon-2}\\
        &+ \delta^\top b b^\top \delta. \label{eq:ar-recon-3}
    \end{align}
    We recognize that \cref{eq:ar-recon-1} equals $\E(Y-\gamma^\top X)^2$. Similarly, we now show that \cref{eq:ar-recon-2,eq:ar-recon-3} match the shift gradients (multiplied appropriately with $\delta$).
    
    First, we assume that $\Sigma=\Id$.
    Since $A$ is a Gaussian with (known) mean $\Id$, the sufficient statistic is $T(A) = A$. Hence, according to \cref{thm:sg-cond-cov-several-shift}, we can compute the shift gradient as
    \begin{equation*}
        \sg^1 = \cov(A, \ell) = \cov(A, (Y-\gamma^\top X)^2) = \cov(A, (b_\gamma^\top A)^2).
    \end{equation*}
    We can calculate the $i$'th entrance of this vector as:
    \begin{align*}
        \sg^1 = \cov(A_i, (b_\gamma^\top A)^2) 
        &= \cov(A_i - \mu_i, (b_\gamma^\top A)^2)) \\
        &= \cov(A_i-\mu_i, b_{\gamma, i}^2 A_i^2 + 2 \sum_{j\neq i} b_i b_j A_i A_j)\\
        &= b_{\gamma, i}^2\cov(A_i-\mu_i, A_i^2) + 2b_{\gamma, i}\sum_{j\neq i} b_j \cov(A_i - \mu_i, A_i A_j),
    \end{align*}
    where in the first equality we use that subtracting a constant doesn't change the covariance, and we use independence of $A_i$ from $A_{j}A_{j'}$ when $i \notin \{j, j'\}$.
    Using the assumption that $A_i$ has unit variance, we now get that
    \begin{align*}
        \cov(A_i - \mu_i, A_i^2) &= \E[A_i^3 - \mu_i A_i^2] = (\mu_i^3 + 3\mu_i) - \mu_i (\mu_i^2 + 1) = 2\mu_i \\
        \cov(A_i - \mu_i, A_i A_j)&=\E[A_i^2 - A_i\mu_i]\E[A_j] = (\mu_i^2 + 1 - \mu_i^2) \mu_j = \mu_j.
    \end{align*}
    By plugging in, we obtain
    \begin{align*}
        \sg^1(\mu_i) &= 2b_{\gamma, i}^2 \mu_i + 2 b_{\gamma,i}\sum_{j\neq i} b_j \mu_j \\
        &=2b_{\gamma,i} b_{\gamma}^\top \mu.
    \end{align*}
    Since this was element-wise, we obtain that the full vector is $\sg^1 = 2 b_\gamma b_\gamma^\top\mu$, which, when multiplied with $\delta$ yields \cref{eq:ar-recon-2}.
    
    We compute $\sg^2$ similarly. The diagonal entries are given by
    \begin{align*}
        \sg^2_{i,i} 
        &= \cov((A_i-\mu_i)^2, (b_\gamma^\top A)^2) \\
        &= \cov((A_i-\mu_i)^2, b_{\gamma, i}^2 A_i^2 + b_{\gamma, i}\sum_{j\neq i} b_{\gamma, j} A_i A_j) \\
        &= b_{\gamma, i}^2 \cov((A_i-\mu_i)^2, A_i^2) + b_{\gamma, i}\sum_{j\neq i} b_{\gamma, j} \cov((A_i-\mu_i)^2, A_i A_j).
    \end{align*}
    Because $\Sigma=\Id$, the second through fourth moments of $A_i$ are given by $\E[A_i^2] = \mu_i^2 + 1$, $\E[A_i^3] =\mu_i^3 + 3\mu_i$ and $\E[A_i^4] = \mu_i^4 + 6\mu_i^2 + 3$. Using this, we get
    \begin{align*}
        \cov((A_i-\mu_i)^2, A_i^2) 
        &= \E[A_i^4 - 2\mu_i A_i^3 + \mu_i^2 A_i^2] - \E[(A_i-\mu_i)^2]\E[A_i^2] \\
        &= (\mu_i^4 + 6\mu_i^2 + 3) - 2\mu_i (\mu_i^3 + 3\mu_i) + \mu_i^2(\mu_i^2 + 1) - 1\cdot(\mu_i^2 + 1)\\
        &= 2,
    \end{align*}
    and for $j\neq i$:
    \begin{align*}
        \cov((A_i-\mu_i)^2, A_i A_j) 
        &= \cov((A_i-\mu_i)^2, (A_i - \mu_i) A_j) + \cov((A_i-\mu_i)^2, \mu_i A_j) \\
        &= \cov((A_i-\mu_i)^2, (A_i - \mu_i) A_j) \\
        &= \E[(A_i-\mu_i)^3]\E[A_j] - \E[(A_i-\mu_i)^2]\E[(A_i-\mu_i)] \E[A_j] \\
        &= 0 - 0,
    \end{align*}
    using linearity of the covariance, that $A_i \indep A_j$ and that the first and third moments are zero for a centered Gaussian $A_i - \mu_i$. Plugging this in, we get that the diagonal entries are given by
    \begin{align*}
        \sg^2_{i,i} = 2b_{\gamma,i}^2.
    \end{align*}
    
    We can compute the off-diagonal entries similarly. For $i \neq j$, we have:
    \begin{align}\label{eq:ar-recon-sg-2-off-diag}
        \sg^2_{i,j} &= \cov\bigg((A_i - \mu_i)(A_j - \mu_j), \\
        &\qquad\qquad b_{\gamma, i}^2 A_i^2 + b_{\gamma, j}^2 A_j^2 + 2 b_{\gamma, i}b_{\gamma, j} A_i A_j + 2\sum_{v \notin\{i,j\}} b_{\gamma, i}b_{\gamma,v} A_i A_v + b_{\gamma, j}b_{\gamma,v} A_j A_v\bigg).\nonumber
    \end{align}
    Using the independence of $A_i$ and $A_j$, we have
    \begin{align*}
        &\cov((A_i - \mu_i)(A_j - \mu_j), A_i^2) \\
        &= \E[A_i^2(A_i-\mu_i)]\underbrace{\E[A_j - \mu_j]}_{=0} - \underbrace{\E[A_i - \mu_i]}_{=0}\E[A_j - \mu_j]\E[A_i^2]\\
        &= 0,
    \end{align*}
    and similarly $\cov((A_i - \mu_i)(A_j - \mu_j), A_j^2)=0$. Using the same reasoning, for $v \notin \{i,j\}$
    \begin{align*}
        &\cov((A_i - \mu_i)(A_j-\mu_j), A_i A_v) \\
        &= \E[(A_i-\mu_i)A_i]\E[A_j-\mu_j] \E[A_v] - \E[(A_i-\mu_i)]\E[A_i]\E[A_j-\mu_j] \E[A_v]\\
        &= 0,
    \end{align*}
    and the same for $\cov((A_i - \mu_i)(A_j-\mu_j), A_j A_v)$. Finally, we have
    \begin{align*}
        &\cov((A_i - \mu_i)(A_j-\mu_j), A_i A_j) \\
        &= \E[(A_i - \mu_i)A_i]\E[(A_j - \mu_j)A_j] - \E[(A_i - \mu_i)]\E[A_i]\E[(A_j - \mu_j)]\E[A_j] \\
        &= \E[(A_i - \mu_i)A_i]\E[(A_j - \mu_j)A_j] \\
        &= \E[A_i^2 - \mu_iA_i]\E[A_j^2 - \mu_j A_j] \\
        &= [(\mu_i^2 + 1) - \mu_i^2][[(\mu_j^2 + 1) - \mu_j^2]] \\
        &= 1.
    \end{align*}
    Plugging into \cref{eq:ar-recon-sg-2-off-diag}, we get that
    \begin{equation*}
        \sg^2_{i,j} = 2 b_{\gamma,i}b_{\gamma,j},
    \end{equation*}
    and hence for both diagonal and off-diagonal entries, $\sg^2_{i,j} = 2 b_{\gamma,i}b_{\gamma,j}$, implying that
    \begin{equation*}
        \sg^2 = 2 b_{\gamma}b_{\gamma}^\top.
    \end{equation*}
    In particular $\tfrac{1}{2}\delta^\top\sg^2\delta$ matches \cref{eq:ar-recon-3}.
    
    Finally, we consider the case $\Sigma\neq \Id$. Let $\Sigma^{-1/2}$ be the `square-root' of $\Sigma^{-1}$, such that $\Sigma^{-1/2}\Sigma^{-\top/2}$ (where the latter denotes $(\Sigma^{-1/2})^\top$.\footnote{Formally, if $\Sigma^{-1} = U\Lambda U^\top$ where $\Lambda = \operatorname{diag}(\lambda_1, \ldots, \lambda_{d_A})$, define $\Sigma^{-1/2} := U\operatorname{diag}(\sqrt{\lambda_1},\ldots, \sqrt{\lambda_{d_A}})$.}
    
    The sufficient statistics for the mean in a multivariate Gaussian distribution with known variance is given by $T(A) = \Sigma^{-1}A$.
    We then have
    \begin{align*}
        \sg^1 
        &= \cov(\Sigma^{-1}A, (b_\gamma^\top A)^2) \\
        &= \Sigma^{-1/2}\cov(\Sigma^{-1/2}A, ((\Sigma^{1/2} b_\gamma)^\top \Sigma^{-1/2} A)^2) \\
        &= \Sigma^{-1/2}\cov_{\tilde{\mu}}(\tilde{A}, (\tilde{b}_\gamma^\top \tilde{A})^2),
    \end{align*}
    where $\tilde{A} = \Sigma^{-1/2}A = \sim\mathcal{N}(\tilde{\mu}, \Id)$, $\tilde{\mu}=\Sigma^{-1/2}\mu$ and $\tilde{b}_\gamma = \Sigma^{1/2}b_\gamma$. In particular, since $\tilde{A}$ has unit variance, we can use the above derivations to obtain
    \begin{align*}
        \sg^1 = 2\Sigma^{-1/2}(\tilde{b}_\gamma \tilde{b}_\gamma^\top \tilde{\mu}) = 2 b_\gamma b_\gamma^\top \mu.
    \end{align*}
    In particular, the first shift gradient is the when $\Sigma\neq\Id$ as when $\Sigma=\Id$.
    Similarly, 
    \begin{align*}
        \sg^2 
        &= \cov(\Sigma^{-1}(A-\mu)(A-\mu)^\top \Sigma^{-\top}, (b_\gamma^\top A)^2) \\
        &= \cov(\Sigma^{-1/2}\Sigma^{-1/2}(A-\mu)(A-\mu)^\top \Sigma^{-\top/2}\Sigma^{-\top/2}, (\Sigma^{1/2}b_\gamma)^\top \Sigma^{-1/2}A)^2) \\
        &= \Sigma^{-1/2}\cov_{\tilde{\mu}}((\tilde{A}-\tilde{\mu})(\tilde{A}-\tilde{\mu})^\top, (\tilde{b}_\gamma^\top \tilde{A})^2)\Sigma^{-\top/2} \\
        &= \Sigma^{-1/2} 2 \tilde{b}_\gamma\tilde{b}_\gamma^\top \Sigma^{-\top/2} \\
        &= 2 b_\gamma b_\gamma^\top.
    \end{align*}
    Hence, also when $\Sigma\neq\Id$, the terms of \cref{eq:ar-recon-2,eq:ar-recon-3} matches the expression given by $\sg^1$ and $\sg^2$. This concludes the proof.
\end{proof}

\end{document}

%% file: packages.tex
\usepackage[utf8]{inputenc} 
\usepackage[T1]{fontenc}    
\usepackage{hyperref}       
\usepackage{url}            
\usepackage{booktabs}       
\usepackage{amsfonts}       
\usepackage{nicefrac}       
\usepackage{microtype}      
\usepackage{wrapfig} 
\usepackage[inline]{enumitem}

\usepackage[usenames,dvipsnames]{xcolor}
\usepackage{csquotes}

\usepackage{amsmath,amssymb,amsthm,thmtools}
\usepackage{tikz}
\usepackage{tikz-qtree}
\usetikzlibrary{bayesnet}

\usepackage[capitalize,noabbrev]{cleveref}
\usepackage{subcaption}
\usepackage{mathtools}
\usepackage{apptools}
\usepackage{algorithm}
\usepackage{algorithmic}

\usepackage{capt-of}

\usetikzlibrary{3d}
\usetikzlibrary{shadings}
\makeatletter
    \tikzoption{canvas is xy plane at z}[]%
    {   \def\tikz@plane@origin{\pgfpointxyz{0}{0}{#1}}%
        \def\tikz@plane@x{\pgfpointxyz{1}{0}{#1}}%
        \def\tikz@plane@y{\pgfpointxyz{0}{1}{#1}}%
        \tikz@canvas@is@plane
    }
\makeatother 
\usepackage{pgfplots}
\pgfplotsset{compat=1.17}

\usepackage{bm}
\usepackage{enumitem} 

%% file: commands.tex
\DeclareMathOperator{\sg}{SG}
\DeclareMathOperator{\Id}{Id}
\DeclareMathOperator{\PA}{PA}
\DeclareMathOperator{\vectorize}{vec}
\DeclareMathOperator{\supp}{supp}
\newcommand{\taylor}{\hat{E}_{\delta, \text{Taylor}}}
\newcommand{\taylorPop}{E_{\delta, \text{Taylor}}}
\newcommand{\ipw}{\hat{E}_{\delta,\text{IS}}}

\definecolor{myred}{HTML}{FDA544}
\definecolor{mygreen}{HTML}{0f8c2e}

\theoremstyle{definition}
\newtheorem{definition}{Definition}
\newtheorem{assumption}{Assumption}
\newtheorem{example}{Example}
\theoremstyle{remark}
\newtheorem{remark}{Remark}
\renewcommand\thmcontinues[1]{Continued}
\AtAppendix{\counterwithin{theorem}{section}}
\AtAppendix{\counterwithin{corollary}{section}}
\AtAppendix{\counterwithin{proposition}{section}}
\AtAppendix{\counterwithin{lemma}{section}}
\AtAppendix{\counterwithin{conjecture}{section}}
\AtAppendix{\counterwithin{definition}{section}}
\AtAppendix{\counterwithin{assumption}{section}}
\AtAppendix{\counterwithin{example}{section}}

\newcommand{\eid}{\stackrel{\text{(d)}}{=}}

\DeclareMathOperator{\cov}{cov}
\DeclareMathOperator{\var}{var}

\def\R{\mathbb{R}}

\newcommand{\cG}{\mathcal{G}}

\newcommand{\cN}{\mathcal{N}}
\newcommand{\cO}{\mathcal{O}}
\newcommand{\cP}{\mathcal{P}}

\newcommand{\cY}{\mathcal{Y}}
\newcommand{\cZ}{\mathcal{Z}}

\newcommand{\bV}{\mathbf{V}}
\newcommand{\bW}{\mathbf{W}}

\newcommand{\bZ}{\mathbf{Z}}

\newcommand{\Ber}{\mathsf{Ber}}

\def\E{\mathbb{E}}

\def\P{\mathbb{P}}

\newcommand\indep{\protect\mathpalette{\protect\independenT}{\perp}}
\def\independenT#1#2{\mathrel{\rlap{$#1#2$}\mkern2mu{#1#2}}}

\newcommand{\1}[1]{\mathbf{1}\left\{#1\right\}}

\newcommand{\norm}[1]{\left\|#1\right\|}

\newcommand{\abs}[1]{\left|#1\right|}

\makeatletter
\newcommand{\mytag}[2]{%
  \text{#1}%
  \@bsphack
  \begingroup
    \@onelevel@sanitize\@currentlabelname
    \edef\@currentlabelname{%
      \expandafter\strip@period\@currentlabelname\relax.\relax\@@@%
    }%
    \protected@write\@auxout{}{%
      \string\newlabel{#2}{%
        {#1}%
        {\thepage}%
        {\@currentlabelname}%
        {\@currentHref}{}%
      }%
    }%
  \endgroup
  \@esphack
}
\makeatother

%% file: figures/celeba-graph.tex
    \begin{tikzpicture}
        \node (Young) at (-2,3) {Age};
        \node (Male) at (2,3) {Gender};
        \node (Eyeglasses) at (-4,1.5) {Eyeglasses};
        \node (Bald) at (-2,1.5) {Bald};
        \node (Mustache) at (0,1.5) {Mustache};
        \node (Smiling) at (2,1.5) {Smiling};
        \node (Lipstick) at (4.2,1.5) {Wearing Lipstick};
        \node (Mouth) at (-2,0) {Mouth Slightly Open};
        \node (Eyes) at (2,0) {Narrow Eyes};
        \draw[->] (Young) to (Eyeglasses);
        \draw[->] (Young) to (Bald);
        \draw[->] (Young) to (Mustache);
        \draw[->] (Young) to (Smiling);
        \draw[->] (Young) to (Lipstick);
        \draw[->] (Young) edge[bend right=35] (Mouth);
        \draw[->] (Young) edge[bend right=20] (Eyes);
        \draw[->] (Male) to (Bald);
        \draw[->] (Male) to (Mustache);
        \draw[->] (Male) to (Smiling);
        \draw[->] (Male) to (Lipstick);
        \draw[->] (Male) edge[bend left=45] (Eyes);
        \draw[->] (Smiling) to (Eyes);
        \draw[->] (Smiling) to (Mouth);
        \node[opacity=1.0] at (-2,3.3) {\includegraphics[width=0.04\linewidth]{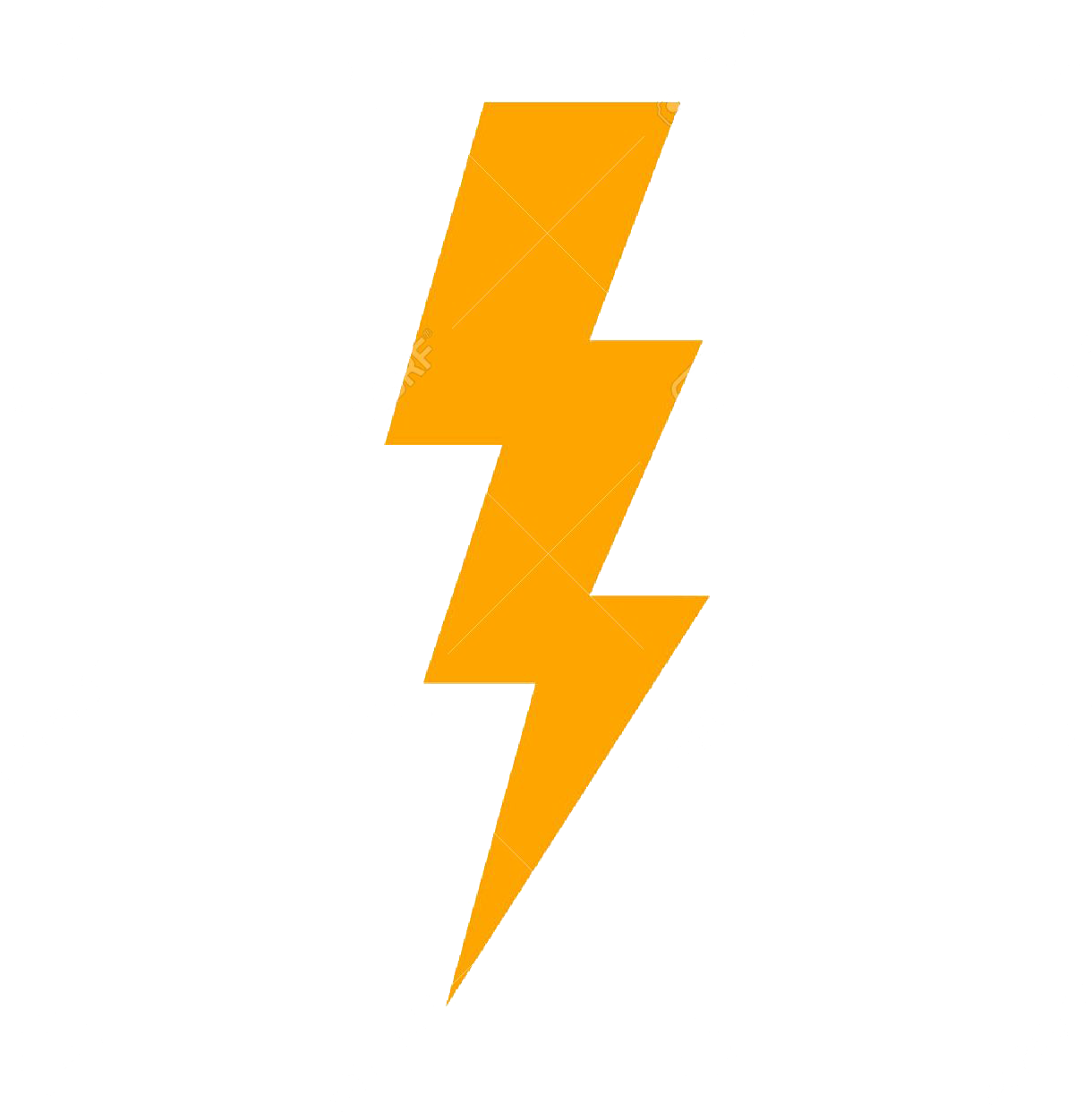}};
        \node[opacity=1.0] at (-4,1.9) {\includegraphics[width=0.04\linewidth]{figures/lightning.png}};
        \node[opacity=1.0] at (-2.2,1.9) {\includegraphics[width=0.04\linewidth]{figures/lightning.png}};
        \node[opacity=1.0] at (0,1.9) {\includegraphics[width=0.04\linewidth]{figures/lightning.png}};
        \node[opacity=1.0] at (2.2,1.9) {\includegraphics[width=0.04\linewidth]{figures/lightning.png}};
        \node[opacity=1.0] at (4,1.9) {\includegraphics[width=0.04\linewidth]{figures/lightning.png}};
        \node[opacity=1.0] at (-2,0.4) {\includegraphics[width=0.04\linewidth]{figures/lightning.png}};
        \node[opacity=1.0] at (1.7,0.4) {\includegraphics[width=0.04\linewidth]{figures/lightning.png}};
    \end{tikzpicture}

%% file: tables/table1_left.tex
\begin{tabular}{rlrrr}
\toprule
Conditional & &  $\delta_i$ &  $\P$ &  $\P_{\delta}$ \\
\midrule
              Bald &| Female, Old &       0.899 & 0.047 &          0.109 \\
              Bald &| Male, Young &      -0.800 & 0.378 &          0.214 \\
                Bald &| Male, Old &      -0.680 & 0.622 &          0.455 \\
Wearing Lipstick &| Female, Young &      -0.618 & 0.924 &          0.868 \\
  Wearing Lipstick &| Female, Old &      -0.543 & 0.953 &          0.921 \\
\bottomrule
\end{tabular}

%% file: tables/table1_right.tex
\begin{tabular}{lrr}
\toprule
Metric &  Example $\delta$ &   Avg. \\
\midrule
Original acc. ($\E[\mathbf{1}\{f(X) = Y\}]$)                           &          &  0.912 \\
\midrule
Acc.\ under Taylor shift ($\E_{\delta_{\text{Taylor}}}[\mathbf{1}\{f(X) = Y\}]$) &    0.874 &  0.874 \\
IS est.\ of acc.\ under Taylor shift ($\hat{E}_{\delta_{\text{Taylor}}, \text{IS}}$)                  &    0.829 &  0.863 \\
Taylor est.\ of acc.\ under Taylor shift ($\hat{E}_{\delta_{\text{Taylor}}, \text{Taylor}}$)          &    0.844 &  0.863 \\
\midrule
Acc.\ under IS shift ($\E_{\delta_{\text{IS}}}[\mathbf{1}\{f(X) = Y\}]$)     &    
                                                                &  0.889 \\
IS est.\ of acc.\ under IS shift ($\hat{E}_{\delta_{\text{IS}}, \text{IS}}$)                      &    
                                                                                      &  0.821 \\
\bottomrule
\end{tabular}

%% file: figures/figure5_left.tex
\begin{tikzpicture}[x=1pt,y=1pt]
\definecolor{fillColor}{RGB}{255,255,255}
\begin{scope}
\definecolor{fillColor}{RGB}{15,140,46}

\path[fill=fillColor] ( 14.62, 32.98) rectangle ( 23.72, 32.98);

\path[fill=fillColor] ( 23.72, 32.98) rectangle ( 32.82, 32.98);

\path[fill=fillColor] ( 32.82, 32.98) rectangle ( 41.92, 34.15);

\path[fill=fillColor] ( 41.92, 32.98) rectangle ( 51.03, 35.89);

\path[fill=fillColor] ( 51.03, 32.98) rectangle ( 60.13, 39.38);

\path[fill=fillColor] ( 60.13, 32.98) rectangle ( 69.23, 48.10);

\path[fill=fillColor] ( 69.23, 32.98) rectangle ( 78.33, 68.46);

\path[fill=fillColor] ( 78.33, 32.98) rectangle ( 87.43, 76.60);

\path[fill=fillColor] ( 87.43, 32.98) rectangle ( 96.53, 78.93);

\path[fill=fillColor] ( 96.53, 32.98) rectangle (105.63, 78.35);

\path[fill=fillColor] (105.63, 32.98) rectangle (114.73, 49.27);

\path[fill=fillColor] (114.73, 32.98) rectangle (123.83, 47.52);

\path[fill=fillColor] (123.83, 32.98) rectangle (132.94, 36.47);

\path[fill=fillColor] (132.94, 32.98) rectangle (142.04, 35.31);
\definecolor{drawColor}{RGB}{0,0,0}

\path[draw=drawColor,line width= 0.6pt,dash pattern=on 6pt off 6pt ,line join=round] ( 79.54, 30.69) -- ( 79.54, 81.22);
\definecolor{drawColor}{RGB}{255,0,0}

\path[draw=drawColor,line width= 0.6pt,dash pattern=on 6pt off 6pt ,line join=round] ( 19.10, 30.69) -- ( 19.10, 81.22);
\end{scope}
\begin{scope}
\definecolor{drawColor}{gray}{0.30}

\node[text=drawColor,anchor=base,inner sep=0pt, outer sep=0pt, scale=  0.88] at ( 28.39, 19.68) {\scriptsize{$88.0\%$}};

\node[text=drawColor,anchor=base,inner sep=0pt, outer sep=0pt, scale=  0.88] at ( 60.46, 19.68) {\scriptsize{$90.0\%$}};

\node[text=drawColor,anchor=base,inner sep=0pt, outer sep=0pt, scale=  0.88] at ( 92.52, 19.68) {\scriptsize{$92.0\%$}};

\node[text=drawColor,anchor=base,inner sep=0pt, outer sep=0pt, scale=  0.88] at (124.59, 19.68) {\scriptsize{$94.0\%$}};
\end{scope}
\begin{scope}
\definecolor{drawColor}{RGB}{0,0,0}

\node[text=drawColor,anchor=base,inner sep=0pt, outer sep=0pt, scale=  1.10] at ( 78.33,  7.64) {\scriptsize{Shift distribution acc.}};
\end{scope}
\begin{scope}
\definecolor{drawColor}{RGB}{255,0,0}

\path[draw=drawColor,line width= 0.6pt,dash pattern=on 6pt off 6pt ,line join=round] (167.32, 71.62) -- (167.32, 79.39);
\end{scope}
\begin{scope}
\definecolor{drawColor}{RGB}{255,0,0}

\path[draw=drawColor,line width= 0.6pt,dash pattern=on 6pt off 6pt ,line join=round] (167.32, 71.62) -- (167.32, 79.39);
\end{scope}
\begin{scope}
\definecolor{drawColor}{RGB}{0,0,0}

\path[draw=drawColor,line width= 0.6pt,dash pattern=on 6pt off 6pt ,line join=round] (167.32, 63.85) -- (167.32, 71.62);
\end{scope}
\begin{scope}
\definecolor{drawColor}{RGB}{0,0,0}

\path[draw=drawColor,line width= 0.6pt,dash pattern=on 6pt off 6pt ,line join=round] (167.32, 63.85) -- (167.32, 71.62);
\end{scope}
\begin{scope}
\definecolor{drawColor}{RGB}{0,0,0}

\node[text=drawColor,anchor=base west,inner sep=0pt, outer sep=0pt, scale=  0.88] at (175.23, 72.48) {\scriptsize{Acc. at $\delta_{\texttt{Taylor}}$}};
\end{scope}
\begin{scope}
\definecolor{drawColor}{RGB}{0,0,0}

\node[text=drawColor,anchor=base west,inner sep=0pt, outer sep=0pt, scale=  0.88] at (175.23, 64.70) {\scriptsize{Training acc.}};
\end{scope}
\begin{scope}
\definecolor{drawColor}{RGB}{0,0,0}

\node[text=drawColor,anchor=base west,inner sep=0pt, outer sep=0pt, scale=  1.10] at (164.91, 41.36) {\scriptsize{Random shift acc.}};
\end{scope}
\begin{scope}
\definecolor{fillColor}{RGB}{15,140,46}

\path[fill=fillColor] (165.51, 30.27) rectangle (169.12, 36.84);
\end{scope}
\begin{scope}
\definecolor{drawColor}{RGB}{0,0,0}

\node[text=drawColor,anchor=base west,inner sep=0pt, outer sep=0pt, scale=  0.88] at (175.23, 30.53) {\scriptsize{Higher than $\mathbb{E}_{\delta_{\texttt{Taylor}}}$}};
\end{scope}
\end{tikzpicture}

%% file: figures/figure5_right.tex
\begin{tikzpicture}[x=1pt,y=1pt]
\definecolor{fillColor}{RGB}{255,255,255}
\begin{scope}
\definecolor{fillColor}{RGB}{15,140,46}

\path[fill=fillColor] ( 16.19, 32.98) rectangle ( 29.43, 42.17);

\path[fill=fillColor] ( 29.43, 32.98) rectangle ( 42.66, 49.06);

\path[fill=fillColor] ( 42.66, 32.98) rectangle ( 55.89, 58.25);

\path[fill=fillColor] ( 55.89, 32.98) rectangle ( 69.13, 58.25);

\path[fill=fillColor] ( 69.13, 32.98) rectangle ( 82.36, 78.93);

\path[fill=fillColor] ( 82.36, 32.98) rectangle ( 95.60, 53.66);

\path[fill=fillColor] ( 95.60, 32.98) rectangle (108.83, 58.25);

\path[fill=fillColor] (108.83, 32.98) rectangle (122.07, 49.06);

\path[fill=fillColor] (122.07, 32.98) rectangle (135.30, 49.06);

\path[fill=fillColor] (135.30, 39.87) rectangle (148.54, 39.87);

\path[fill=fillColor] (148.54, 35.28) rectangle (161.77, 35.28);

\path[fill=fillColor] (161.77, 32.98) rectangle (175.00, 32.98);
\definecolor{fillColor}{RGB}{253,165,68}

\path[fill=fillColor] ( 16.19, 32.98) rectangle ( 29.43, 32.98);

\path[fill=fillColor] ( 29.43, 32.98) rectangle ( 42.66, 32.98);

\path[fill=fillColor] ( 42.66, 32.98) rectangle ( 55.89, 32.98);

\path[fill=fillColor] ( 55.89, 32.98) rectangle ( 69.13, 32.98);

\path[fill=fillColor] ( 69.13, 32.98) rectangle ( 82.36, 32.98);

\path[fill=fillColor] ( 82.36, 32.98) rectangle ( 95.60, 32.98);

\path[fill=fillColor] ( 95.60, 32.98) rectangle (108.83, 32.98);

\path[fill=fillColor] (108.83, 32.98) rectangle (122.07, 32.98);

\path[fill=fillColor] (122.07, 32.98) rectangle (135.30, 32.98);

\path[fill=fillColor] (135.30, 32.98) rectangle (148.54, 39.87);

\path[fill=fillColor] (148.54, 32.98) rectangle (161.77, 35.28);

\path[fill=fillColor] (161.77, 32.98) rectangle (175.00, 32.98);
\end{scope}
\begin{scope}
\definecolor{drawColor}{gray}{0.30}

\node[text=drawColor,anchor=base,inner sep=0pt, outer sep=0pt, scale=  0.88] at ( 16.19, 19.68) {\scriptsize{$-3.0\%$}};

\node[text=drawColor,anchor=base,inner sep=0pt, outer sep=0pt, scale=  0.88] at ( 55.89, 19.68) {\scriptsize{$-2.0\%$}};

\node[text=drawColor,anchor=base,inner sep=0pt, outer sep=0pt, scale=  0.88] at ( 95.60, 19.68) {\scriptsize{$-1.0\%$}};

\node[text=drawColor,anchor=base,inner sep=0pt, outer sep=0pt, scale=  0.88] at (135.30, 19.68) {\scriptsize{$0.0\%$}};

\node[text=drawColor,anchor=base,inner sep=0pt, outer sep=0pt, scale=  0.88] at (175.00, 19.68) {\scriptsize{$1.0\%$}};
\end{scope}
\begin{scope}
\definecolor{drawColor}{RGB}{0,0,0}

\node[text=drawColor,anchor=base,inner sep=0pt, outer sep=0pt, scale=  1.10] at ( 95.60,  7.64) {\scriptsize{Difference in Shifted acc. ($\mathbb{E}_{\delta_{\texttt{Taylor}}} - \mathbb{E}_{\delta_{\texttt{IS}}}$)}};
\end{scope}
\begin{scope}
\definecolor{drawColor}{RGB}{0,0,0}

\node[text=drawColor,anchor=base west,inner sep=0pt, outer sep=0pt, scale=  1.10] at (199.44, 61.36) {\scriptsize{Lower Acc.}};
\end{scope}
\begin{scope}
\definecolor{fillColor}{RGB}{15,140,46}

\path[fill=fillColor] (200.05, 50.28) rectangle (203.66, 56.84);
\end{scope}
\begin{scope}
\definecolor{fillColor}{RGB}{253,165,68}

\path[fill=fillColor] (200.05, 42.51) rectangle (203.66, 49.07);
\end{scope}
\begin{scope}
\definecolor{drawColor}{RGB}{0,0,0}

\node[text=drawColor,anchor=base west,inner sep=0pt, outer sep=0pt, scale=  0.88] at (209.76, 50.53) {\scriptsize{Taylor}};
\end{scope}
\begin{scope}
\definecolor{drawColor}{RGB}{0,0,0}

\node[text=drawColor,anchor=base west,inner sep=0pt, outer sep=0pt, scale=  0.88] at (209.76, 42.76) {\scriptsize{IS}};
\end{scope}
\end{tikzpicture}

%% file: figures/loss-cube.tex
\pgfmathtruncatemacro{\Divisions}{15} 
\pgfmathsetmacro{\Cube}{1.5} 

\begin{tikzpicture}
[   x={(0.5cm,0.5cm)},
    y={(0.95cm,-0.25cm)},
    z={(0cm,0.9cm)}
]   
    \draw[-,thin] (0, 0, 0) -- (-\Cube/2, 0, 0);
    \draw[-,thin] (0, 0, 0) -- (0, \Cube/2, 0);
    \draw[-,thin] (0, 0, 0) -- (0, 0, \Cube/2);

    \begin{scope}[canvas is yz plane at x=-\Cube/2]
        \clip (-\Cube/2,-\Cube/2) rectangle (\Cube/2,\Cube/2);
        \colorlet{SIDE}[RGB]{myred}
        \colorlet{MIDDLE}[RGB]{mygreen}
        \foreach \x in {1,...,\Divisions}
        {   \pgfmathtruncatemacro{\px}{min((\x+2)/(\Divisions-1)*100, (\Divisions - \x+2)/(\Divisions-1)*100)}
            \colorlet{C}[RGB]{MIDDLE!\px!SIDE}
            \fill[C,fill opacity=0.95] ({-\Cube/2+\Cube*(\x-1)/\Divisions},{-\Cube/2}) rectangle ({-\Cube/2+\Cube*(\x+0.1)/\Divisions},{\Cube/2});
        }
        \draw (-\Cube/2,-\Cube/2) rectangle (\Cube/2,\Cube/2);
    \end{scope}
    
    \begin{scope}[canvas is xz plane at y=\Cube/2]
        \clip (-\Cube/2,-\Cube/2) rectangle (\Cube/2,\Cube/2);
        \fill[myred,fill opacity=0.95] (-\Cube/2,-\Cube/2) rectangle (\Cube/2,\Cube/2);
        \draw (-\Cube/2,-\Cube/2) rectangle (\Cube/2,\Cube/2);
    \end{scope}
    
    \begin{scope}[canvas is xy plane at z=\Cube/2]
        \clip (-\Cube/2,-\Cube/2) rectangle (\Cube/2,\Cube/2);
        \colorlet{MIDDLE}[RGB]{myred}
        \colorlet{SIDE}[RGB]{mygreen}
        \foreach \y in {1,...,\Divisions}
        {   \pgfmathtruncatemacro{\py}{min((\y+2)/(\Divisions-1)*100, (\Divisions-\y+2)/(\Divisions-1)*100)}
            \colorlet{C}[RGB]{SIDE!\py!MIDDLE}
            \fill[C,fill opacity=0.95] ({-\Cube/2},{-\Cube/2+\Cube*(\y-1)/\Divisions}) rectangle ({\Cube/2},{-\Cube/2+\Cube*(\y+0.1)/\Divisions});
        }
        \draw (-\Cube/2,-\Cube/2) rectangle (\Cube/2,\Cube/2);
    \end{scope}
    
    \node at (-1.04*\Cube, 0.4*\Cube/2, -\Cube/2) {Loss};
    \shade[left color=mygreen, right color=myred] (-\Cube, \Cube/2, -\Cube/2) rectangle (-\Cube, 4*\Cube/2, 0);
    
    \draw[->,thin] (-\Cube/2, 0, 0) -- (-2*\Cube/2, 0, 0);
    \draw[->,thin] (0, \Cube/2, 0) -- (0, 2*\Cube/2, 0);
    \draw[->,thin] (0, 0, \Cube/2) -- (0, 0, 2*\Cube/2);
    \node at (-2.8*\Cube/2, 0, 0) {$\delta_1$};
    \node at (0, 2.3*\Cube/2, 0) {$\delta_2$};
    \node at (0, 0, 2.3*\Cube/2) {$\delta_3$};
\end{tikzpicture}

%% file: figures/compare_ipw.tex
\begin{tikzpicture}[x=1pt,y=1pt]
\definecolor{fillColor}{RGB}{255,255,255}
\begin{scope}
\definecolor{drawColor}{gray}{0.92}

\path[draw=drawColor,line width= 0.3pt,line join=round] ( 41.49, 53.59) --
	(183.28, 53.59);

\path[draw=drawColor,line width= 0.3pt,line join=round] ( 41.49, 84.58) --
	(183.28, 84.58);

\path[draw=drawColor,line width= 0.3pt,line join=round] ( 41.49,115.56) --
	(183.28,115.56);

\path[draw=drawColor,line width= 0.3pt,line join=round] ( 41.49,146.55) --
	(183.28,146.55);

\path[draw=drawColor,line width= 0.3pt,line join=round] ( 65.33, 30.69) --
	( 65.33,158.60);

\path[draw=drawColor,line width= 0.3pt,line join=round] (100.14, 30.69) --
	(100.14,158.60);

\path[draw=drawColor,line width= 0.3pt,line join=round] (134.94, 30.69) --
	(134.94,158.60);

\path[draw=drawColor,line width= 0.3pt,line join=round] (169.74, 30.69) --
	(169.74,158.60);

\path[draw=drawColor,line width= 0.6pt,line join=round] ( 41.49, 38.10) --
	(183.28, 38.10);

\path[draw=drawColor,line width= 0.6pt,line join=round] ( 41.49, 69.08) --
	(183.28, 69.08);

\path[draw=drawColor,line width= 0.6pt,line join=round] ( 41.49,100.07) --
	(183.28,100.07);

\path[draw=drawColor,line width= 0.6pt,line join=round] ( 41.49,131.06) --
	(183.28,131.06);

\path[draw=drawColor,line width= 0.6pt,line join=round] ( 47.93, 30.69) --
	( 47.93,158.60);

\path[draw=drawColor,line width= 0.6pt,line join=round] ( 82.74, 30.69) --
	( 82.74,158.60);

\path[draw=drawColor,line width= 0.6pt,line join=round] (117.54, 30.69) --
	(117.54,158.60);

\path[draw=drawColor,line width= 0.6pt,line join=round] (152.34, 30.69) --
	(152.34,158.60);
\definecolor{drawColor}{RGB}{248,118,109}

\path[draw=drawColor,line width= 1.1pt,line join=round] ( 47.93, 68.73) --
	( 73.71, 68.97) --
	( 99.49, 68.57) --
	(125.27, 68.38) --
	(151.05, 67.29) --
	(176.83, 66.07);
\definecolor{drawColor}{RGB}{0,186,56}

\path[draw=drawColor,line width= 1.1pt,line join=round] ( 47.93, 68.73) --
	( 73.71, 68.90) --
	( 99.49, 68.23) --
	(125.27, 67.42) --
	(151.05, 65.33) --
	(176.83, 62.29);
\definecolor{drawColor}{RGB}{97,156,255}

\path[draw=drawColor,line width= 1.1pt,line join=round] ( 47.93, 68.73) --
	( 73.71, 68.98) --
	( 99.49, 68.65) --
	(125.27, 68.79) --
	(151.05, 68.37) --
	(176.83, 68.68);
\definecolor{fillColor}{RGB}{248,118,109}

\path[fill=fillColor,fill opacity=0.05] ( 47.93, 70.84) --
	( 73.71, 71.31) --
	( 99.49, 72.18) --
	(125.27, 74.89) --
	(151.05, 81.01) --
	(176.83,100.30) --
	(176.83, 57.57) --
	(151.05, 61.89) --
	(125.27, 64.46) --
	( 99.49, 65.74) --
	( 73.71, 66.81) --
	( 47.93, 66.90) --
	cycle;
\definecolor{drawColor}{RGB}{248,118,109}

\path[draw=drawColor,line width= 0.1pt,line join=round,line cap=round] ( 47.93, 70.84) --
	( 73.71, 71.31) --
	( 99.49, 72.18) --
	(125.27, 74.89) --
	(151.05, 81.01) --
	(176.83,100.30);

\path[draw=drawColor,line width= 0.1pt,line join=round,line cap=round] (176.83, 57.57) --
	(151.05, 61.89) --
	(125.27, 64.46) --
	( 99.49, 65.74) --
	( 73.71, 66.81) --
	( 47.93, 66.90);
\definecolor{fillColor}{RGB}{0,186,56}

\path[fill=fillColor,fill opacity=0.05] ( 47.93, 70.84) --
	( 73.71, 71.07) --
	( 99.49, 71.44) --
	(125.27, 71.85) --
	(151.05, 71.85) --
	(176.83, 70.54) --
	(176.83, 56.32) --
	(151.05, 60.83) --
	(125.27, 63.89) --
	( 99.49, 65.38) --
	( 73.71, 66.76) --
	( 47.93, 66.90) --
	cycle;
\definecolor{drawColor}{RGB}{0,186,56}

\path[draw=drawColor,line width= 0.1pt,line join=round,line cap=round] ( 47.93, 70.84) --
	( 73.71, 71.07) --
	( 99.49, 71.44) --
	(125.27, 71.85) --
	(151.05, 71.85) --
	(176.83, 70.54);

\path[draw=drawColor,line width= 0.1pt,line join=round,line cap=round] (176.83, 56.32) --
	(151.05, 60.83) --
	(125.27, 63.89) --
	( 99.49, 65.38) --
	( 73.71, 66.76) --
	( 47.93, 66.90);
\definecolor{fillColor}{RGB}{97,156,255}

\path[fill=fillColor,fill opacity=0.05] ( 47.93, 70.84) --
	( 73.71, 71.27) --
	( 99.49, 71.88) --
	(125.27, 73.08) --
	(151.05, 75.78) --
	(176.83, 78.29) --
	(176.83, 61.01) --
	(151.05, 63.05) --
	(125.27, 65.08) --
	( 99.49, 65.77) --
	( 73.71, 66.82) --
	( 47.93, 66.90) --
	cycle;
\definecolor{drawColor}{RGB}{97,156,255}

\path[draw=drawColor,line width= 0.1pt,line join=round,line cap=round] ( 47.93, 70.84) --
	( 73.71, 71.27) --
	( 99.49, 71.88) --
	(125.27, 73.08) --
	(151.05, 75.78) --
	(176.83, 78.29);

\path[draw=drawColor,line width= 0.1pt,line join=round,line cap=round] (176.83, 61.01) --
	(151.05, 63.05) --
	(125.27, 65.08) --
	( 99.49, 65.77) --
	( 73.71, 66.82) --
	( 47.93, 66.90);
\end{scope}
\begin{scope}
\definecolor{drawColor}{gray}{0.92}

\path[draw=drawColor,line width= 0.3pt,line join=round] (188.78, 53.59) --
	(330.57, 53.59);

\path[draw=drawColor,line width= 0.3pt,line join=round] (188.78, 84.58) --
	(330.57, 84.58);

\path[draw=drawColor,line width= 0.3pt,line join=round] (188.78,115.56) --
	(330.57,115.56);

\path[draw=drawColor,line width= 0.3pt,line join=round] (188.78,146.55) --
	(330.57,146.55);

\path[draw=drawColor,line width= 0.3pt,line join=round] (212.62, 30.69) --
	(212.62,158.60);

\path[draw=drawColor,line width= 0.3pt,line join=round] (247.43, 30.69) --
	(247.43,158.60);

\path[draw=drawColor,line width= 0.3pt,line join=round] (282.23, 30.69) --
	(282.23,158.60);

\path[draw=drawColor,line width= 0.3pt,line join=round] (317.03, 30.69) --
	(317.03,158.60);

\path[draw=drawColor,line width= 0.6pt,line join=round] (188.78, 38.10) --
	(330.57, 38.10);

\path[draw=drawColor,line width= 0.6pt,line join=round] (188.78, 69.08) --
	(330.57, 69.08);

\path[draw=drawColor,line width= 0.6pt,line join=round] (188.78,100.07) --
	(330.57,100.07);

\path[draw=drawColor,line width= 0.6pt,line join=round] (188.78,131.06) --
	(330.57,131.06);

\path[draw=drawColor,line width= 0.6pt,line join=round] (195.22, 30.69) --
	(195.22,158.60);

\path[draw=drawColor,line width= 0.6pt,line join=round] (230.03, 30.69) --
	(230.03,158.60);

\path[draw=drawColor,line width= 0.6pt,line join=round] (264.83, 30.69) --
	(264.83,158.60);

\path[draw=drawColor,line width= 0.6pt,line join=round] (299.63, 30.69) --
	(299.63,158.60);
\definecolor{drawColor}{RGB}{248,118,109}

\path[draw=drawColor,line width= 1.1pt,line join=round] (195.22, 68.85) --
	(221.00, 68.48) --
	(246.78, 67.88) --
	(272.56, 67.02) --
	(298.34, 64.20) --
	(324.12, 58.04);
\definecolor{drawColor}{RGB}{0,186,56}

\path[draw=drawColor,line width= 1.1pt,line join=round] (195.22, 68.85) --
	(221.00, 68.30) --
	(246.78, 67.07) --
	(272.56, 64.83) --
	(298.34, 59.64) --
	(324.12, 50.10);
\definecolor{drawColor}{RGB}{97,156,255}

\path[draw=drawColor,line width= 1.1pt,line join=round] (195.22, 68.85) --
	(221.00, 68.47) --
	(246.78, 67.66) --
	(272.56, 66.79) --
	(298.34, 64.35) --
	(324.12, 60.08);
\definecolor{fillColor}{RGB}{248,118,109}

\path[fill=fillColor,fill opacity=0.05] (195.22, 71.31) --
	(221.00, 72.39) --
	(246.78, 75.37) --
	(272.56, 85.79) --
	(298.34,103.40) --
	(324.12,152.79) --
	(324.12, 39.32) --
	(298.34, 50.97) --
	(272.56, 58.27) --
	(246.78, 62.80) --
	(221.00, 65.35) --
	(195.22, 66.48) --
	cycle;
\definecolor{drawColor}{RGB}{248,118,109}

\path[draw=drawColor,line width= 0.1pt,line join=round,line cap=round] (195.22, 71.31) --
	(221.00, 72.39) --
	(246.78, 75.37) --
	(272.56, 85.79) --
	(298.34,103.40) --
	(324.12,152.79);

\path[draw=drawColor,line width= 0.1pt,line join=round,line cap=round] (324.12, 39.32) --
	(298.34, 50.97) --
	(272.56, 58.27) --
	(246.78, 62.80) --
	(221.00, 65.35) --
	(195.22, 66.48);
\definecolor{fillColor}{RGB}{0,186,56}

\path[fill=fillColor,fill opacity=0.05] (195.22, 71.31) --
	(221.00, 71.76) --
	(246.78, 72.74) --
	(272.56, 74.12) --
	(298.34, 75.95) --
	(324.12, 68.62) --
	(324.12, 36.50) --
	(298.34, 48.75) --
	(272.56, 57.49) --
	(246.78, 62.48) --
	(221.00, 65.17) --
	(195.22, 66.48) --
	cycle;
\definecolor{drawColor}{RGB}{0,186,56}

\path[draw=drawColor,line width= 0.1pt,line join=round,line cap=round] (195.22, 71.31) --
	(221.00, 71.76) --
	(246.78, 72.74) --
	(272.56, 74.12) --
	(298.34, 75.95) --
	(324.12, 68.62);

\path[draw=drawColor,line width= 0.1pt,line join=round,line cap=round] (324.12, 36.50) --
	(298.34, 48.75) --
	(272.56, 57.49) --
	(246.78, 62.48) --
	(221.00, 65.17) --
	(195.22, 66.48);
\definecolor{fillColor}{RGB}{97,156,255}

\path[fill=fillColor,fill opacity=0.05] (195.22, 71.31) --
	(221.00, 72.22) --
	(246.78, 74.13) --
	(272.56, 77.45) --
	(298.34, 79.84) --
	(324.12, 80.14) --
	(324.12, 44.01) --
	(298.34, 52.75) --
	(272.56, 58.91) --
	(246.78, 62.96) --
	(221.00, 65.29) --
	(195.22, 66.48) --
	cycle;
\definecolor{drawColor}{RGB}{97,156,255}

\path[draw=drawColor,line width= 0.1pt,line join=round,line cap=round] (195.22, 71.31) --
	(221.00, 72.22) --
	(246.78, 74.13) --
	(272.56, 77.45) --
	(298.34, 79.84) --
	(324.12, 80.14);

\path[draw=drawColor,line width= 0.1pt,line join=round,line cap=round] (324.12, 44.01) --
	(298.34, 52.75) --
	(272.56, 58.91) --
	(246.78, 62.96) --
	(221.00, 65.29) --
	(195.22, 66.48);
\end{scope}
\begin{scope}
\definecolor{drawColor}{gray}{0.10}

\node[text=drawColor,anchor=base,inner sep=0pt, outer sep=0pt, scale=  0.88] at (112.38,163.86) {Linear};
\end{scope}
\begin{scope}
\definecolor{drawColor}{gray}{0.10}

\node[text=drawColor,anchor=base,inner sep=0pt, outer sep=0pt, scale=  0.88] at (259.67,163.86) {Nonlinear};
\end{scope}
\begin{scope}
\definecolor{drawColor}{gray}{0.30}

\node[text=drawColor,anchor=base,inner sep=0pt, outer sep=0pt, scale=  0.88] at ( 47.93, 19.68) {0.0};

\node[text=drawColor,anchor=base,inner sep=0pt, outer sep=0pt, scale=  0.88] at ( 82.74, 19.68) {0.3};

\node[text=drawColor,anchor=base,inner sep=0pt, outer sep=0pt, scale=  0.88] at (117.54, 19.68) {0.6};

\node[text=drawColor,anchor=base,inner sep=0pt, outer sep=0pt, scale=  0.88] at (152.34, 19.68) {0.9};
\end{scope}
\begin{scope}
\definecolor{drawColor}{gray}{0.30}

\node[text=drawColor,anchor=base,inner sep=0pt, outer sep=0pt, scale=  0.88] at (195.22, 19.68) {0.0};

\node[text=drawColor,anchor=base,inner sep=0pt, outer sep=0pt, scale=  0.88] at (230.03, 19.68) {0.3};

\node[text=drawColor,anchor=base,inner sep=0pt, outer sep=0pt, scale=  0.88] at (264.83, 19.68) {0.6};

\node[text=drawColor,anchor=base,inner sep=0pt, outer sep=0pt, scale=  0.88] at (299.63, 19.68) {0.9};
\end{scope}
\begin{scope}
\definecolor{drawColor}{gray}{0.30}

\node[text=drawColor,anchor=base east,inner sep=0pt, outer sep=0pt, scale=  0.88] at ( 36.54, 35.07) {-0.25};

\node[text=drawColor,anchor=base east,inner sep=0pt, outer sep=0pt, scale=  0.88] at ( 36.54, 66.05) {0.00};

\node[text=drawColor,anchor=base east,inner sep=0pt, outer sep=0pt, scale=  0.88] at ( 36.54, 97.04) {0.25};

\node[text=drawColor,anchor=base east,inner sep=0pt, outer sep=0pt, scale=  0.88] at ( 36.54,128.03) {0.50};
\end{scope}
\begin{scope}
\definecolor{drawColor}{RGB}{0,0,0}

\node[text=drawColor,anchor=base,inner sep=0pt, outer sep=0pt, scale=  1.10] at (186.03,  7.64) {Shift Strength};
\end{scope}
\begin{scope}
\definecolor{drawColor}{RGB}{0,0,0}

\node[text=drawColor,rotate= 90.00,anchor=base,inner sep=0pt, outer sep=0pt, scale=  1.10] at ( 13.08, 94.64) {Prediction Error};
\end{scope}
\begin{scope}
\definecolor{drawColor}{RGB}{0,0,0}

\node[text=drawColor,anchor=base west,inner sep=0pt, outer sep=0pt, scale=  1.10] at (347.07,115.29) {Method};
\end{scope}
\begin{scope}
\definecolor{drawColor}{RGB}{248,118,109}

\path[draw=drawColor,line width= 1.1pt,line join=round] (348.51,101.49) -- (360.08,101.49);
\end{scope}
\begin{scope}
\definecolor{drawColor}{RGB}{0,186,56}

\path[draw=drawColor,line width= 1.1pt,line join=round] (348.51, 87.04) -- (360.08, 87.04);
\end{scope}
\begin{scope}
\definecolor{drawColor}{RGB}{97,156,255}

\path[draw=drawColor,line width= 1.1pt,line join=round] (348.51, 72.58) -- (360.08, 72.58);
\end{scope}
\begin{scope}
\definecolor{drawColor}{RGB}{0,0,0}

\node[text=drawColor,anchor=base west,inner sep=0pt, outer sep=0pt, scale=  0.88] at (367.02, 98.46) {IS};
\end{scope}
\begin{scope}
\definecolor{drawColor}{RGB}{0,0,0}

\node[text=drawColor,anchor=base west,inner sep=0pt, outer sep=0pt, scale=  0.88] at (367.02, 84.01) {IS (clipped)};
\end{scope}
\begin{scope}
\definecolor{drawColor}{RGB}{0,0,0}

\node[text=drawColor,anchor=base west,inner sep=0pt, outer sep=0pt, scale=  0.88] at (367.02, 69.55) {Taylor};
\end{scope}
\end{tikzpicture}

%% file: tables/table3.tex
\begin{tabular}{lr}
\toprule
                                   Conditional &  $\delta_i$ \\
\midrule
                    Bald | Male$=0$, Young$=0$ &       0.899 \\
                    Bald | Male$=1$, Young$=1$ &      -0.800 \\
                    Bald | Male$=1$, Young$=0$ &      -0.680 \\
        Wearing Lipstick | Male$=0$, Young$=1$ &      -0.618 \\
        Wearing Lipstick | Male$=0$, Young$=0$ &      -0.543 \\
                        Eyeglasses | Young$=1$ &       0.507 \\
                Mustache | Male$=1$, Young$=0$ &      -0.476 \\
                Mustache | Male$=0$, Young$=0$ &       0.449 \\
                Mustache | Male$=1$, Young$=1$ &      -0.415 \\
                        Eyeglasses | Young$=0$ &       0.399 \\
                 Smiling | Male$=0$, Young$=0$ &      -0.261 \\
        Wearing Lipstick | Male$=1$, Young$=0$ &       0.205 \\
Narrow Eyes | Male$=0$, Smiling$=0$, Young$=0$ &       0.192 \\
  Mouth Slightly Open | Smiling$=1$, Young$=1$ &       0.191 \\
                 Smiling | Male$=1$, Young$=0$ &       0.183 \\
Narrow Eyes | Male$=1$, Smiling$=1$, Young$=1$ &       0.179 \\
  Mouth Slightly Open | Smiling$=0$, Young$=1$ &      -0.153 \\
                Mustache | Male$=0$, Young$=1$ &       0.133 \\
                    Bald | Male$=0$, Young$=1$ &       0.128 \\
  Mouth Slightly Open | Smiling$=1$, Young$=0$ &      -0.127 \\
Narrow Eyes | Male$=0$, Smiling$=1$, Young$=0$ &      -0.125 \\
        Wearing Lipstick | Male$=1$, Young$=1$ &       0.123 \\
Narrow Eyes | Male$=1$, Smiling$=1$, Young$=0$ &      -0.117 \\
Narrow Eyes | Male$=0$, Smiling$=0$, Young$=1$ &       0.106 \\
                            Young | No parents &       0.092 \\
Narrow Eyes | Male$=0$, Smiling$=1$, Young$=1$ &       0.057 \\
Narrow Eyes | Male$=1$, Smiling$=0$, Young$=1$ &      -0.050 \\
Narrow Eyes | Male$=1$, Smiling$=0$, Young$=0$ &      -0.039 \\
  Mouth Slightly Open | Smiling$=0$, Young$=0$ &       0.028 \\
                 Smiling | Male$=1$, Young$=1$ &       0.028 \\
                 Smiling | Male$=0$, Young$=1$ &       0.017 \\
\bottomrule
\end{tabular}